\documentclass[3p,authoryear]{elsarticle}

\usepackage{amsmath, amssymb, amsthm, bm}
\usepackage{mathtools}

\usepackage{amsfonts}

\def\eqref#1{equation~\ref{#1}}

\def\1{\bm{1}}

\DeclareMathAlphabet{\mathsfit}{\encodingdefault}{\sfdefault}{m}{sl}
\SetMathAlphabet{\mathsfit}{bold}{\encodingdefault}{\sfdefault}{bx}{n}

\usepackage{tikz}
\usepackage{pgfplots}
\pgfplotsset{compat=1.18}
\usepgfplotslibrary{groupplots}
\usepackage{booktabs}
\usepackage{algorithm}
\usepackage{algpseudocode}
\usepackage{enumitem}
\usepackage[caption=false]{subfig}
\usepackage{float}
\usepackage[section]{placeins}
\usepackage[colorlinks=true,linkcolor=blue,citecolor=blue,urlcolor=blue]{hyperref}
\usepackage{url}

\newtheorem{theorem}{Theorem}

\newtheorem{proposition}[theorem]{Proposition}
\newtheorem{corollary}[theorem]{Corollary}
\newtheorem{assumption}[theorem]{Assumption}
\theoremstyle{definition}
\newtheorem{definition}[theorem]{Definition}
\theoremstyle{remark}
\newtheorem{remark}[theorem]{Remark}
\newtheorem{example}[theorem]{Example}

\newcommand{\Pa}{\mathrm{Pa}}
\newcommand{\Ch}{\mathrm{Ch}}

\begin{document}
\raggedbottom

\begin{frontmatter}

\title{Inter-Layer Hessian Analysis of Neural Networks with DAG Architectures}

\author[inst1]{Maxim Bolshim\corref{cor1}}
\cortext[cor1]{Corresponding author}
\ead{maxim.bolshim@yandex.ru}
\author[inst1]{Alexander Kugaevskikh}
\ead{a-kugaevskikh@yandex.ru}
\affiliation[inst1]{organization={ITMO University},
            city={Saint Petersburg},
            country={Russia}}

\begin{abstract}
  Automatic differentiation computes loss Hessians, while interpreting
  inter-layer interactions requires an explicit choice of graph-indexed
  blocks.  This paper presents an
  analytical formalism that decomposes the full Hessian into
  blocks indexed by the DAG of an architecture.  The standard
  decomposition $H = H^{GN} + H^T$ separates the Gauss--Newton component
  (positive semidefinite for a convex output loss) from the tensor
  component (residual curvature).  If every node map is locally affine
  in its parent activations, the tensor component of the input Hessian
  vanishes on each smooth cell ($H^{T}_{v,w}=0$,
  $H^f_{v,w}=H^{GN}_{v,w}$); ReLU activations alone do not suffice.
  The full parametric
  Hessian contains residual terms that do not reduce to the GGN\@.
  Building on this decomposition, we use diagnostic metrics
  (inter-layer resonance~$\mathcal{R}$, geometric coupling~$\mathcal{C}$,
  stable rank~$\mathcal{D}$, GN-Gap) with stochastic HVP-based estimates
  whose cost depends on graph evaluation and the number of probes.
  The theoretical analysis gives conditional bounds on curvature decay
  and residual-chain transport;
  empirical examples span fully connected MLPs (Exp.\,1--5) and
  convolutional architectures (ResNet-18, ${\sim}11$M~parameters, Exp.\,6).  When the
  architecture reduces to a single node, all definitions collapse to
  the standard Hessian
  $\nabla^2_\theta\mathcal{L}(\theta)\in\mathbb{R}^{p\times p}$.

  \textit{Archival correction.} The original general novelty claims
  are no longer maintained: general DAG-Hessian calculus has prior
  precedents. This is a corrected archival version.
\end{abstract}


\end{frontmatter}

\section{Introduction}
\label{sec:intro}

The second-order Hessian $\nabla^2\mathcal{L}$ plays a fundamental role
in curvature analysis of the loss landscape and in the design of
optimization methods for neural
networks~\citep{martens2020new}.  Second-order methods
such as Newton, trust-region, and their variants require accurate
curvature information for efficient
optimization~\citep{nocedal2006numerical}.  However, for deep neural networks
computing and storing the full Hessian is computationally prohibitive,
requiring various approximations.

The most common approach---the Gauss--Newton (GN)
approximation---captures only a subset of the second derivatives,
disregarding significant curvature
contributions~\citep{schraudolph2002fast, martens2010deep}.  Block-diagonal
approximations (K-FAC; \citealp{martens2015optimizing}) are restricted
to diagonal blocks and do not capture inter-layer interactions in
arbitrary DAG architectures.  Hessian--vector product (HVP) methods~\citep{pearlmutter1994fast}
yield exact Hessian--vector products without materializing the blocks, which can
also be accessed by selecting directions.  Moreover, for a convex
output loss, the GN approximation is blind to negative curvature: the
block matrix $[H^{GN}_{v,w}]\succeq 0$, rendering it insensitive to
negative-curvature directions~\citep{dauphin2014identifying}.
The standard decomposition $H = H^{GN} + H^T$ retains this information
in the residual component~$H^T$~\citep{schraudolph2002fast,dangel2020modular}.
For non-smooth activations (ReLU, max-pooling), the classical identities
apply on smooth cells; an AD convention at switching boundaries is not
by itself a generalized Hessian construction.

Recent studies of per-layer Hessian
spectra~\citep{bolshim2025local, ghorbani2019investigation} suggest that
diagonal blocks carry diagnostic information about overfitting
and expressivity that a global spectrum alone does not localize.
These per-layer summaries concern
$\partial^2\mathcal{L}/\partial\theta_v^2$, rather than the
cross-blocks considered here.

The present work uses an
analytical formalism for
\emph{structural analysis} of inter-layer Hessian blocks
$\partial^2\mathcal{L}/\partial f_v\,\partial f_w$ and
$\partial^2\mathcal{L}/\partial\theta_v\,\partial\theta_w$ in neural
networks with DAG architectures.  The notation organizes
recurrence relations over the graph, the standard splitting
$H^{GN} + H^T$, path-wise contributions, and diagnostic metrics
(GN-Gap, resonance, coupling).  Such graph-based second-order
calculus has precedents in exact neural-network Hessians
\citep{bishop1992exact}, vector-node recurrences
\citep{martens2012estimating}, and source/path expansions
\citep{gower2012framework}.

\emph{Relation to the GGN literature:} omitting the residual in an
approximation is not the same as lacking a full-Hessian calculus.
The decomposition $H = H^{GN} + H^T$ and local residual terms already
appear in prior work~\citep{schraudolph2002fast,botev2017practical,dangel2020modular};
they are not claimed here as a new general decomposition.

\smallskip\noindent\textbf{Contributions.}
\begin{enumerate}[label=\textbf{C\arabic*}.,leftmargin=1.2cm]
  \item \textbf{Full inter-layer Hessian for DAG architectures.}
    A recurrence relation for blocks $H^f_{v,w}$ accounting for
    pure, mixed, and cross-block second derivatives with respect to
    inputs and parameters, including weight sharing across nodes,
    using established second-order chain-rule constructions
    \citep{bishop1992exact,martens2012estimating,gower2012framework}
    (Eq.~\eqref{eq:Hf}, Eq.~\eqref{eq:Htheta},
    \ref{app:canonical-proof}).

  \item \textbf{Standard decomposition $H = H^{GN} + H^{T}$.}
    Splitting into the Gauss--Newton component (block matrix
    $[H^{GN}_{v,w}]\succeq 0$ for a convex output loss) and the
    tensor component~\citep{schraudolph2002fast,dangel2020modular};
    the former cannot represent negative-curvature directions.
    The quantitative deviation measure is the GN-Gap
    (Theorem~\ref{thm:canonical-decomposition},
    Definition~\ref{def:gn-gap}).

  \item \textbf{Non-smooth case: input Hessian on smooth cells.}
    If every node map is locally affine in its parent activations,
    the tensor component of the input Hessian vanishes on each
    smooth cell: $H^T_{v,w}=0$ and $H^f_{v,w}=H^{GN}_{v,w}$.
    ReLU activations alone do not ensure this condition.
    The parametric residual~\eqref{eq:Htheta} can retain mixed
    input--parameter derivatives even when activation-space
    residuals vanish.  GN-Gap measures the deviation of $H^f$
    from~$H^{GN}$; AD conventions at switching boundaries are
    distinguished from classical second derivatives, without a
    second-order correctness or Newton convergence guarantee from
    CSVF theory~\citep{bolte2021conservative}
    (\ref{app:proof-clarke}).

  \item \textbf{Curvature routing and the HVP operator.}
    Recurrence~\eqref{eq:Hf} describes curvature routing through the
    DAG: decomposition of $H^{GN} + H^T$ along edges and path-wise
    contributions~\citep{martens2012estimating,gower2012framework}.
    Exact HVPs can be evaluated by differentiating the gradient,
    with cost proportional to graph evaluation under standard AD
    cost assumptions~\citep{pearlmutter1994fast}
    (Definition~\ref{def:hb-operator},
      Corollary~\ref{cor:hvp-complexity},
    \ref{app:full-algorithm}).
\end{enumerate}

Building on this framework, we define diagnostic metrics (inter-layer
  resonance~$\mathcal{R}$, geometric coupling~$\mathcal{C}$, stable
rank~$\mathcal{D}$, GN-Gap) that reveal structural curvature
interactions between layers; details are in
\ref{app:diagnostics}.  Stochastic estimates require repeated
Hessian--vector products; their cost depends on graph evaluation,
probe and iteration counts, and the pairs selected.  Memory includes
both probe storage and AD intermediates, not only the $P$ parameters.
The theoretical analysis gives conditional bounds on absolute
activation-GGN norms, residual-chain endpoint transport, and
bottleneck GGN rank; skip connections alone do not guarantee
preservation of long-range coupling.
The reported measurements include a
convolutional architecture (ResNet-18, ${\sim}11$M parameters)
through stochastic curvature estimation.

\emph{Special case:} if the architecture reduces to a single node
with no parents or children, all proposed definitions naturally collapse
to the standard Hessian
$\nabla^2_\theta\mathcal{L}(\theta)
\in\mathbb{R}^{p\times p}$.

\emph{Scope.}  The formalism is designed for \emph{structural analysis},
not for novel optimization algorithms.  Experiments~1--5 report
measurements on fully connected MLPs and toy Attention;
Exp.\,6 extends the empirical scope to a convolutional architecture
(ResNet-18, ${\sim}11$M parameters).
Limits of this scope are discussed in
Section~\ref{sec:limitations}.

\section{Related Work}
\label{sec:related}

\textbf{Second-order curvature approximations.}
Methods based on the GN
approximation~\citep{martens2010deep, martens2012training, botev2017practical}
omit the model residual but need not be block-diagonal.
K-FAC~\citep{martens2015optimizing} and its eigenbasis
variants~\citep{george2018fast} introduce additional block and
factorization approximations.  Shampoo~\citep{gupta2018shampoo}
is another structured preconditioning method.
AdaHessian~\citep{yao2021adahessian} estimates only the Hessian
diagonal; Sophia~\citep{liu2023sophia} uses a stochastic diagonal
estimate for large language models.
Pearlmutter HVP~\citep{pearlmutter1994fast} yields exact
Hessian--vector products, including selected block products without
materializing the matrix.
\citet{abreu2025potential} studied the full GN for LLMs.
\citet{kunstner2019limitations} demonstrated limitations of empirical
Fisher approximations.  The standard decomposition
$H\!=\!H^{GN}\!+\!H^T$~\citep{schraudolph2002fast,dangel2020modular}
separates GGN from the Hessian residual; GN-Gap measures the relative
norm of this residual, not all Fisher or factorization errors.
\citet{dangel2022vivit} exploit the low-rank structure of the GGN for
curvature access; the path expansion here expresses its Jacobian
factors in graph notation.

\textbf{Modular curvature backpropagation.}
\citet{botev2017practical} give preactivation Hessian recurrences with
local residual terms as well as cross-layer GGN expressions.
\citet{dangel2020modular} formulate modular Hessian backpropagation,
including the local second-derivative term, for block-diagonal
curvature approximations; BackPACK~\citep{dangel2020backpack} provides
modular curvature computations.
These block-diagonal approximations do not exhaust prior full-Hessian
calculus: \citet{bishop1992exact} treats exact neural-network Hessians,
\citet{martens2012estimating} gives deterministic mixed recurrences
on vector-node computational graphs before constructing stochastic
estimators, and \citet{gower2012framework} gives local-source and
weighted-path expansions.  The graph-indexed blocks and path
expressions used here are applications of this established calculus,
not its first general formulation.

\textbf{Spectral analysis and loss landscape.}
\citet{papyan2019measurements},
\citet{ghorbani2019investigation}, and
\citet{sagun2017empirical} studied the \emph{global} Hessian spectrum;
\citet{tang2025hessian} and \citet{singh2025geometry} analyzed Hessian
structure in CNNs and LLMs.
Per-layer Hessian spectra have been linked to local curvature
properties and generalization~\citep{bolshim2025local, sagun2017empirical};
the present work considers inter-layer blocks alongside diagonal
blocks to describe specific layer pairs.
\citet{schoenholz2017deep} and \citet{hayou2019impact} established edge
of chaos conditions for signals; the condition $s\rho\!<\!1$
(Corollary~\ref{thm:exponential-decay-main}) is a sufficient
pointwise bound on absolute activation-GGN norms, not an edge-of-chaos
criterion for training.
\citet{cohen2021gradient} empirically showed that GD operates at the
edge of the maximum Hessian eigenvalue (edge of stability); the
fixed-state block bounds here do not establish that dynamical phenomenon.
For a model linearized in its parameters, as in an NTK
approximation~\citep{jacot2018neural}, the model second-derivative
residual vanishes.  A finite-width nonlinear model need not have this
property.

\textbf{Non-smooth activations.}
\citet{clarke1990optimization} develops generalized derivatives;
\citet{bolte2021conservative} develops conservative set-valued fields
for first-order differentiation.  These results do not, by themselves,
identify a twice-differentiated AD program at a kink with a generalized
Hessian or guarantee Newton convergence.
\citet{zhang2018local} studies local Hessian blocks in backpropagation.

\textbf{Positioning.}
The retained material concerns graph-indexed inter-layer blocks,
path expressions, and reported curvature diagnostics.  The original
claim of a new general DAG-Hessian calculus is not maintained.

\section{Methodology}
\label{sec:methodology}

\subsection{Analysis framework}
\label{sec:scope}

\begin{remark}[Scope and paradigm of analysis]
  \label{rem:scope}
  The object of analysis is the \emph{AD-Hessian}---the numerical object
  returned by modern AD frameworks when computing second derivatives.
  On an open cell where all relevant graph operations are $C^2$, it
  coincides with the classical Hessian.  The activation-space residual
  $H^T$ vanishes there if every node map is locally affine in its
  parent activations.  At switching boundaries, choices such as
  $\sigma''(0):=0$ specify an AD convention, not a classical or Clarke
  Hessian.  CSVF theory~\citep{bolte2021conservative} does not supply
  a general second-order correctness or Newton convergence guarantee
  for that convention.

  The goal of the formalism is an explicit decomposition of the Hessian
  into inter-layer cross-blocks, path-wise contributions, and tensor
  components over the DAG (curvature routing), analogous to
  backpropagation routing gradients, using established graph
  calculus~\citep{martens2012estimating,gower2012framework}.
  A full parameter Hessian has $O(P^2)$ entries and is intended for
  analysis of small architectures; exact HVPs instead have
  graph-evaluation cost under standard AD assumptions
  \citep{pearlmutter1994fast}
  (Corollary~\ref{cor:hvp-complexity}).
\end{remark}

\subsection{Notation and network model}
\label{sec:notation}

Core notation is summarized in Table~\ref{tab:notations}; detailed
index conventions and tensor contraction rules are deferred to
\ref{app:notation-details}.

\begin{table}[htbp]
  \centering
  \caption{Core notation.}
  \label{tab:notations}
  \footnotesize
  \begin{tabular}{@{}ll@{}}
    \toprule
    \textbf{Symbol} & \textbf{Definition} \\
    \midrule
    \multicolumn{2}{@{}l}{\emph{Graph and derivatives}} \\
    $G = (V, E)$ & DAG of the neural network \\
    $\Pa(v)$, $\Ch(v)$ & Parents / children of node $v$ \\
    $f_v \in \mathbb{R}^{d_v}$ & Output of node $v$ \\
    $\theta_v \in \mathbb{R}^{p_v}$ & Parameters of node $v$ \\
    $\delta_v$ & Gradient $\nabla_{f_v}\mathcal{L}$ \\
    $D_{u \gets v}$ & Jacobian $\partial f_u/\partial f_v$ \\
    $T_{u;v}$, $T_{u;v,w}$ & Hessians / mixed second derivatives w.r.t.\ inputs \\
    $H^f_{v,w}$ & Input Hessian block \\
    $\partial_C^2 F$ & Clarke-gradient Hessian, only where $F\in C^{1,1}$ \\
    \midrule
    \multicolumn{2}{@{}l}{\emph{Decomposition and metrics
    (Def.\,\ref{def:resonance-coupling-main},\,\ref{def:stable-rank-main})}} \\
    $H^{GN}$, $H^T$ & GN and tensor components \\
    $\mathcal{R}(v,w)$ & Inter-layer resonance $\|H^f_{v,w}\|_F$ \\
    $\mathcal{C}(v,w)$ & Geometric coupling (normalized $\mathcal{R}$) \\
    $\mathcal{D}(v,w)$ & Stable rank~\eqref{eq:stable-rank-main} \\
    $\text{GN-Gap}$ & $\|H^T_{v,w}\|_F/(\|H^{GN}_{v,w}\|_F+\epsilon)$~\eqref{eq:gn-gap} \\
    \bottomrule
  \end{tabular}
\end{table}

We distinguish two regularity cases (detailed definitions of function
spaces are in \ref{app:notation-details}).

\begin{assumption}[Regularity of node functions]
  \label{ass:regularity}
  The output loss $\mathcal{L}$ is $C^2$ in its output argument.
  For every node $v \in V$:
  (A)~in the smooth case $g_v \in C^2$;
  (B)~in the non-smooth case $g_v \in PC^2$ and is locally Lipschitz.
  In case~(B), classical second-derivative identities apply only on
  open smooth regions. Case~(B) does not imply that the gradient
  of the composed loss is locally Lipschitz across their boundaries.
\end{assumption}

\textbf{Architecture.}
A neural network is specified by a DAG $G = (V, E)$.  For each node
$v \in V$: output
$f_v = g_v(f_{\Pa(v)},\theta_v)\in\mathbb{R}^{d_v}$, parameters
$\theta_v\in\mathbb{R}^{p_v}$, loss
$\mathcal{L}:\mathbb{R}^{d_{out}}\to\mathbb{R}$.  Forward evaluation
proceeds in topological order over~$G$; adjoints are propagated in
reverse topological order.

Gradients $\delta_v$, Jacobians $D_{u\gets v}$, and
second-derivative tensors $T_{u;v}$, $T_{u;v,w}$ are defined in
Table~\ref{tab:notations}; element-wise definitions, contraction rules,
and notational conventions are in
\ref{app:notation-details}
(notation follows \citealp{magnus2019matrix}).

\begin{remark}[Element-wise form of the tensor terms]
  \label{rem:tensor-notation}
  For a node $u$ with a single parent~$v$ and node function
  $f_u = \sigma(W_u f_v + b_u)$, the second-derivative tensor has
  entries
  \begin{equation}\label{eq:tensor-elementwise}
    [T_{u;v}]_{i,j,k}
    = \frac{\partial^2 f_{u,i}}{\partial f_{v,j}\,\partial f_{v,k}}
    = \sigma''(z_{u,i})\,W_{u,ij}\,W_{u,ik},
  \end{equation}
  where $z_u = W_u f_v + b_u$.  For a node~$u$ with two
  parents $v,w\in\Pa(u)$ (fan-in~$\geq 2$), the mixed tensor
  $[T_{u;v,w}]_{i,j,k}
  = \partial^2 f_{u,i}/(\partial f_{v,j}\,\partial f_{w,k})$
  captures cross-input curvature; it vanishes for linear merge
  ($f_u = f_v + f_w$), but can be nonzero for a nonlinear
  merge, including bilinear products or a smooth activation
  with $\sigma''\neq 0$ (cf.\ Example~\ref{ex:diamond}).
  The contraction
  $\sum_i [T_{u;v}]_{i,\bullet,\bullet}\,\delta_{u,i}$
  appearing in~\eqref{eq:Hf} is a $d_v\!\times\!d_v$ matrix
  weighted by the gradient at node~$u$.
\end{remark}

\subsection{Non-smooth case: piecewise-smooth AD convention}
\label{sec:clarke-robustness}

On an open region where all relevant node maps are $C^2$, the
classical chain-rule identities apply. ReLU, Leaky~ReLU, and max-pooling
are locally affine away from switching boundaries, so their local
second-derivative tensors vanish there. If \emph{every} node map is
locally affine in its inputs, the activation-space tensor residual
$H^T$ vanishes on that region. ReLU activations alone do not imply this
for graphs containing bilinear products or normalization; nor do they
remove the generally nonzero parametric residual
(Remark~\ref{rem:act-vs-param}). The full GGN is positive semidefinite
when the output loss is convex; a rectangular off-diagonal block
$H^{GN}_{v,w}$ has no PSD ordering.

At switching boundaries a classical Hessian need not exist. The
conventions $\sigma'(0):=0$ and $\sigma''(0):=0$ specify a
\emph{computational} AD result for ReLU primitives, not a
classical/Clarke second derivative or a convergence guarantee.
The distinction is detailed in \ref{app:proof-clarke}.

\subsection{Full input Hessian}
\label{sec:input-hessian}

\begin{definition}[Input Hessian]
  At fixed parameters introduce independent additive post-node
  perturbations $\varepsilon_v$ by
  $f_v(\varepsilon)=g_v(f_{\Pa(v)}(\varepsilon),\theta_v)
  +\varepsilon_v$ and $F(\varepsilon)=\mathcal{L}(f_{out}(\varepsilon))$.
  Then $\delta_v=\nabla_{\varepsilon_v}F(0)$ and the full input
  Hessian has blocks $H^f_{v,w}=\nabla^2_{\varepsilon_v\varepsilon_w}
  F(0)\in\mathbb{R}^{d_v\times d_w}$.
\end{definition}

Write $\mathcal{J}_{p\gets w}=\partial f_p/\partial\varepsilon_w$
for the full perturbation response, with $\mathcal{J}_{w\gets w}=I$
and zero when $w$ cannot reach $p$.  In contrast,
$D_{u\gets v}=\partial g_u/\partial f_v$ denotes the local edge
Jacobian, holding all other parent inputs fixed.  Even when $(v,u)$
is an edge, $\mathcal{J}_{u\gets v}$ also includes any indirect paths.  For $v\neq out$, differentiating
$\delta_v=\sum_{u\in\Ch(v)}D_{u\gets v}^\top\delta_u$
gives the one-sided recursion
\begin{equation}\label{eq:Hf}
  H^f_{v,w}=\sum_{u\in\Ch(v)}D_{u\gets v}^\top H^f_{u,w}
  +\sum_{u\in\Ch(v)}\sum_{p\in\Pa(u)}\sum_{i=1}^{d_u}
  \delta_{u,i}\,[T_{u;v,p}]_{i,\bullet,\bullet}
  \mathcal{J}_{p\gets w},\qquad v\neq out,\quad w\in V.
\end{equation}
Here $T_{u;v,p}$ differentiates the local map in its two parent
slots; the second term includes both pure and mixed sources and the
\emph{full} response $\mathcal{J}_{p\gets w}$.  No independence of the
children of incomparable nodes is assumed.

\begin{example}[Two-layer network]\label{ex:two-layer}
  Consider $v_1\to v_2\to out$ with no skip connections:
  $\Ch(v_1)=\{v_2\}$, $\Ch(v_2)=\{out\}$.
  For $v=w=v_1$, Eq.~\eqref{eq:Hf} reduces to one transported
  block and the local source at $v_2$:
  \[
    H^f_{v_1,v_1} = D_{v_2\gets v_1}^\top H^f_{v_2,v_2}\, D_{v_2\gets v_1}
    + \sum_{i=1}^{d_{v_2}} [T_{v_2;v_1}]_{i,\bullet,\bullet}\,\delta_{v_2,i}.
  \]
  The first term propagates curvature from the descendant~$v_2$; its
  \textbf{GN component} is PSD when the output loss Hessian is PSD
  and is given via $\mathcal{J}_{out\gets v_1}$
  (Theorem~\ref{thm:canonical-decomposition}), but the full
  $H^f_{v_2,v_2}$ also contains a tensor part.  The second term is the
  local tensor-source contribution (which may be indefinite).

  Cross-block $H^f_{v_1,v_2}$: since $f_{v_1}$ affects $\mathcal{L}$
  only through~$v_2$, by the chain rule
  $\frac{\partial^2\mathcal{L}}{\partial f_{v_1}\partial f_{v_2}}
  = D_{v_2\gets v_1}^\top\frac{\partial^2\mathcal{L}}{\partial f_{v_2}^2}
  = D_{v_2\gets v_1}^\top H^f_{v_2,v_2}$.
  This is a special case of the one-directional path
  (\ref{app:theoretical-results},
  Remark~\ref{prop:one-directional-path}).
  With a skip connection $v_1\to out$, Eq.~\eqref{eq:Hf}
  additionally transports $H^f_{out,v_1}$ and includes local sources
  from the skip edge; these contributions account for both paths.
\end{example}

\begin{example}[Diamond graph]\label{ex:diamond}
  Consider a DAG with branching: $v_1\to v_2$, $v_1\to v_3$,
  $v_2\to v_4$, $v_3\to v_4$---two branches converging at
  node~$v_4$.
  For the cross-branch pair $(v_2, v_3)$:
  $\Ch(v_2)\cap\Ch(v_3)=\{v_4\}$ (common child with
  fan-in~$\geq 2$).
  Equation~\eqref{eq:Hf} gives:
  \[
    H^f_{v_2,v_3} = D_{v_4\gets v_2}^\top H^f_{v_4,v_4}\,D_{v_4\gets v_3}
    + \sum_{i=1}^{d_{v_4}} [T_{v_4;\,v_2,v_3}]_{i,\bullet,\bullet}\,\delta_{v_4,i}.
  \]
  The first term transports the \emph{full} curvature through~$v_4$
  (including its residual); the second is the local mixed source.
  It vanishes for linear merging ($f_{v_4}=f_{v_2}+f_{v_3}$),
  but can be nonzero for nonlinear merging.

  Unlike Example~\ref{ex:two-layer}, here $v_2$ and $v_3$ are not
  connected by a directed path---the cross-block $H^f_{v_2,v_3}$
  arises exclusively through the common descendant~$v_4$.
  This illustrates a local mixed source between incomparable parent nodes.
\end{example}

The output row is $H^f_{out,w}=\nabla^2\mathcal{L}(f_{out})
\mathcal{J}_{out\gets w}$ for every $w$, including $H^f_{out,out}=
\nabla^2\mathcal{L}(f_{out})$; the output column follows by
symmetry.  Equation~\eqref{eq:Hf} computes the remaining rows in
reverse topological order.  In particular, its source vanishes for
$w=out$, giving $H^f_{v,out}=\sum_{u\in\Ch(v)}
D_{u\gets v}^\top H^f_{u,out}$.

\subsection{Canonical decomposition of the Hessian}
\label{sec:canonical-decomposition}

Equation~\eqref{eq:Hf} naturally splits into two structurally distinct
components, each with a clear geometric interpretation.

\begin{theorem}[Canonical decomposition]
  \label{thm:canonical-decomposition}
  The input Hessian $H^f_{v,w}$ admits a canonical decomposition:
  \begin{equation}\label{eq:decomposition}
    H^f_{v,w} = H^{GN}_{v,w} + H^{T}_{v,w},
  \end{equation}
  where the components are defined as follows:
  \begin{enumerate}
    \item \textbf{Gauss--Newton component}
      $H^{GN}_{v,w}=\mathcal{J}_{out\gets v}^\top\nabla^2\!\mathcal{L}\,
      \mathcal{J}_{out\gets w}$, where $\mathcal{J}_{out\gets v}$ is the full
      perturbation-to-output Jacobian.  The block matrix
      $[H^{GN}_{v,w}]_{v,w\in V}$ is positive semi-definite when
      $\nabla^2\mathcal{L}\succeq0$ (Remark~\ref{rem:gn-psd-block});
      a rectangular cross-block has no PSD ordering.

    \item \textbf{Tensor component} (residual curvature)
      $H^{T}_{v,w}:=H^f_{v,w}-H^{GN}_{v,w}$ has output row
      $H^T_{out,w}=0$ and, for $v\neq out$, obeys
      \begin{align*}
        H^T_{v,w}={}&\sum_{u\in\Ch(v)}D_{u\gets v}^\top H^T_{u,w}\\
        &+\sum_{u\in\Ch(v)}\sum_{p\in\Pa(u)}\sum_i
          \delta_{u,i}[T_{u;v,p}]_{i,\bullet,\bullet}\mathcal{J}_{p\gets w}.
      \end{align*}
      This component collects local node-function curvature;
      its diagonal blocks can be indefinite.
  \end{enumerate}
\end{theorem}

\begin{proof}[Proof sketch]
  $H^{GN}_{out,w}=\nabla^2\!\mathcal{L}\,\mathcal{J}_{out\gets w}$;
  the chain-rule identity $\mathcal{J}_{out\gets v}=\sum_{u\in\Ch(v)}
  \mathcal{J}_{out\gets u}D_{u\gets v}$ gives its one-sided transport.
  Subtracting this transport from~\eqref{eq:Hf} yields the
  residual recursion, with zero output row.  PSD, when the
  output loss is convex at the evaluation point, follows from
  Remark~\ref{rem:gn-psd-block}.
  Full proof: \ref{app:canonical-proof}.
\end{proof}

\begin{remark}[PSD of the block GN matrix]
  \label{rem:gn-psd-block}
  When $H^{\mathcal{L}}_{out}\!\succeq\!0$
  the block matrix $\mathcal{H}^{GN}\!:=[H^{GN}_{v,w}]_{v,w\in V}$ is PSD:
  $r^\top\!\mathcal{H}^{GN}r
  = \|\textstyle\sum_{v} \mathcal{J}_{out\gets v}\,r_v\|^2_{H^{\mathcal{L}}_{out}}
  \geq 0$.
  PSD is inherited from the \emph{output} Hessian
  $H^{\mathcal{L}}_{out}$; individual blocks $H^f_{u,u}$
  (which include~$H^T$) need not be PSD.
\end{remark}

\begin{remark}[$H^{GN}$ and the GGN]
  \label{rem:ggn-link}
  In unrolled form:
  \[
    H^{GN}_{v,w}=\mathcal{J}_{out\gets v}^\top\,\nabla^2\!\mathcal{L}\;\mathcal{J}_{out\gets w}.
  \]
  Let $J:=[\mathcal{J}_{out\gets v_1},\dots,\mathcal{J}_{out\gets v_n}]$ be the
  Jacobian with respect to the joint perturbations.  Then
  $\mathcal{H}^{GN}=J^\top\!\nabla^2\!\mathcal{L}\,J$, which is
  precisely the \emph{Generalized Gauss--Newton} (GGN)
  matrix~\citep{schraudolph2002fast, martens2015optimizing}.
  GGN requires no probabilistic model.  For a negative log-likelihood
  in the natural output parameters of a regular exponential family,
  the output loss Hessian equals the model Fisher in those coordinates,
  and their pullbacks coincide.  A general negative log-likelihood
  alone does not imply pointwise equality~\citep{martens2020new}.
  Here $H^T$ is an activation-space residual, not the complete
  parameter-space error of K-FAC.
\end{remark}

\begin{proposition}[Input Hessian for piecewise-linear networks]
  \label{prop:ad-ggn-equiv}
  If every node map is locally affine in the joint parent activations
  on the region considered, then $H^T_{v,w}=0$ there for all
  pairs~$(v,w)$ and $H^f_{v,w}=H^{GN}_{v,w}$.
  Piecewise-linear activations alone do not suffice when other nodes
  perform nonlinear operations such as bilinear merging or normalization.
  The parameter Hessian can retain mixed input--parameter curvature
  even when the activation-space residual vanishes
  (Remark~\ref{rem:act-vs-param}).
\end{proposition}

\begin{proof}[Proof sketch]
  Local affinity in all parent activations gives
  $T_{u;v,p}=0$ at every node.  Thus all sources in the
  residual recursion vanish, and its output row is zero.
  Reverse-topological induction yields $H^T_{v,w}=0$.
  ReLU and max-pool nodes are locally affine away from their
  switching boundaries, but this argument applies only if all
  remaining graph operations are locally affine as well.
\end{proof}

\begin{definition}[GN-Gap]
  \label{def:gn-gap}
  The deviation of the full Hessian from the Gauss--Newton
  approximation is measured by the \emph{GN-Gap}:
  \begin{equation}\label{eq:gn-gap}
    \mathrm{Gap}_{GN}(v,w) := \frac{\|H^{T}_{v,w}\|_F}
    {\|H^{GN}_{v,w}\|_F + \epsilon},
  \end{equation}
  where $\epsilon > 0$ is a small constant for numerical stability.
\end{definition}

\begin{corollary}[Interpretation of GN-Gap]
  \label{cor:gn-gap-interpretation}
  GN-Gap quantifies the quality of the Gauss--Newton approximation
  for a given layer pair:
  \begin{itemize}
    \item $\mathrm{Gap}_{GN}(v,w) \approx 0$: the GGN block
      closely approximates the measured activation-Hessian block
      relative to the chosen denominator; this does not certify
      a parameter-space optimizer approximation.
    \item $\mathrm{Gap}_{GN}(v,w) \gg 0$: the residual block
      has appreciable norm relative to the GGN block.
  \end{itemize}
  Since GN-Gap uses the Frobenius norm, it measures the
  \emph{magnitude} of the ignored tensor component, not the sign of
  curvature.  To assess \emph{negative} curvature specifically, one
  should use the negative-curvature mass
  (\ref{app:theoretical-results}).

  \textbf{Scope of applicability.}
  For a jointly piecewise-affine activation DAG, the tensor
  component vanishes away from switching boundaries and GN-Gap is
  zero there.  A smooth node can create a local residual source,
  but its contraction with the adjoint, its transport, or cancellation
  can make $H^T_{v,w}=0$ at a given point.  ReLU alone gives no
  vanishing guarantee in the presence of nonlinear merges.
\end{corollary}

\begin{remark}[Relation to practical optimizers]
  Dropping $H^T$ in~\eqref{eq:decomposition} gives the
  activation-GGN, not a K-FAC approximation.  Parameter-space
  residual terms and K-FAC's block and factorization errors
  require separate analysis~\citep{martens2015optimizing}.
  If the output loss is locally convex, then the full
  activation-GGN matrix $[H^{GN}_{v,w}]$ is PSD
  (Remark~\ref{rem:gn-psd-block}); its off-diagonal blocks
  need not be PSD.  This GGN component cannot represent negative
  curvature of the full matrix caused by the residual.
\end{remark}

\begin{proposition}[Structural sparsity and routing]
  \label{prop:sparsity-routing}
  A necessary condition for $H^f_{v,w}\neq0$ is that both
  perturbations can reach a common loss-dependent node.  This
  condition does not guarantee nonzero curvature, since contributions
  can vanish or cancel.  For $w=out$ the source in~\eqref{eq:Hf}
  vanishes and the recursion reduces to
  $H^f_{v,out}=\sum_{u\in\Ch(v)}D_{u\gets v}^\top H^f_{u,out}$.
  A directed path $v\to^*w$ alone does not remove local sources
  for a general internal $w$.
\end{proposition}

\begin{proof}
  Follows from the topological structure of the graph and the chain
  rule (see \ref{app:canonical-proof}).
\end{proof}

\textbf{Symmetry property.}
In the smooth case (Case~A) $H^f_{v,w}=(H^f_{w,v})^\top$ for all
$v,w\in V$, following from equality of mixed partial derivatives for
twice continuously differentiable functions.

In the non-smooth case (Case~B), at points of non-differentiability,
symmetry of the AD-Hessian may not hold.  In such cases one can
symmetrize:
$\hat{H}^f_{v,w}=\tfrac{1}{2}(H^f_{v,w}+(H^f_{w,v})^\top)$.

\begin{remark}[Symmetrization in the non-smooth case]
  Symmetrization
  $\hat{H}^f_{v,w}=\tfrac{1}{2}(H^f_{v,w}+(H^f_{w,v})^\top)$
  enforces symmetry of the assembled block matrix, but does not
  ensure positive definiteness or permit Cholesky factorization.
  At non-smooth points it is a computational convention, not a
  guarantee of a classical second derivative.
\end{remark}

\subsection{Full parametric Hessian}
\label{sec:parametric-hessian}

For independent parameter copies $\eta_v$ at each node, set
$A_v=\partial g_v/\partial\eta_v$,
$B_v=\nabla^2_{\eta_v\eta_v}\langle\delta_v,g_v\rangle$,
and $M_{v,p}=\nabla^2_{\eta_v f_p}\langle\delta_v,g_v\rangle$
for $p\in\Pa(v)$, with the adjoint held fixed in these local
derivatives.  The local parameter lift is
\begin{align}\label{eq:Htheta}
 H_{\eta_v,\eta_w}={}&A_v^\top H^f_{v,w}A_w
 +\mathbf{1}_{v=w}B_v\nonumber\\
 &+\sum_{p\in\Pa(v)}M_{v,p}\mathcal{J}_{p\gets w}A_w
 +\sum_{p\in\Pa(w)}A_v^\top \mathcal{J}_{p\gets v}^\top M_{w,p}^\top.
\end{align}
The two mixed terms have opposite parameter--input orientations;
the full response Jacobians include all ancestor paths.  If unique
parameters $\theta$ are tied linearly through $\eta=S\theta$, the
complete Hessian is $H_\theta=S^\top H_\eta S$, including every
cross-block between parameter copies.  Linear tying adds no local
second-derivative source (see~\ref{app:Htheta-full}).

\begin{remark}[Activation vs.\ parameter space]
  \label{rem:act-vs-param}
  The canonical decomposition $H^f=H^{GN}+H^T$
  (Theorem~\ref{thm:canonical-decomposition}) is defined for the
  \emph{input} Hessian.  Even when $H^T=0$ for jointly piecewise-affine node maps,
  $H_\eta$ need not equal parameter-space GGN: the local
  $B_v$ and $M_{v,p}$ terms in~\eqref{eq:Htheta} can be nonzero
  (for example, $\partial^2(Wf_p)/(\partial W\,\partial f_p)
  \neq0$).  Linear parameter tying then pulls back all blocks
  without adding curvature of its own.
\end{remark}

\begin{corollary}[Input-to-parameter Hessian bridge]
  \label{cor:activation-to-param-bridge}
  The triangle inequality applied to~\eqref{eq:Htheta} gives
  \[
    \|H_{\eta_v,\eta_w}\|_F\leq
    \|A_v\|_2\,\|H^f_{v,w}\|_F\,\|A_w\|_2+R_{v,w},
  \]
  where $R_{v,w}$ is the sum of the Frobenius norms of the
  local pure and both mixed contributions in~\eqref{eq:Htheta}.
  Activation-space diagnostics alone therefore do not bound
  parameter curvature without these local terms and Jacobian scales;
  for tied parameters, the full pullback must also be applied.
\end{corollary}

\subsection{Curvature routing and HVP operator}
\label{sec:hb-protocol}

Formulas~\eqref{eq:Hf} and~\eqref{eq:Htheta} describe
inter-layer curvature blocks.  The one-sided recursion differentiates
backpropagated adjoints; it does not by itself give a parameter-HVP
contraction.  Exact HVPs can instead be computed by differentiating
the parameter gradient in a fixed direction
(\ref{app:full-algorithm}).

\begin{definition}[Curvature routing recursion]
  \label{def:hb-protocol}
  Let $G=(V,E)$ be the DAG with output loss $\mathcal{L}$.
  Compute $H^f_{out,w}$ from the output boundary and each internal
  row by~\eqref{eq:Hf} in reverse topological order.  The blocks
  are Hessians with respect to independent post-node perturbations;
  symmetry provides the transposed blocks (Figure~\ref{fig:hb-protocol}).
\end{definition}

\begin{figure*}[htbp]
  \centering
  \begin{tikzpicture}[
      node distance=1.6cm and 2.0cm,
      layer/.style={circle, draw, thick, minimum size=8mm, font=\small},
      fwd/.style={->, thick, >=stealth},
      skip/.style={->, thick, >=stealth, dashed, blue!70!black},
      hb/.style={->, thick, >=stealth, red!70!black},
      lbl/.style={font=\scriptsize, midway},
    ]

    \node[font=\small\bfseries] at (1.8, 2.0) {(a) DAG \& forward paths};
    \node[layer] (v0) at (0, 0) {$v_0$};
    \node[layer] (v1) at (1.8, 0) {$v_1$};
    \node[layer] (v2) at (3.6, 0) {$v_2$};
    \node[layer] (out) at (5.4, 0) {$out$};

    \draw[fwd] (v0) -- node[lbl, above] {$D_1$} (v1);
    \draw[fwd] (v1) -- node[lbl, above] {$D_2$} (v2);
    \draw[fwd] (v2) -- node[lbl, above] {$D_3$} (out);

    \draw[skip] (v0) to[bend left=35] node[lbl, above] {$I$\,(skip)} (v2);

    \node[font=\scriptsize, text=black!60] at (2.7, -1.1) {%
    $\mathcal{J}_{out\gets v_0}=D_3\,(I+D_2 D_1)$};

    \node[font=\small\bfseries] at (10.5, 2.0) {(b) Backward pass: curvature routing};
    \node[layer, fill=red!8] (ho) at (12.6, 0) {$out$};
    \node[layer, fill=red!8] (h2) at (10.5, 0) {$v_2$};
    \node[layer, fill=red!8] (h1) at (8.4, 0) {$v_1$};
    \node[layer, fill=red!8] (h0) at (6.6, 0) {$v_0$};

    \draw[hb] (ho) -- node[lbl, above, text=red!70!black] {$D_3^\top$} (h2);
    \draw[hb] (h2) -- node[lbl, above, text=red!70!black] {$D_2^\top$} (h1);
    \draw[hb] (h1) -- node[lbl, above, text=red!70!black] {$D_1^\top$} (h0);

    \draw[hb, dashed, blue!50!red] (h2) to[bend left=35]
    node[lbl, above, text=blue!50!red] {$I^\top$} (h0);

    \node[font=\scriptsize, text=red!70!black, align=center] at (13.1, -0.9) {%
    $H^f_{out,out}=\nabla^2\!\mathcal{L}$};
    \node[font=\scriptsize, text=red!70!black] at (10.8, -0.9) {%
    $H^f_{v_2,out}\!=\!D_3^\top\nabla^2\!\mathcal{L}$};
    \node[font=\scriptsize, text=red!70!black, align=center] at (8.6, -1.25) {%
    $H^f_{v_1,out}\!=\!D_2^\top H^f_{v_2,out}$};
    \node[font=\scriptsize, text=blue!50!red, align=center] at (8.2, 1.0) {%
    $H^f_{v_0,out}\!=\!D_1^\top H^f_{v_1,out}+I^\top H^f_{v_2,out}$};

  \end{tikzpicture}
  \caption{
    (a)~DAG of a neural network with skip connection ($v_0\!\to\!v_2$):
    forward edges (solid) and skip edge (dashed).
    The full Jacobian is a sum of path-wise contributions.
    (b)~Backward pass of curvature routing: curvature matrices $H^f_{v,w}$
    are propagated from $out$ over the graph via transposed Jacobians;
    with a skip connection, the block $H^f_{v_0,out}$ receives contributions
    from both paths.
  }
  \label{fig:hb-protocol}
\end{figure*}

\begin{theorem}[Recursion completeness]
  \label{thm:hb-completeness}
  For smooth node maps and output loss, Eq.~\eqref{eq:Hf}
  recovers the independent-injection Hessian and
  Eq.~\eqref{eq:Htheta}, followed by the full linear pullback for
  shared parameters, recovers $\nabla^2_\theta\mathcal{L}$.
\end{theorem}

\begin{proof}
  Differentiate the backpropagated adjoints with respect to
  independent node injections for~\eqref{eq:Hf}, then apply the
  parameter chain rule, including both mixed orientations in
  \eqref{eq:Htheta}.  The full tying pullback follows from
  linearity of the copy map.  See~\ref{app:canonical-proof} and
  \ref{app:Htheta-full}.
\end{proof}

\subsubsection*{Operator formulation: Hessian--vector products}

The block recursion defines the joint injection-HVP without
constructing a parameter-Hessian matrix.

\begin{definition}[HVP operator]
  \label{def:hb-operator}
  For independent injection directions $r=(r_w)_{w\in V}$,
  $\mathrm{HVP}^f_v(r)=\sum_w H^f_{v,w}r_w$.
  Propagated tangents instead satisfy
  \begin{equation}\label{eq:tangent-forward}
    \dot f_v=r_v+\sum_{w\in\Pa(v)}D_{v\gets w}\dot f_w.
  \end{equation}
  The $r_v$ here are \emph{independent injections}, not the
  already-propagated tangents $\dot f_v$.
\end{definition}

For a parameter direction $q$, the parameter-HVP is
$\nabla_\theta(\langle\nabla_\theta\mathcal{L},q\rangle)$
with $q$ held fixed; it includes the mixed local terms of
Eq.~\eqref{eq:Htheta} and all parameter sharing
(\ref{app:full-algorithm}).

\begin{corollary}[HVP complexity]
  \label{cor:hvp-complexity}
  One exact parameter-HVP costs a constant number of graph
  forward and backward directional sweeps, i.e.,
  \[
    O(T_{\mathrm{fwd}}+T_{\mathrm{bwd}}),
  \]
  under standard AD cost assumptions.  The cost depends on graph operations,
  activations and weight reuse, not solely on the number $P$
  of unique parameters.
\end{corollary}

\begin{proof}
  Differentiating the graph gradient in a fixed parameter
  direction evaluates each local first- and second-directional
  derivative a constant number of times.  Thus the cost scales
  with the corresponding forward/backward graph evaluation,
  provided those local operations admit comparable directional
  derivative costs~\citep{pearlmutter1994fast}.
\end{proof}

\begin{remark}[Relation to AD and practical significance]
  A parameter-HVP can be implemented by directional
  differentiation of the parameter gradient
  (e.g., forward-over-reverse AD)~\citep{pearlmutter1994fast}.
  A joint-injection-HVP uses independent node directions and
  must not substitute already-propagated tangents as those
  directions.  Neither requires storing all Hessian blocks.
\end{remark}

\section{Diagnostics and structural analysis}
\label{sec:structural-analysis}

The block structure $\{H^f_{v,w}\}$ and the canonical decomposition
induce quantitative metrics and structural theorems that connect the
geometry of~$\mathcal{L}$ with the network architecture.

\begin{definition}[Inter-layer resonance and geometric coupling]
  \label{def:resonance-coupling-main}
  For a pair of nodes $v,w\in V$ we define the
  \emph{inter-layer resonance}
  $\mathcal{R}(v,w):=\|H^f_{v,w}\|_F$---a measure of joint
  curvature influence, and the \emph{geometric coupling}
  $\mathcal{C}(v,w):=\mathcal{R}(v,w)\big/\!\sqrt{\mathcal{R}(v,v)
  \cdot\mathcal{R}(w,w)}$---a normalized connectivity measure.
  When both diagonal norms are nonzero and the full
  activation block matrix is PSD, $\mathcal{C}\leq1$.
  A value $\mathcal{C}>1$ rules out that PSD condition;
  with a PSD terminal GGN, residual curvature is then
  required, but this ratio does not identify its source or
  the magnitude of any negative eigenvalue.
  Its behavior under a compensated change of activation
  coordinates is considered in
  Theorem~\ref{thm:coupling-invariance} in Appendix;
  arbitrary adjacent weight rescaling is not an invariance
  of this metric.
  Because the absolute value $\mathcal{R}(v,w)$ depends on
  the selected activation spaces, $\mathcal{C}(v,w)$ provides
  a relative descriptive summary; it does not by itself
  establish comparable training interactions across widths.
\end{definition}

\begin{definition}[Stable rank of inter-layer block]
  \label{def:stable-rank-main}
  For a pair of nodes $v,w\in V$ we define the \emph{stable rank}
  of the inter-layer block:
  \begin{equation}\label{eq:stable-rank-main}
    \mathcal{D}(v, w) := \frac{\|H^f_{v,w}\|_F^2}
    {\|H^f_{v,w}\|_2^2}.
  \end{equation}
  For a nonzero block, the metric describes its singular-value
  profile:\\
  $1\leq\mathcal{D}\leq\mathrm{rank}(H^f_{v,w})$,
  with $\mathcal{D}=1$ only if the block has rank~$1$.
  Stable rank need not equal exact rank.  We assign
  $\mathcal{D}=0$ to an exactly zero block by convention;
  a numerically unresolved nonzero block is not assigned zero.
  Unlike $d_{\mathrm{eff}}:=\|H\|_*/\|H\|_2$
  (Definition~\ref{def:effective-interaction-dim}), the stable rank
  admits stochastic estimation via block actions:
  $\|H\|_F^2$ by the Hutchinson
  estimator~\citep{avron2011randomized},
  $\|H\|_2^2$ by power iteration using the block and its
  true Euclidean adjoint~\citep{nocedal2006numerical}.
  With $m$ probes and $T$ iterations, the cost scales with
  $O(m+2T)$ such actions; their cost and memory depend on
  graph evaluation, stored activations, and the AD strategy,
  not solely on the number $P$ of unique parameters.
  The two-sided bound
  $\mathcal{D}\leq d_{\mathrm{eff}}\leq\sqrt{r\,\mathcal{D}}$ makes
  $\mathcal{D}$ a practical replacement for $d_{\mathrm{eff}}$ in
  large-scale analysis
  (Proposition~\ref{prop:stable-rank-properties},
    Algorithm~\ref{alg:stochastic-stable-rank} in
  \ref{subsubsec:stable-rank}).
\end{definition}

\begin{theorem}[Lyapunov decay of resonance]
  \label{thm:lyapunov-decay}
  Let $v_0\!\to\!v_1\!\to\!\cdots\!\to\!v_L$ be a sequential chain
  in the DAG with $|\mathrm{Ch}(v_i)|=1$ for $0\le i<L$,
  $\Pi_{j\gets i}=D_{v_j\gets v_{j-1}}\cdots D_{v_{i+1}\gets v_i}$
  the product of inter-layer Jacobians, and
  \begin{equation}\label{eq:lyapunov-exponent}
    \lambda_1^{(m)} \;=\;
    \frac{1}{m}\,\log\bigl\|\Pi_{i+m\gets i}\bigr\|_2
  \end{equation}
  the finite-depth Lyapunov exponent.  Then for $0\le i<j\le L$:
  \begin{equation}\label{eq:lyapunov-bound}
    \mathcal{R}(v_i,\,v_j)
    \;\le\; C_j\;\bigl\|\Pi_{j\gets i}\bigr\|_2
    \;=\; C_j\,e^{\,(j-i)\,\lambda_1^{(j-i)}},
  \end{equation}
  where $C_j=\|H^f_{v_j,v_j}\|_F$.  When $v_L=out$ is the sole terminal-loss node,
  $C_L=\|H^f_{out,out}\|_F
  =\|\nabla^2\!\mathcal{L}\|_F$, regardless of the
  internal activation types.
\end{theorem}

\begin{proof}
  Remark~\ref{prop:one-directional-path} with
  $|\mathrm{Ch}(v_k)|=1$ gives
  $H^f_{v_k,v_j}=D_{v_{k+1}\gets v_k}^\top H^f_{v_{k+1},v_j}$.
  Applying $(j{-}i)$ times from $k=i$ to $k=j{-}1$:
  \begin{equation*}
    H^f_{v_i,v_j}
    = D_{v_{i+1}\gets v_i}^\top\!
    D_{v_{i+2}\gets v_{i+1}}^\top\!
    \cdots\,
    D_{v_j\gets v_{j-1}}^\top
    H^f_{v_j,v_j}
    = \Pi_{j\gets i}^\top H^f_{v_j,v_j},
  \end{equation*}
  whence
  $\mathcal{R}(v_i,v_j)
  \le\|\Pi_{j\gets i}\|_2\,C_j$.
\end{proof}

\begin{corollary}[Exponential curvature sensitivity for arbitrary DAGs]
  \label{thm:exponential-decay-main}
  Let $s=\max_v|\Ch(v)|$ be the maximum out-degree and
  $\|D_{u\gets v}\|_2\leq\rho$ for all edges $(v,u)\in E$, with
  $s\rho<1$.  Then for GN-resonance
  $\mathcal{R}^{GN}(v,w):=\|H^{GN}_{v,w}\|_F$:
  \begin{equation}\label{eq:resonance-decay-main}
    \mathcal{R}^{GN}(v,w)\leq C^{GN}\,(s\rho)^{\mathrm{dist}(v,w)},
  \end{equation}
  where $C^{GN}=\|H^{GN}_{out,out}\|_F=\|\nabla^2\!\mathcal{L}\|_F$.
  This pointwise activation-GGN bound assumes a single terminal loss;
  it controls absolute block norms, not normalized coupling or
  parameter-Hessian blocks.  When $s=1$, merging DAGs as well as
  chains are possible.  A proof for arbitrary DAGs is in
  \ref{app:theoretical-results}.
\end{corollary}

\begin{remark}[Multiplicative ergodic theorem]
  \label{rem:lyapunov}
  On a sequential chain with $s=1$ and a terminal
  endpoint, bound~\eqref{eq:lyapunov-bound} uses the actual
  Jacobian product, whereas~\eqref{eq:resonance-decay-main}
  uses the uniform edge norm.  Their gap can be exponential (at $\rho\!=\!2{.}6$, $L\!=\!8$:
  $\rho^8\!\approx\!2088$ vs.\ $\|\Pi\|_2\!\approx\!5$).
  If the Jacobians $\{D_i\}$ form a stationary ergodic
  sequence under the hypotheses of the Oseledets
  theorem~\citep{oseledets1968multiplicative},
  $\lambda_1=\lim_{m\to\infty}\lambda_1^{(m)}$ exists a.s.\
  and $\mathcal{R}\le C\,e^{(\lambda_1+o(1))L}$.
  A negative exponent suffices for this upper bound to decay
  when the terminal curvature factor $C$ stays bounded with depth;
  a nonnegative exponent does not force the curvature block to grow.
\end{remark}

\begin{corollary}[Graph-distance bound on GN blocks]
  \label{cor:bandedness-main}
  When $0<s\rho<1$, each pointwise activation-GGN block with
  $\mathrm{dist}(v,w)>k$ has Frobenius norm at most
  $C^{GN}(s\rho)^{k+1}$.
  This is an absolute graph-distance bound: it does not imply
  normalized-coupling decay, block bandedness in a topological
  ordering, or a bound on the number of nearby pairs.
  An extension to the full input Hessian requires all node maps
  to be locally affine, not merely ReLU activations.
\end{corollary}

\begin{corollary}[Scope of the GGN distance estimate]
  \label{cor:kfac-error}
  The pointwise activation-GGN estimate in
  Corollary~\ref{cor:bandedness-main} does not by itself bound
  the full parameter-Hessian blocks discarded by
  K-FAC~\citep{martens2015optimizing}: parameter Jacobians,
  mixed parameter terms, residual curvature, and cumulative
  approximation error require separate control.
\end{corollary}

\begin{remark}[Tightness of the $s\rho<1$ condition]
  \label{rem:srho-tightness}
  The condition $s\rho<1$ is sufficient but not necessary:
  a contracted Jacobian product can also yield a decaying
  upper bound along a sequential chain
  (Remark~\ref{rem:lyapunov}).  In practice
  (Table~\ref{tab:rho-accuracy}), $\rho_{\max}\gg 1$
  at He initialization, while observed Jacobian products may
  contract due to their singular-vector alignment.
\end{remark}

Beyond norm decay, narrow layers constrain the
\emph{dimensionality} of coordination between distant blocks.
For a shared separating bottleneck of width $d_u$ on all
paths from both nodes to the terminal output, both total
Jacobian factors pass through $\mathbb{R}^{d_u}$ and the
pointwise activation-GGN rank inherits this compression.

\begin{proposition}[Rank bottleneck]
  \label{thm:rank-bottleneck-main}
  If \textbf{all} paths from~$v$ and~$w$ to~$\mathrm{out}$
  pass through a node~$u$ with
  $d_u\ll\min(d_v,d_w)$, then the Gauss--Newton component is
  strictly bounded:
  $\mathrm{rank}(H^{GN}_{v,w})\leq d_u$.
  This exact rank bound applies to the pointwise
  activation-GGN block.  The residual $H^T_{v,w}$ can restore
  full rank to $H^f_{v,w}$; averaging pointwise blocks need not
  preserve their rank bound, and low stable rank is not an
  exact-rank guarantee.
  (Follows directly from submultiplicativity of rank under matrix
  multiplication~\citep{horn2012matrix}.)
\end{proposition}

\begin{proof}[\textbf{Proof sketch}]
  Since every directed path from~$v$ (resp.~$w$) to~$\mathrm{out}$
  traverses~$u$, the end-to-end Jacobian factors as
  $\mathcal{J}_{\mathrm{out}\leftarrow v}=\mathcal{J}_{\mathrm{out}\leftarrow u}\,
  \mathcal{J}_{u\leftarrow v}$,
  where $\mathcal{J}_{u\leftarrow v}\in\mathbb{R}^{d_u\times d_v}$.
  Substituting into the Gauss--Newton block gives
  \[
    H^{GN}_{v,w}
    = \mathcal{J}_{u\leftarrow v}^{\!\top}\,
    \underbrace{\mathcal{J}_{\mathrm{out}\leftarrow u}^{\!\top}\,
      \nabla^{2}\!\mathcal{L}\,
    \mathcal{J}_{\mathrm{out}\leftarrow u}}_{\in\,\mathbb{R}^{d_u\times d_u}}\,
    \mathcal{J}_{u\leftarrow w}.
  \]
  By submultiplicativity of rank under matrix
  products~\citep{horn2012matrix}, the outer factors cannot
  increase rank beyond the inner $d_u\!\times\!d_u$ core, hence
  $\mathrm{rank}(H^{GN}_{v,w})\leq d_u$.
  Full details are in \ref{app:theoretical-results}.
\end{proof}

\begin{remark}[Skip connections and the rank constraint]
  Proposition~\ref{thm:rank-bottleneck-main} applies when
  \emph{all} paths pass through node~$u$.  With skip connections
  bypassing~$u$, additional terms appear in the path
  decomposition~\eqref{eq:path-decomp-main}, and
  $\mathrm{rank}(\sum_p D_p^\top(\cdots)D_p)$ may exceed~$d_u$.
  Bypass paths can remove this particular GGN rank bound,
  although they do not guarantee full rank, a larger block norm,
  or stronger coupling.
\end{remark}

Corollary~\ref{thm:exponential-decay-main} gives an absolute
pointwise activation-GGN bound under $s\rho<1$, while
Proposition~\ref{thm:rank-bottleneck-main} constrains its rank
when the specified paths factor through a narrow layer.
Stable rank $\mathcal{D}(v,w)$
(Definition~\ref{def:stable-rank-main}) separately summarizes
the singular-value profile of the measured block.
The path expansion identifies possible contributions without
implying monotone norms.

\begin{theorem}[Path decomposition of the Hessian]
  \label{thm:path-decomposition-main}
  Let $\mathcal{P}(v\to c)$ be the set of directed paths from $v$
  to $c$, including the length-zero path when $v=c$, and
  $D_p=\prod_{(u_i,u_{i+1})\in p}D_{u_{i+1}\gets u_i}$
  the path-wise Jacobian.  For a loss depending only on the
  terminal output,
  \begin{equation}\label{eq:path-decomp-main}
    H^f_{v,w}=\sum_{\substack{p_v\in\mathcal{P}(v\to out)\\
                              p_w\in\mathcal{P}(w\to out)}}
    D_{p_v}^\top\nabla^2\!\mathcal{L}\,D_{p_w}
    + H^T_{v,w}.
  \end{equation}
  The first sum is the terminal GGN contribution counted once;
  $H^T_{v,w}$ is the sum of transported local second-derivative
  sources (expanded in \ref{app:diagnostics}).  Matrix
  contributions add, but their Frobenius norms need not.
\end{theorem}

\noindent\textit{Proof: \ref{app:diagnostics}.}

Equation~\eqref{eq:path-decomp-main} makes the available
Jacobian paths explicit.  Adding a skip edge can change the path
sum and its interference terms; neither a larger cross-block
norm nor less decay follows without further assumptions.

\begin{theorem}[Endpoint bounds: vanilla network vs.\ residual chain]
  \label{thm:resnet-main}
  Let $v_L=out$ carry the sole terminal loss and put
  $C=\|\nabla^2\!\mathcal{L}\|_F$.  Suppose the vanilla
  chain has $L$ edge Jacobians of norm at most $\rho<1$.
  For the residual chain, suppose $N=L/k$ square blocks have
  Jacobians $I+B_j$ with $\|B_j\|_2\leq\beta_j<1$.

  \textbf{(a) Vanilla (upper bound):}
  $\mathcal{R}(v_0,v_L)\leq C\rho^L\to0$
  when $C$ stays bounded.

  \textbf{(b) Residual chain (conditional upper bound):}
  \[
    \mathcal{R}(v_0,v_L)
    \leq C\prod_{j=1}^{N}(1+\beta_j).
  \]
  This bound is not depth-independent for fixed branch size.

  \textbf{(b$'$) Residual chain (conditional lower bound):}
  \[
    \mathcal{R}(v_0,v_L)
    \geq C\prod_{j=1}^{N}(1-\beta_j).
  \]
  If $\beta_j\leq\alpha\rho^k$ and $\alpha=c/L$ with
  fixed $k,c,\rho$, the upper factor is at most
  $(1+c\rho^k/L)^{L/k}\to e^{c\rho^k/k}$ and the
  lower factor is at least
  $(1-c\rho^k/L)^{L/k}\to e^{-c\rho^k/k}$.
  These are relative endpoint bounds, not unconditional
  preservation of curvature or arbitrary off-diagonal blocks.
\end{theorem}

\begin{proof}[\textbf{Proof sketch}]
  At the terminal endpoint, the post-node Hessian block is
  $H^f_{v_0,out}=\mathcal{J}_{out\gets v_0}^\top\nabla^2\!\mathcal{L}$;
  local residual sources do not enter this block.
  For the vanilla chain, submultiplicativity gives
  $\|\mathcal{J}_{out\gets v_0}\|_2\leq\rho^L$.
  For the residual chain,
  $\mathcal{J}_{out\gets v_0}=\prod_{j=1}^{N}(I+B_j)$ and
  \[
    \prod_j(1-\beta_j)
    \leq\sigma_{\min}(\mathcal{J}_{out\gets v_0})
    \leq\|\mathcal{J}_{out\gets v_0}\|_2
    \leq\prod_j(1+\beta_j).
  \]
  Apply the minimum-singular-value and operator-norm bounds
  to $\mathcal{J}_{out\gets v_0}^\top\nabla^2\!\mathcal{L}$.
  Under the scaled-branch condition, the product factors
  have finite nonzero bounds as shown above.
  Full proof: \ref{app:diagnostics}.
\end{proof}

\begin{remark}[Edge of chaos for curvature]
  For a contracting vanilla chain, the endpoint upper bound
  decays when its terminal curvature stays bounded.
  Identity skips replace this with product bounds involving
  residual-branch Jacobians.  A depth-uniform positive lower
  bound relative to a nonzero terminal curvature follows when
  the cumulative branch norms remain bounded below one per
  block; fixed-size branches alone do not give such a bound.
\end{remark}

\begin{remark}[Practical implications]
  \label{rem:resnet-precon}
  For selected activation pairs, large measured cross-block
  norms can motivate testing approximations that retain those
  blocks.  Neither the GN distance bound nor path counting
  establishes parameter-space K-FAC error, an optimal banded
  preconditioner, or universal DenseNet/U-Net coupling.
  The Attention example below describes possible local sources;
  see also \ref{app:attention-example},~\ref{app:theoretical-results}.
\end{remark}

\begin{example}[Attention block as DAG: density of cross-blocks]
  \label{ex:attention}
  A single-head Attention defines a DAG with nodes
  $\{v_{\mathrm{in}},v_Q,v_K,v_V,v_{\mathrm{out}}\}$, where
  $v_{\mathrm{out}}$ computes
  $O=\sigma(QK^\top\!/\sqrt{d_k})\,V$
  (Softmax$\times$Value) with
  $\Pa(v_{\mathrm{out}})=\{v_Q,v_K,v_V\}$ (fan-in$\,=3$).
  For nondegenerate multi-position Softmax the Jacobian can
  couple positions.  The local mixed derivative of
  $O=\sigma(QK^\top/\sqrt{d_k})V$ with respect to $Q,K$
  has contributions both from the Softmax Hessian and from
  its first derivative acting on the bilinear $QK^\top$ map.
  After contraction with the downstream adjoint and transport,
  these terms may vanish or cancel in a particular Hessian
  block; common ancestry alone does not make $H^T_{Q,K}$
  nonzero.

  Structural implications:
  (i)~cross-position coupling is possible, so a banded
  approximation requires separate quantitative justification;
  (ii)~the GN-Gap for $(Q,K)$ describes the residual relative
  to GGN when measured, rather than guaranteeing a nonzero
  value or a Softmax-only source;
  (iii)~multi-head analysis may involve additional interactions
  through shared inputs, output projection, or the terminal loss.
  Further component formulas are in \ref{app:attention-example}.
\end{example}

\section{Experimental validation of diagnostic metrics}
\label{sec:experiments}

The experiments describe activation-space curvature profiles and
structural null controls, rather than propose a new algorithm.
The conditional theoretical bounds are distinguished below from
observations outside their assumptions.
Six experiments are reported.
Experiments~1--4 use CIFAR-10 (CIFAR-100 for Exp.\,2) with
fully connected MLPs on the flattened input
($d_{\mathrm{in}}=3072$; channel normalization to~$[-1,1]$,
no augmentation), enabling exact Hessian block computation.
Experiment~5 uses synthetic data and a minimal
Attention model (\S\,\ref{sec:exp5}).
Experiment~6 extends the validation to a convolutional
architecture (ResNet-18, ${\sim}11$M parameters) on CIFAR-10
via stochastic curvature estimation (\S\,\ref{sec:exp6}).
Each configuration is trained with
5~random seeds $\{42,\ldots,46\}$; metrics are computed at
checkpoints: \textit{init} (before training), \textit{mid}
(half epochs), \textit{final}.

\textbf{Common protocol.}
Training: SGD ($\eta\!=\!0.01$, momentum~$0.9$,
weight decay~$10^{-4}$), cosine schedule, $50$~epochs,
batch size~$128$, gradient clipping ($\|\nabla\|_2\!\leq\!1$).
Plain~MLP blocks: Linear--LayerNorm--$\sigma$ (LayerNorm disabled
  in Exp.\,3, where piecewise linearity is
  required---\S\,\ref{sec:exp3}; and in Exp.\,1b with spectral
normalization---\ref{app:exp-results});
Residual~MLP blocks: Linear--$\sigma$ + identity skip (no
LayerNorm, $\alpha\!=\!L^{-1/2}$).%
\footnote{LayerNorm can contribute to~$H^T$ even with ReLU.
  Exp.\,3 disables it for the piecewise-affine null control of
  Proposition~\ref{prop:ad-ggn-equiv}.}
Thus the main Plain/Residual MLP comparison changes both skip
connections and normalization, not skips alone.
Stochastic estimates use 30~Rademacher probes shared across examples
and a subsample of 32~examples unless specified otherwise.
Blocks are averaged over examples before computing norms; distance
profiles then average the pairwise metrics.
The historical spectral-norm estimator uses 20~power iterations,
but its transpose product for nonlinear ancestor--descendant pairs
is not certified as the adjoint of the forward product.
Stochastic $\mathcal{D}$ values are retained as historical estimates,
not as verified rank evidence (\ref{app:ablation}); this issue does
not by itself invalidate exact $\mathcal{D}$ or the forward-product
estimates of $\mathcal{R}$, $\mathcal{C}$, and GN-Gap.

\subsection{\texorpdfstring{Exp.\,1}{Exp. 1}: Resonance and coupling decay}
\label{sec:exp1}

\textbf{Protocol.}
For each depth $L\!\in\!\{8,10,12\}$, Plain~MLP and
Residual~MLP (skip at every block, $\alpha\!=\!L^{-1/2}$) of
width~64 are trained with SGD and cosine schedule, 50~epochs
(additional depths $L\!\in\!\{16,32\}$: Appendix).  At each
checkpoint we compute $\bar{\mathcal{R}}(d)$ and
$\bar{\mathcal{C}}(d)$---averages over pairs at distance~$d$.

\textbf{Results} (Figure~\ref{fig:decay-R}).
(a)~In Plain~MLP, $\bar{\mathcal{R}}(d)$ shows an approximately
exponential decline with~$d$.  Log-linear
regression $\log\bar{\mathcal{R}}=a-bd$ yields $R^2\!>\!0.91$
for all depths $L\!\in\!\{8,10,12\}$ at init and final
(Table~\ref{tab:rho-accuracy}).
(b)~In Residual~MLP, $\bar{\mathcal{R}}(d)$ declines more slowly
($b\!\approx\!0.02$).
These profiles do not validate
Corollary~\ref{thm:exponential-decay-main} or
Theorem~\ref{thm:resnet-main}: their contraction/branch-norm
conditions are not established here, and $\alpha\!=\!L^{-1/2}$
is not the depth-uniform $O(1/L)$ scaling in the latter theorem.
The comparison also changes LayerNorm, so it does not isolate a
causal skip effect.
(c)~Coupling $\bar{\mathcal{C}}(d)$
(Figure~\ref{fig:coupling-C})
qualitatively reproduces the same trends: decay in Plain,
preservation in ResNet.

\textbf{Condition $\boldsymbol{s\rho < 1}$ and the Lyapunov
exponent.}
For He initialization $\rho_{\max}\!\approx\!2.6$
(Table~\ref{tab:rho-accuracy}), so the sufficient condition
$s\rho\!<\!1$ of
Corollary~\ref{thm:exponential-decay-main} is \emph{formally
not satisfied}, although decay of
$\bar{\mathcal{R}}(d)$ is observed.  Products of Jacobians can
be much smaller than products of their norms, so this sufficient
condition need not be necessary.  Exp.\,1b
(\ref{app:exp-results}) reports a spectral-normalization control;
$\rho\leq1$ alone is not the strict condition $s\rho<1$.
Identity skips likewise do not guarantee a particular norm or decay
rate without bounds on the residual branches.

\begin{remark}[Batch-averaged decay beyond the Lyapunov bound]
  \label{rem:batch-decay}
  Theorem~\ref{thm:lyapunov-decay} gives a \emph{per-sample}
  upper bound; a positive exponent does not imply growth of
  the measured block.  The experiment instead computes
  \[
    \mathcal{R}_{B}(v,w)
    =\left\|\frac{1}{B}\sum_b H^f_{v,w}(x_b)\right\|_F,
  \]
  then averages $\mathcal{R}_{B}(v,w)$ over pairs at distance~$d$
  to obtain $\bar{\mathcal{R}}(d)$.
  This is not $B^{-1}\sum_b\|H^f_{v,w}(x_b)\|_F$.
  Matrix contributions can cancel before the norm is taken, but
  these measurements do not establish cancellation as the cause
  of the observed distance decay.
\end{remark}

\textbf{U-shaped artifact.}
At $L\!\in\!\{16,32\}$ the unnormalized $\bar{\mathcal{R}}(d)$
exhibits a rise at large distances, while normalized
$\bar{\mathcal{C}}(d)$ declines
(Table~\ref{tab:exp1-deep} in Appendix).
This contrast need not imply an increase in pairwise dependence;
the unnormalized profile is also affected by self-curvature scales.

\begin{figure}[htbp]
  \centering
  \begin{tikzpicture}
    \begin{groupplot}[
        group style={
          group size=1 by 2,
          vertical sep=1.5cm,
        },
        width=0.95\columnwidth,
        height=5.5cm,
        xlabel={Distance $d$},
        ylabel={$\bar{\mathcal{C}}(d)$},
        ymin=0, ymax=1.05,
        grid=major,
        legend style={font=\footnotesize, at={(0.02,0.02)},
        anchor=south west, legend columns=2},
      ]
      \nextgroupplot[title={$L=8$}]
      \addplot[blue, thick, mark=triangle*, mark size=2,
      error bars/.cd, y dir=both, y explicit] coordinates {
        (0,1.0000) +- (0,0.0000) (1,0.9452) +- (0,0.0172)
        (2,0.8950) +- (0,0.0287) (3,0.8316) +- (0,0.0447)
        (4,0.7542) +- (0,0.0586) (5,0.6657) +- (0,0.0642)
        (6,0.5555) +- (0,0.0495) (7,0.4307) +- (0,0.0626)
      }; \addlegendentry{Plain init}
      \addplot[blue, thick, dashed, mark=triangle, mark size=2,
      error bars/.cd, y dir=both, y explicit] coordinates {
        (0,1.0000) +- (0,0.0000) (1,0.9658) +- (0,0.0062)
        (2,0.9349) +- (0,0.0169) (3,0.8952) +- (0,0.0286)
        (4,0.8613) +- (0,0.0370) (5,0.8254) +- (0,0.0514)
        (6,0.7602) +- (0,0.0539) (7,0.6946) +- (0,0.0558)
      }; \addlegendentry{Plain final}
      \addplot[red, thick, mark=square*, mark size=2,
      error bars/.cd, y dir=both, y explicit] coordinates {
        (0,1.0000) +- (0,0.0000) (1,0.9989) +- (0,0.0002)
        (2,0.9978) +- (0,0.0004) (3,0.9968) +- (0,0.0006)
        (4,0.9957) +- (0,0.0011) (5,0.9946) +- (0,0.0014)
        (6,0.9933) +- (0,0.0018) (7,0.9921) +- (0,0.0022)
      }; \addlegendentry{Res.\ init}
      \addplot[red, thick, dashed, mark=square, mark size=2,
      error bars/.cd, y dir=both, y explicit] coordinates {
        (0,1.0000) +- (0,0.0000) (1,0.9970) +- (0,0.0002)
        (2,0.9917) +- (0,0.0009) (3,0.9843) +- (0,0.0021)
        (4,0.9746) +- (0,0.0046) (5,0.9630) +- (0,0.0059)
        (6,0.9477) +- (0,0.0074) (7,0.9291) +- (0,0.0086)
      }; \addlegendentry{Res.\ final}
      \nextgroupplot[title={$L=12$}]
      \addplot[blue, thick, mark=triangle*, mark size=2,
      error bars/.cd, y dir=both, y explicit] coordinates {
        (0,1.0000) +- (0,0.0000) (1,0.9478) +- (0,0.0112)
        (2,0.8948) +- (0,0.0322) (3,0.8382) +- (0,0.0422)
        (4,0.7812) +- (0,0.0449) (5,0.7219) +- (0,0.0428)
        (6,0.6547) +- (0,0.0380) (7,0.5809) +- (0,0.0321)
        (8,0.5071) +- (0,0.0266) (9,0.4335) +- (0,0.0259)
        (10,0.3514) +- (0,0.0316) (11,0.2386) +- (0,0.0403)
      };
      \addplot[blue, thick, dashed, mark=triangle, mark size=2,
      error bars/.cd, y dir=both, y explicit] coordinates {
        (0,1.0000) +- (0,0.0000) (1,0.9380) +- (0,0.0147)
        (2,0.9156) +- (0,0.0133) (3,0.8964) +- (0,0.0162)
        (4,0.8723) +- (0,0.0213) (5,0.8447) +- (0,0.0209)
        (6,0.8226) +- (0,0.0208) (7,0.7964) +- (0,0.0257)
        (8,0.7564) +- (0,0.0312) (9,0.7184) +- (0,0.0383)
        (10,0.6674) +- (0,0.0515) (11,0.6584) +- (0,0.0760)
      };
      \addplot[red, thick, mark=square*, mark size=2,
      error bars/.cd, y dir=both, y explicit] coordinates {
        (0,1.0000) +- (0,0.0000) (1,0.9993) +- (0,0.0001)
        (2,0.9986) +- (0,0.0003) (3,0.9979) +- (0,0.0003)
        (4,0.9973) +- (0,0.0004) (5,0.9966) +- (0,0.0005)
        (6,0.9960) +- (0,0.0006) (7,0.9953) +- (0,0.0007)
        (8,0.9945) +- (0,0.0010) (9,0.9936) +- (0,0.0012)
        (10,0.9928) +- (0,0.0014) (11,0.9921) +- (0,0.0014)
      };
      \addplot[red, thick, dashed, mark=square, mark size=2,
      error bars/.cd, y dir=both, y explicit] coordinates {
        (0,1.0000) +- (0,0.0000) (1,0.9984) +- (0,0.0001)
        (2,0.9956) +- (0,0.0003) (3,0.9917) +- (0,0.0006)
        (4,0.9867) +- (0,0.0010) (5,0.9806) +- (0,0.0017)
        (6,0.9734) +- (0,0.0028) (7,0.9654) +- (0,0.0038)
        (8,0.9567) +- (0,0.0050) (9,0.9466) +- (0,0.0071)
        (10,0.9363) +- (0,0.0084) (11,0.9263) +- (0,0.0085)
      };
    \end{groupplot}
  \end{tikzpicture}
  \caption{Exp.~1: Geometric coupling
    $\bar{\mathcal{C}}(d)$ vs.\ distance~$d$
    (error bars: $\pm 1\sigma$ over 5~seeds).
    \textbf{(a)}~$L\!=\!8$, \textbf{(b)}~$L\!=\!12$.
    Plain~MLP: $\mathcal{C}$ decays monotonically from~1 to
    0.24 ($L\!=\!12$, init), reflecting loss of geometric
    coherence between distant layers.
    Residual~MLP retains higher coupling in these configurations.
    After training the Plain profiles shift upward, whereas the
    Residual profiles shift downward. The comparison also changes
    LayerNorm and does not isolate a causal effect of skip connections.}
  \label{fig:coupling-C}
\end{figure}

\begin{table}[htbp]
  \centering
  \caption{Spectral norms of Jacobians $\rho_{\max}$, Lyapunov
    exponent $\lambda_1$, and test accuracy for Exp.\,1
  configurations (CIFAR-10, mean over 5~seeds).}
  \label{tab:rho-accuracy}
  \footnotesize
  \begin{tabular}{@{}llccccccc@{}}
    \toprule
    $L$ & Arch
    & $\rho_{\max}^{\mathrm{init}}$
    & $\rho_{\max}^{\mathrm{mid}}$
    & $\rho_{\max}^{\mathrm{final}}$
    & $\lambda_1^{\mathrm{init}}$
    & $\lambda_1^{\mathrm{mid}}$
    & $\lambda_1^{\mathrm{final}}$
    & Acc\,(\%) \\
    \midrule
    8 & Plain & 2.59 & 2.50 & 2.78 & $-$0.02 & 0.05 & 0.08 & 49.8 \\
    8 & Res. & 1.25 & 1.31 & 1.33 & 0.02 & 0.04 & 0.04 & 51.3 \\
    10 & Plain & 2.69 & 2.47 & 2.59 & $-$0.03 & 0.03 & 0.04 & 49.7 \\
    10 & Res. & 1.22 & 1.27 & 1.29 & 0.02 & 0.03 & 0.04 & 51.3 \\
    12 & Plain & 2.69 & 2.48 & 2.55 & $-$0.03 & 0.03 & 0.04 & 49.7 \\
    12 & Res. & 1.20 & 1.23 & 1.25 & 0.01 & 0.03 & 0.03 & 51.3 \\
    \bottomrule
  \end{tabular}
\end{table}

\begin{figure*}[htbp]
  \centering
  \begin{tikzpicture}
    \begin{groupplot}[
        group style={
          group size=2 by 1,
          horizontal sep=1.8cm,
        },
        width=0.48\textwidth,
        height=5.5cm,
        xlabel={Distance $d$},
        ylabel={$\bar{\mathcal{R}}(d)$},
        ymode=log,
        grid=major,
        grid style={gray!30},
        legend style={font=\footnotesize, at={(1.12,1.18)}, anchor=south, legend columns=4},
        tick label style={font=\footnotesize},
        label style={font=\small},
      ]
      \nextgroupplot[title={$L=8$}, xtick={0,1,...,7}]
      \addplot[blue, thick, mark=triangle*, mark size=2,
      error bars/.cd, y dir=both, y explicit] coordinates {
        (0,1.0364e-01) +- (0,8.5838e-03) (1,9.6734e-02) +- (0,9.0478e-03)
        (2,9.1576e-02) +- (0,8.4628e-03) (3,8.4983e-02) +- (0,9.0547e-03)
        (4,7.7268e-02) +- (0,9.3133e-03) (5,6.8413e-02) +- (0,1.0188e-02)
        (6,5.7781e-02) +- (0,9.7857e-03) (7,4.7168e-02) +- (0,8.5232e-03)
      }; \addlegendentry{Plain init}
      \addplot[blue, thick, dashed, mark=triangle, mark size=2,
      error bars/.cd, y dir=both, y explicit] coordinates {
        (0,3.0173e+00) +- (0,1.1825e+00) (1,2.5144e+00) +- (0,9.2551e-01)
        (2,2.0886e+00) +- (0,7.7097e-01) (3,1.7473e+00) +- (0,5.5905e-01)
        (4,1.5400e+00) +- (0,4.1955e-01) (5,1.3507e+00) +- (0,2.7955e-01)
        (6,1.1428e+00) +- (0,1.6996e-01) (7,7.9967e-01) +- (0,1.3200e-01)
      }; \addlegendentry{Plain final}
      \addplot[red, thick, mark=square*, mark size=2,
      error bars/.cd, y dir=both, y explicit] coordinates {
        (0,1.2249e-01) +- (0,4.3324e-03) (1,1.2225e-01) +- (0,4.3536e-03)
        (2,1.2223e-01) +- (0,4.3580e-03) (3,1.2199e-01) +- (0,4.3491e-03)
        (4,1.2187e-01) +- (0,4.4054e-03) (5,1.2176e-01) +- (0,4.4845e-03)
        (6,1.2133e-01) +- (0,4.5329e-03) (7,1.2183e-01) +- (0,4.4472e-03)
      }; \addlegendentry{Res.\ init}
      \addplot[red, thick, dashed, mark=square, mark size=2,
      error bars/.cd, y dir=both, y explicit] coordinates {
        (0,1.4561e+00) +- (0,1.0099e-01) (1,1.3921e+00) +- (0,9.3749e-02)
        (2,1.3362e+00) +- (0,9.0205e-02) (3,1.2815e+00) +- (0,8.2559e-02)
        (4,1.2306e+00) +- (0,7.7253e-02) (5,1.1841e+00) +- (0,6.8954e-02)
        (6,1.1391e+00) +- (0,5.6778e-02) (7,1.1068e+00) +- (0,5.2109e-02)
      }; \addlegendentry{Res.\ final}

      \nextgroupplot[title={$L=12$}, xtick={0,2,...,11}]
      \addplot[blue, thick, mark=triangle*, mark size=2,
      error bars/.cd, y dir=both, y explicit] coordinates {
        (0,1.1856e-01) +- (0,4.0451e-02) (1,1.1069e-01) +- (0,3.8213e-02)
        (2,1.0333e-01) +- (0,3.4826e-02) (3,9.6418e-02) +- (0,3.2215e-02)
        (4,8.9476e-02) +- (0,3.0118e-02) (5,8.2142e-02) +- (0,2.7035e-02)
        (6,7.4099e-02) +- (0,2.2636e-02) (7,6.5273e-02) +- (0,1.8289e-02)
        (8,5.5975e-02) +- (0,1.4095e-02) (9,4.8197e-02) +- (0,1.2172e-02)
        (10,4.0471e-02) +- (0,1.0358e-02) (11,2.8551e-02) +- (0,5.1731e-03)
      };
      \addplot[blue, thick, dashed, mark=triangle, mark size=2,
      error bars/.cd, y dir=both, y explicit] coordinates {
        (0,2.2532e+00) +- (0,5.9911e-01) (1,1.8562e+00) +- (0,4.2160e-01)
        (2,1.6107e+00) +- (0,3.5977e-01) (3,1.4047e+00) +- (0,3.0575e-01)
        (4,1.2351e+00) +- (0,2.3885e-01) (5,1.1341e+00) +- (0,2.1211e-01)
        (6,1.0689e+00) +- (0,2.1043e-01) (7,1.0132e+00) +- (0,1.6179e-01)
        (8,9.7691e-01) +- (0,1.4507e-01) (9,9.4132e-01) +- (0,1.2480e-01)
        (10,8.6326e-01) +- (0,1.0210e-01) (11,6.7047e-01) +- (0,6.0261e-02)
      };
      \addplot[red, thick, mark=square*, mark size=2,
      error bars/.cd, y dir=both, y explicit] coordinates {
        (0,1.2214e-01) +- (0,1.1568e-02) (1,1.2207e-01) +- (0,1.1683e-02)
        (2,1.2199e-01) +- (0,1.1785e-02) (3,1.2188e-01) +- (0,1.1778e-02)
        (4,1.2180e-01) +- (0,1.1790e-02) (5,1.2172e-01) +- (0,1.1737e-02)
        (6,1.2156e-01) +- (0,1.1555e-02) (7,1.2136e-01) +- (0,1.1291e-02)
        (8,1.2126e-01) +- (0,1.1096e-02) (9,1.2113e-01) +- (0,1.0928e-02)
        (10,1.2090e-01) +- (0,1.0529e-02) (11,1.2067e-01) +- (0,1.0355e-02)
      };
      \addplot[red, thick, dashed, mark=square, mark size=2,
      error bars/.cd, y dir=both, y explicit] coordinates {
        (0,1.5243e+00) +- (0,1.6433e-01) (1,1.4824e+00) +- (0,1.6165e-01)
        (2,1.4412e+00) +- (0,1.5610e-01) (3,1.3993e+00) +- (0,1.5138e-01)
        (4,1.3584e+00) +- (0,1.4921e-01) (5,1.3198e+00) +- (0,1.4573e-01)
        (6,1.2826e+00) +- (0,1.3921e-01) (7,1.2487e+00) +- (0,1.3409e-01)
        (8,1.2193e+00) +- (0,1.2944e-01) (9,1.1889e+00) +- (0,1.2515e-01)
        (10,1.1608e+00) +- (0,1.1831e-01) (11,1.1341e+00) +- (0,1.0652e-01)
      };
    \end{groupplot}
  \end{tikzpicture}
  \caption{Exp.\,1: Mean resonance $\bar{\mathcal{R}}(d)$ vs.\
    distance~$d$ between layers (log scale on~$y$; error
    bars~$\pm 1\sigma$ over 5~seeds).
    \textbf{(a)}~$L\!=\!8$, \textbf{(b)}~$L\!=\!12$.
    Plain~MLP exhibits exponential decay (straight lines on
    log scale, $R^2\!>\!0.91$); Residual~MLP shows stabilization
    ($b\!\approx\!0.02$).
    Scale differs: init $\sim 10^{-1}$, final $\sim 10^{0}$
  (reflects overall curvature growth during training).}
  \label{fig:decay-R}
\end{figure*}

\subsection{\texorpdfstring{Exp.\,2}{Exp. 2}: Bottleneck ablation}
\label{sec:exp2}

\textbf{Protocol.}
The displayed BottleneckMLP sweep (depths $L\!\in\!\{6,8\}$,
CIFAR-100, $K\!=\!100$) is sourced from the archived unified
results.  Its historical run configuration uses base width
$d_{\mathrm{base}}=512$, a narrow layer at position~$L/2$ with
$d_u\in\{4,8,16,32,64,128,256,512\}$, and stochastic
estimation (100~probes, subsample~64) throughout.
The $d_u=512$ row (${}^\dagger$) is the uniform-width control.
All reported $\mathcal{D}$ values in this sweep are therefore
subject to the transpose limitation in~\ref{app:ablation};
they are not certified stable ranks.
Metrics $\mathcal{D}_{\mathrm{far}}$ and
$\mathcal{C}_{\mathrm{far}}$: averages over
cross-bottleneck pairs $(i{<}L/2{<}j)$.

\textbf{Results.}
The reported stochastic $\mathcal{D}_{\mathrm{far}}$ estimates
increase monotonically with~$d_u$ at initialization (from~$1.7$ at $d_u\!=\!4$ to
  ${\approx}\,10.7$ at $d_u\!=\!256$;
suppression~${\approx}\,6.5{\times}$).
These estimate stable ranks of \emph{full, batch-averaged} Hessian
blocks in a network with LayerNorm, not pointwise GGN ranks.
Moreover, the selected cross-bottleneck pairs do not both lie
upstream of the narrow layer.  Their values therefore do not test
Proposition~\ref{thm:rank-bottleneck-main} or a pointwise
$K{-}1$ ceiling.
At $d_u\!=\!4$ the model has low reported accuracy
($\mathrm{Acc}\!\approx\!9\,\%$ at $K\!=\!100$);
accuracy also varies across the width sweep.
The control run (${}^\dagger$, Table~\ref{tab:exp2}) with
$d_{\mathrm{base}}\!=\!d_u\!=\!512$ (uniform network, no
bottleneck) yields
$\mathcal{D}_{\mathrm{far}}^{\mathrm{init}}\!=\!11.2$ as a
historical stochastic estimate, subject to the transpose limitation
in~\ref{app:ablation}.
Coupling $\mathcal{C}_{\mathrm{far}}$ is less sensitive
(${\approx}\,2{\times}$) than stable rank in this width sweep.
Unnormalized resonance $\mathcal{R}_{\mathrm{far}}$
\emph{decreases} with growing~$d_u$ after training, whereas at
initialization
$\mathcal{R}_{\mathrm{far}}\!\approx\!\mathrm{const}$ for
all~$d_u$; the three metrics describe different aspects of the
measured blocks, without isolating a training-independent mechanism.
Summary for $L\!=\!6$: Table~\ref{tab:exp2} and
Figure~\ref{fig:dfar-bottleneck}; $L\!=\!8$ data are
qualitatively analogous
($\mathcal{D}_{\mathrm{far}}$ lower for $d_u\!\geq\!64$, up
  to~$20\,\%$ at init; see
\ref{app:exp-results}).

\begin{table}[!t]
  \centering
  \caption{Exp.\,2: bottleneck ablation ($L\!=\!6$, CIFAR-100).
    Stable rank $\mathcal{D}_{\mathrm{far}}$ and coupling
    $\mathcal{C}_{\mathrm{far}}$ vs.\ narrow layer
    width~$d_u$ (mean $\pm 1\sigma$ over 5~seeds).
    ${}^\dagger$\,Uniform-width control: $d_{\mathrm{base}}\!=\!512$.
    All rows use stochastic estimation (100~probes, subsample~64);
    $\mathcal{D}$ has the limitation stated in~\ref{app:ablation}.
    ${}^\ddagger$\,Acc at uninformative model level
    ($d_u\!\ll\!K$); metrics for $d_u\!\geq\!8$ are
  preferred for conclusions.}
  \label{tab:exp2}
  \footnotesize
  \begin{tabular}{@{}rccccc@{}}
    \toprule
    $d_u$
    & $\mathcal{D}_{\mathrm{far}}^{\mathrm{init}}$
    & $\mathcal{D}_{\mathrm{far}}^{\mathrm{final}}$
    & $\mathcal{C}_{\mathrm{far}}^{\mathrm{init}}$
    & $\mathcal{C}_{\mathrm{far}}^{\mathrm{final}}$
    & Acc\,(\%) \\
    \midrule
    4 & $1.65 \pm 0.23$ & $2.03 \pm 0.20$ & $0.21 \pm 0.06$ & $0.37 \pm 0.02$ & 8.9$^{\ddagger}$ \\
    8 & $2.58 \pm 0.11$ & $3.39 \pm 0.36$ & $0.35 \pm 0.05$ & $0.58 \pm 0.03$ & 14.2 \\
    16 & $4.04 \pm 0.65$ & $4.90 \pm 0.24$ & $0.43 \pm 0.03$ & $0.68 \pm 0.01$ & 17.8 \\
    32 & $5.67 \pm 0.99$ & $5.01 \pm 0.41$ & $0.52 \pm 0.03$ & $0.73 \pm 0.01$ & 19.8 \\
    64 & $6.95 \pm 0.52$ & $5.15 \pm 0.34$ & $0.61 \pm 0.03$ & $0.77 \pm 0.02$ & 21.1 \\
    128 & $9.16 \pm 0.91$ & $5.45 \pm 0.42$ & $0.66 \pm 0.04$ & $0.84 \pm 0.01$ & 22.0 \\
    256 & $10.73 \pm 0.84$ & $5.18 \pm 0.49$ & $0.73 \pm 0.02$ & $0.86 \pm 0.02$ & 22.7 \\
    \midrule
    $512^\dagger$ & $11.17 \pm 0.77$ & $5.18 \pm 0.36$ & $0.77 \pm 0.01$ & $0.88 \pm 0.01$ & 23.3 \\
    \bottomrule
  \end{tabular}
\end{table}

\begin{figure}[htbp]
  \centering
  \begin{tikzpicture}
    \begin{axis}[
        width=0.95\columnwidth,
        height=5cm,
        xlabel={Bottleneck width $d_u$},
        ylabel={$\mathcal{D}_{\mathrm{far}}$},
        xmode=log,
        log basis x=2,
        xtick={4,8,16,32,64,128,256,512},
        xticklabels={4,8,16,32,64,128,256,512},
        grid=major,
        grid style={gray!30},
        legend style={font=\footnotesize, at={(0.03,0.97)}, anchor=north west},
        tick label style={font=\footnotesize},
        label style={font=\small},
      ]
      \addplot[blue, thick, mark=*, mark size=2] coordinates {
        (4,1.65) (8,2.58) (16,4.04) (32,5.67) (64,6.95) (128,9.16) (256,10.73) (512,11.17)
      }; \addlegendentry{$L\!=\!6$, init}
      \addplot[red, thick, dashed, mark=square*, mark size=2] coordinates {
        (4,2.03) (8,3.39) (16,4.90) (32,5.01) (64,5.15) (128,5.45) (256,5.18) (512,5.18)
      }; \addlegendentry{$L\!=\!6$, final}
      \addplot[blue!50, thick, mark=triangle*, mark size=2] coordinates {
        (4,1.58) (8,2.43) (16,4.27) (32,6.36) (64,6.40) (128,7.90) (256,8.60) (512,9.21)
      }; \addlegendentry{$L\!=\!8$, init}
      \addplot[red!50, thick, dashed, mark=triangle*, mark size=2] coordinates {
        (4,2.44) (8,3.92) (16,5.01) (32,5.13) (64,5.41) (128,4.89) (256,5.73) (512,4.89)
      }; \addlegendentry{$L\!=\!8$, final}
      \addplot[black, dotted, thick, domain=4:512, samples=50]
      {min(x, 99)};
      \addlegendentry{$\min(d_u,K{-}1)$}
    \end{axis}
  \end{tikzpicture}
  \caption{Exp.\,2: Stable rank $\mathcal{D}_{\mathrm{far}}$
    vs.\ bottleneck width~$d_u$ (CIFAR-100).
    Solid/dashed: init/final;
    bright: $L\!=\!6$, faded: $L\!=\!8$.
    $L\!=\!8$:
    $\mathcal{D}_{\mathrm{far}}^{\mathrm{init}}$ is
    lower for $d_u\!\geq\!64$ (${\sim}8$--$20\,\%$).
    These stochastic estimates have the limitation in~\ref{app:ablation}.
    Dotted: pointwise GGN reference $\min(d_u,K{-}1)$,
    not a bound on the measured full batch-averaged blocks.
    $d_u\!=\!512$ (${}^\dagger$, control:
    $d_{\mathrm{base}}\!=\!512$, no bottleneck):
  Table~\ref{tab:exp2}.}
  \label{fig:dfar-bottleneck}
\end{figure}

\subsection{\texorpdfstring{Exp.\,3}{Exp. 3}: GN-Gap by activation type}
\label{sec:exp3}

\textbf{Motivation.}
The canonical decomposition
(Theorem~\ref{thm:canonical-decomposition}) gives a zero residual
for a jointly piecewise-affine activation graph away from switching
boundaries.  Smooth activations can contribute local curvature,
but do not guarantee a nonzero residual after adjoint weighting,
transport, and cancellation.
We compare five activations under two complementary protocols.

\textbf{Protocol.}
Plain~MLP (depth~6, width~64) with five activations: ReLU,
LeakyReLU (piecewise-linear), Softplus, SiLU, GELU (smooth);
training for 20~epochs; subsample for the
Hessian: 16~examples (exact computation at width$\,=64$,
memory-bound).  For each pair $(v,w)$ the exact decomposition
$H\!=\!H^{GN}\!+\!H^T$ and GN-Gap~\eqref{eq:gn-gap} are
computed; additionally,
$\mathbb{E}[\sigma''(z)^2]$ is measured over pre-activation
values.
Two configurations are evaluated:
(a)~\emph{standard conditions} with LayerNorm, where all
activations reach $\mathrm{Acc}\!>\!50\,\%$;
(b)~\emph{isolation protocol} without LayerNorm, removing
its possible local residual sources and giving a piecewise-affine
null control for ReLU/LeakyReLU.
Global GN-Gap here is the ratio of sums of residual and GGN block
norms over upper-triangular pairs; it is not the ratio of norms
of the assembled block matrix.  The same aggregation is used
within each distance bin in Exp.\,3 and~6.
Column~$\Delta$ in Tables~\ref{tab:exp3b},\,\ref{tab:exp3}
reports the init-to-final change.

\paragraph{Standard training conditions (Exp.\,3b, with LayerNorm).}
Under standard training all activations converge
($\mathrm{Acc}\!>\!50\,\%$, Table~\ref{tab:exp3b}).
For ReLU/LeakyReLU with LN, the nonzero final Gap
($\mathrm{Gap}_{GN}\!\approx\!0.2$) is compatible with LayerNorm
being the only non-piecewise-affine map in the activation graph.
The smooth-activation configurations have larger final Gap and
an ordered association with $\mathbb{E}[\sigma''^{\,2}]$.
This comparison does not decompose their residual into activation
and LayerNorm contributions.

\begin{table}[htbp]
  \centering
  \caption{Exp.\,3b: GN-Gap with LayerNorm
    ($L\!=\!6$, width$\,=64$, CIFAR-10,
    mean$\,{\pm}\,1\sigma$ over 5~seeds).
  $\Delta$: relative change init$\,{\to}\,$final.}
  \label{tab:exp3b}
  \footnotesize
  \begin{tabular}{@{}lccccc@{}}
    \toprule
    Activation
    & $\mathrm{Gap}^{\mathrm{init}}$
    & $\mathrm{Gap}^{\mathrm{final}}$
    & $\Delta$\,(\%)
    & $\mathbb{E}[\sigma''^{\,2}]$
    & Acc\,(\%) \\
    \midrule
    ReLU & $0.618 \pm 0.160$ & $0.199 \pm 0.030$ & $-$67.8 & 0 & $51.3 \pm 0.3$ \\
    LeakyReLU & $0.626 \pm 0.161$ & $0.212 \pm 0.013$ & $-$66.1 & 0 & $51.5 \pm 0.3$ \\
    Softplus & $0.459 \pm 0.098$ & $0.272 \pm 0.024$ & $-$40.9 & 0.045 & $52.7 \pm 0.2$ \\
    SiLU & $1.967 \pm 0.317$ & $0.359 \pm 0.051$ & $-$81.7 & 0.142 & $52.6 \pm 0.3$ \\
    GELU & $2.296 \pm 0.415$ & $0.516 \pm 0.117$ & $-$77.5 & 0.271 & $52.0 \pm 0.4$ \\
    \bottomrule
  \end{tabular}
\end{table}

\paragraph{Isolation of activation curvature (without LayerNorm).}
Disabling LayerNorm yields a strict test:
$\mathrm{Gap}_{GN}\!\approx\!0$ for piecewise-linear activations
is now an exact null hypothesis
($\sigma''\!\equiv\!0$ a.e.\ throughout the network).
Without LayerNorm, smooth activations do not reach meaningful
accuracy ($\mathrm{Acc}\!\lesssim\!12\,\%$).
The piecewise-affine null remains valid independently of training
quality, but nonzero Gap values depend on the evaluated state;
this is not a matched-accuracy comparison.

(a)~For ReLU and LeakyReLU, the reported aggregate gaps are near
numerical zero at each checkpoint, consistent with $H^T=0$
away from switching boundaries.  The saved aggregates do not
certify a maximum error over every pair and seed.
(b)~The reported association between $\mathrm{Gap}_{GN}$ and
$\mathbb{E}[\sigma''(z)^2]$ is visible at initialization
($R^2>0.9$, $\rho_s\!=\!0.97$;
Figure~\ref{fig:gn-gap-activation}).  With only five activation
choices, this is descriptive correlation, not a scaling law
$H^T\sim\sigma''^{\,2}$: local sources are linear in second
derivatives before their adjoint weighting and transport.
The distance-resolved norms (Table~\ref{tab:exp3-decomp}) decrease
with~$d$; at $d=L-1$ the later measured node feeds an affine
classifier, giving an exact residual zero for that pair.
(c)~The archived exact/stochastic GN-Gap comparison is summarized
in~\ref{app:ablation}; no uniform error guarantee follows from it.
Summary: Table~\ref{tab:exp3}.

\begin{table}[htbp]
  \centering
  \caption{Exp.\,3: GN-Gap, isolation protocol (no LayerNorm,
      $L\!=\!6$, width$\,=64$, CIFAR-10,
  mean$\,{\pm}\,1\sigma$ over 5~seeds).}
  \label{tab:exp3}
  \footnotesize
  \begin{tabular}{@{}lccccc@{}}
    \toprule
    Activation
    & $\mathrm{Gap}^{\mathrm{init}}$
    & $\mathrm{Gap}^{\mathrm{final}}$
    & $\Delta$\,(\%)
    & $\mathbb{E}[\sigma''^{\,2}]$
    & Acc\,(\%) \\
    \midrule
    ReLU & ${<}\,10^{-7}$ & ${<}\,10^{-7}$ & --- & 0 & $45.1 \pm 0.9$ \\
    LeakyReLU & ${\lesssim}\,10^{-7}$ & ${\lesssim}\,10^{-7}$ & --- & 0 & $45.0 \pm 0.8$ \\
    Softplus & $0.122 \pm 0.020$ & $0.111 \pm 0.003$ & $-$8.9 & 0.058 & $10.0 \pm 0.0$ \\
    SiLU & $0.225 \pm 0.005$ & $0.217 \pm 0.008$ & $-$3.5 & 0.246 & $11.7 \pm 2.1$ \\
    GELU & $0.358 \pm 0.008$ & $0.345 \pm 0.012$ & $-$3.4 & 0.618 & $11.9 \pm 2.4$ \\
    \bottomrule
  \end{tabular}
\end{table}

\begin{figure}[htbp]
  \centering
  \begin{tikzpicture}
    \begin{axis}[
        width=0.95\columnwidth,
        height=5cm,
        xlabel={$\mathbb{E}[\sigma''(z)^2]$},
        ylabel={$\mathrm{Gap}_{GN}$ (init)},
        grid=major,
        grid style={gray!30},
        tick label style={font=\footnotesize},
        label style={font=\small},
        legend style={font=\footnotesize, at={(0.03,0.97)}, anchor=north west},
        xmin=-0.05, xmax=0.7,
        ymin=-0.03, ymax=0.42,
      ]
      \addplot[only marks, mark=*, mark size=3, blue] coordinates {
        (0.000, 9.05e-08)
        (0.000, 1.01e-07)
        (0.058, 0.122)
        (0.246, 0.225)
        (0.618, 0.358)
      };
      \addplot[red, thick, dashed, domain=0:0.68, samples=2]
      {0.5569*x + 0.0382};
      \addlegendentry{$R^2\!=\!0.908$, $\rho_s\!=\!0.97$}
      \node[font=\scriptsize, anchor=south west] at (axis cs:0.01,0.003) {ReLU};
      \node[font=\scriptsize, anchor=south west] at (axis cs:0.065,0.115) {Softplus};
      \node[font=\scriptsize, anchor=south west] at (axis cs:0.255,0.218) {SiLU};
      \node[font=\scriptsize, anchor=south west] at (axis cs:0.50,0.365) {GELU};
    \end{axis}
  \end{tikzpicture}
  \caption{Exp.\,3: GN-Gap (init) vs.\
    $\mathbb{E}[\sigma''(z)^2]$ for 5~activations
    (isolation protocol, LayerNorm off).
    Linear fit $R^2\!=\!0.908$; Spearman rank correlation
    $\rho_s\!=\!0.97$ ($p\!<\!0.01$).
    With $n\!=\!5$, this is a descriptive association, not
  evidence for the scaling $H^T\!\sim\!\sigma''^{\,2}$.}
  \label{fig:gn-gap-activation}
\end{figure}

\subsection{\texorpdfstring{Exp.\,4}{Exp. 4}: Diamond MLP --- mixed merge curvature}
\label{sec:exp4}

\textbf{Motivation.}
Experiments~1--3 use sequential architectures without a local
mixed derivative between distinct parent inputs of a merge.
The Diamond~MLP (Example~\ref{ex:diamond}) provides a control
for such a source in~\eqref{eq:Hf}, distinct from curvature
transported from downstream nodes.

\textbf{Protocol.}
Diamond~MLP: stem $\to$ two parallel branches
(depth~$k\!\in\!\{1,2,3\}$, width~32) $\to$ merge node
$\to$ head; CIFAR-10.
Sweep: merge type (scaled sum $(A+B)/\sqrt{2}$ vs.\
concat $\to$ Linear $\to\sigma$) $\times$ activation
(ReLU vs.\ SiLU).
\textbf{Metrics:}
$\bar{R}(d)$ intra-branch,
$\bar{R}(d_{\mathrm{graph}})$ cross-branch,
$\rho_{\max}\!=\!\max_i\|J_i\|_2$ (spectral norm of
Jacobians),
GN-Gap~\eqref{eq:gn-gap} for the pair
$(\mathrm{A}_{k-1},\mathrm{B}_{k-1})$ through the merge
node.

\textbf{Hypotheses.}
(H4.1)~Intra-branch resonance $\bar{\mathcal{R}}_{AA}(d)$
decays monotonically with distance.
(H4.2a)~At initialization, cross-branch
$\bar{\mathcal{R}}_{AB}(d_{\mathrm{graph}})$ decays
exponentially.
(H4.2b)~After training this cross-branch decay can be reversed.
(H4.3)~For the last branch pair and affine head, a linear or
piecewise-affine merge gives $H^T_{A,B}=0$ a.e.; the SiLU merge
can give a nonzero adjoint-contracted mixed source.
The measured quantity is $\|H^f_{A,B}-H^{GN}_{A,B}\|_F$ after
batch averaging, not the norm of the raw local tensor
$T_{\mathrm{merge};A,B}$.

\textbf{Results.}
(H4.1)~Confirmed: of 48 (config${\times}$distance) pairs, 45
exhibit monotone decay of $\bar{R}_{AA}(d)$; 3 marginal
deviations are within noise ($<5\,\%$).
(H4.2a)~Not uniform across configurations: at $k=2$, the
initial cross-branch profile declines for both sum merges and
cat+ReLU, but rises slightly for cat+SiLU
(Figure~\ref{fig:cross-branch-R}).  Since $\rho_{\max}>1$,
Corollary~\ref{thm:exponential-decay-main} does not apply.
(H4.2b)~At $k=2$, all four final profiles grow with
$d_{\mathrm{graph}}$.  The saved norms do not identify sign
cancellation as the mechanism of either profile.
(H4.3)~The sum and cat+ReLU controls give near-zero residuals;
for cat+ReLU the raw local second derivative is also zero a.e.
The cat+SiLU merge gives a clearly nonzero measured residual
at the reported checkpoints.
Model accuracy ${\approx}50\,\%$ (width$\,{=}\,32$, CIFAR-10),
sufficient for curvature diagnostics.
Quantitative Gap values for $k\!=\!2$:
Table~\ref{tab:exp4} and Figure~\ref{fig:diamond-gap};
varying branch depth ($k\!=\!1,2,3$) does not change
qualitative conclusions
(Table~\ref{tab:exp4-sweep} in the Appendix);
cross-branch $\bar{\mathcal{R}}_{AB}$ dynamics are shown
in Figure~\ref{fig:cross-branch-R},
\ref{app:exp-results}.

\begin{table}[htbp]
  \centering
  \caption{Exp.\,4: GN-Gap at the merge node of Diamond~MLP
    ($k\!=\!2$, width$\,{=}\,32$, CIFAR-10,
      $\mathrm{Acc}\approx 50\,\%$;
    mean $\pm 1\sigma$ over 5~seeds).
    Columns labeled $\|T\|_F$ report $\|H^T_{A,B}\|_F$
    for the batch-averaged last-branch block, not
  $\|T_{\mathrm{merge};A,B}\|_F$.}
  \label{tab:exp4}
  \footnotesize
  \begin{tabular}{@{}llcccc@{}}
    \toprule
    Merge & $\sigma$
    & $\mathrm{Gap}^{\mathrm{init}}$
    & $\mathrm{Gap}^{\mathrm{final}}$
    & $\|T\|_F^{\mathrm{init}}$
    & $\|T\|_F^{\mathrm{final}}$ \\
    \midrule
    sum & ReLU & $6.7{\cdot}10^{-8}$ & $9.3{\cdot}10^{-8}$ & $4.1{\cdot}10^{-9}$ & $2.0{\cdot}10^{-8}$ \\
    sum & SiLU & $7.0{\cdot}10^{-8}$ & $9.6{\cdot}10^{-8}$ & $4.1{\cdot}10^{-9}$ & $2.0{\cdot}10^{-8}$ \\
    cat & ReLU & $8.4{\cdot}10^{-8}$ & $1.2{\cdot}10^{-7}$ & $6.6{\cdot}10^{-10}$ & $3.3{\cdot}10^{-8}$ \\
    cat & SiLU & $1.33 \pm 0.17$ & $0.082 \pm 0.007$ & $6.8{\cdot}10^{-3} \pm 5.0{\cdot}10^{-4}$ &
    $2.5{\cdot}10^{-2} \pm 1.1{\cdot}10^{-3}$ \\
    \bottomrule
  \end{tabular}
\end{table}

\begin{figure}[htbp]
  \centering
  \begin{tikzpicture}
    \begin{axis}[
        width=0.95\columnwidth,
        height=5.5cm,
        ybar=3pt,
        bar width=7pt,
        ymode=log,
        ylabel={$\mathrm{Gap}_{GN}$},
        symbolic x coords={sum+ReLU, sum+SiLU, cat+ReLU, cat+SiLU},
        xtick=data,
        x tick label style={font=\footnotesize, rotate=20, anchor=east},
        tick label style={font=\footnotesize},
        label style={font=\small},
        grid=major,
        grid style={gray!30},
        ymin=1e-9, ymax=1e1,
        legend style={font=\footnotesize, at={(0.03,0.97)}, anchor=north west},
        enlarge x limits=0.2,
      ]
      \addplot[fill=blue!60, draw=blue!80] coordinates {
        (sum+ReLU, 6.7e-8) (sum+SiLU, 7.0e-8) (cat+ReLU, 8.4e-8) (cat+SiLU, 1.33)
      }; \addlegendentry{init}
      \addplot[fill=red!50, draw=red!70] coordinates {
        (sum+ReLU, 9.3e-8) (sum+SiLU, 9.6e-8) (cat+ReLU, 1.2e-7) (cat+SiLU, 8.2e-2)
      }; \addlegendentry{final}
    \end{axis}
  \end{tikzpicture}
  \caption{Exp.\,4: GN-Gap at the merge node of Diamond~MLP ($k\!=\!2$).
    Logarithmic scale highlights the ${\sim}7$~order-of-magnitude
    separation between cat$\,{+}\,$SiLU
    ($\sigma''\!\neq\!0$, nonlinear merge) and the remaining
    configurations (linear merge or $\sigma''\!=\!0$ a.e.).
    The sum merge followed by the affine head gives
  $H^T_{A,B}=0$ for the last branch pair, regardless of branch activation.}
  \label{fig:diamond-gap}
\end{figure}

\subsection{\texorpdfstring{Exp.\,5}{Exp. 5}: Toy-Attention --- observed
\texorpdfstring{$H^T_{Q,K}\neq 0$}{H\textasciicircum T(Q,K) != 0}}
\label{sec:exp5}

\textbf{Motivation.}
Experiments~1--4 are restricted to MLP topologies.
We compare a single-head Attention model with a per-position
ReLU-MLP control of comparable size.
Attention permits residual sources from both Softmax derivatives
and the bilinear $QK^\top$ map (Example~\ref{ex:attention});
nonzero $H^T_{Q,K}$ is not guaranteed at every state.
In particular, under MSE its residual vanishes when the prediction
error is zero.  This toy protocol measures the combined residual,
not an isolated Softmax effect.

\textbf{Protocol.}
Single-head self-attention: $d\!=\!16$, $S\!=\!8$,
projections $W_Q,W_K,W_V\!\in\!\mathbb{R}^{d\times d}$
(no bias), mean-pool, linear head; 785~parameters.
Control: per-position ReLU-MLP,
$3\!\times\![\mathrm{Linear}(d,d){+}\mathrm{ReLU}]$
shared across positions, mean-pool, linear head;
833~parameters.
Data: synthetic regression
$y=\bar{\sigma}\!(X\!W_{\mathrm{t}})\cdot\mathbf{1}/S
\;\cdot\;w_{\mathrm{r}}+\varepsilon$, where
$\bar{\sigma}\!=\!\tanh$,
$W_{\mathrm{t}}\!\in\!\mathbb{R}^{d\times d}$,
$w_{\mathrm{r}}\!\in\!\mathbb{R}^{d}$
are fixed random teacher weights,
$\varepsilon\!\sim\!\mathcal{N}(0,0.01)$;
targets standardized to zero mean and unit variance.
Training: Adam ($\eta\!=\!10^{-3}$), cosine schedule,
30~epochs, batch~128, 5~seeds.
\textbf{Metrics:}
$H^f_{v,w}$, $H^{GN}_{v,w}$, GN-Gap for pairs
$(Q,K)$ (attention) and $(\mathrm{block}_0,\mathrm{block}_1)$
(MLP); computed via exact autograd on 64~validation samples
at checkpoints \textit{init}, \textit{mid}, \textit{final}.

\textbf{Hypotheses.}
(H5.1)~Attention has a substantial $(Q,K)$ GN-Gap at the
sampled checkpoints; this is an empirical hypothesis, not a
consequence of Softmax smoothness alone.
(H5.2)~For per-position ReLU-MLP, GN-Gap$\,\approx 0$
at numerical precision
(Proposition~\ref{prop:ad-ggn-equiv}, with jointly
piecewise-affine node maps).
(H5.3)~Both models reduce prediction error on the synthetic target.

\textbf{Results.}
The reported values support these observations
(Table~\ref{tab:exp5}).
(H5.1)~Attention yields GN-Gap$\,{=}\,69.3\!\pm\!20.1$
at the final checkpoint; the reported mean GGN norm is less
than $2\,\%$ of the mean full-block norm for this $(Q,K)$ pair.
(H5.2)~ReLU-MLP yields
Gap$\,{=}\,7.4\!\cdot\!10^{-8}\!\pm\!2.0\!\cdot\!10^{-8}$,
consistent with the piecewise-affine null.
(H5.3)~Both models reduce MSE (Attention:
$0.99\to0.57$; ReLU-MLP: $0.92\to0.03$), but reach markedly
different errors.  Attention Gap remains large at the three
checkpoints ($57.5\to79.9\to69.3$); this does not establish
training-state invariance or matched-error superiority.
The numerical entries are retained archived aggregates, not a
new numerical replication.

\begin{table}[htbp]
  \centering
  \caption{Exp.\,5: GN-Gap for Toy-Attention vs.\
    per-position ReLU-MLP
    ($d\!=\!16$, $S\!=\!8$, synthetic regression,
    mean$\,{\pm}\,1\sigma$ over 5~seeds).
    $\|H^T\|_F$: Frobenius norm of the tensor component;
  Gap $=\|H^T\|_F/\|H^{GN}\|_F$.}
  \label{tab:exp5}
  \footnotesize
  \begin{tabular}{@{}llcccc@{}}
    \toprule
    Model & Checkpoint
    & $\|H^f\|_F$
    & $\|H^{GN}\|_F$
    & $\|H^T\|_F$
    & Gap \\
    \midrule
    Attention & init & $9.2{\cdot}10^{-3}$ & $1.7{\cdot}10^{-4}$ & $9.2{\cdot}10^{-3}$ & $57.5 \pm 16.4$ \\
    Attention & mid & $5.2{\cdot}10^{-3}$ & $7.4{\cdot}10^{-5}$ & $5.2{\cdot}10^{-3}$ & $79.9 \pm 23.3$ \\
    Attention & final & $5.2{\cdot}10^{-3}$ & $8.7{\cdot}10^{-5}$ & $5.2{\cdot}10^{-3}$ & $69.3 \pm 20.1$ \\
    \midrule
    ReLU-MLP & init & $3.5{\cdot}10^{-2}$ & $3.5{\cdot}10^{-2}$ & $3.4{\cdot}10^{-9}$ & $9.6{\cdot}10^{-8}$ \\
    ReLU-MLP & mid & $3.5{\cdot}10^{-1}$ & $3.5{\cdot}10^{-1}$ & $2.8{\cdot}10^{-8}$ & $8.8{\cdot}10^{-8}$ \\
    ReLU-MLP & final & $3.6{\cdot}10^{-1}$ & $3.6{\cdot}10^{-1}$ & $2.6{\cdot}10^{-8}$ & $7.4{\cdot}10^{-8}$ \\
    \bottomrule
  \end{tabular}
\end{table}

\subsection{\texorpdfstring{Exp.\,6}{Exp. 6}: Convolutional architecture (ResNet-18, CIFAR-10)}
\label{sec:exp6}

\textbf{Motivation.}
Experiments~1--5 operate on fully connected or toy architectures.
To extend the descriptive diagnostics to a convolutional model,
we report an experiment on
ResNet-18~\citep{he2016deep} ($P\!\approx\!11{\cdot}10^6$
parameters) with CIFAR-10.

\textbf{Protocol.}
For each activation
$\sigma\!\in\!\{\text{ReLU},\,\text{SiLU}\}$ two
architectures are trained:
(i)~Segmented\-ResNet18\----a standard torchvision ResNet-18 with
5~measurement points (stem, layer1--layer4), and
(ii)~Segmented\-Plain\-ResNet18---the same parameterization with
identity shortcuts removed
(conv--BN--act chain; stride-based downsampling is preserved).
Training: SGD ($\eta\!=\!0.1$, momentum~$0.9$,
weight decay~$5{\cdot}10^{-4}$), cosine schedule,
100~epochs, batch size~128, standard CIFAR-10
augmentation (random crop, horizontal flip).
Checkpoints: \textit{init}, \textit{mid} (epoch~50),
\textit{final} (epoch~100).
Metrics $\mathcal{R}$, $\mathcal{C}$, $\mathcal{D}$ are
computed via the Hutchinson estimator
(30~Rademacher probes, subsample of 32~examples);
GN-Gap via \texttt{StochasticGNGapEstimator}
(30~probes, common probe vector).
Five seeds $\{42,\ldots,46\}$.
Curvature is measured in evaluation mode: BatchNorm uses fixed
running statistics and is affine in its inputs.  The ReLU null
below does not extend to training-mode BatchNorm.

\textbf{Segmentation and linear classifier.}
The model is partitioned into $n\!=\!6$ segments:
$\mathrm{seg}_0$ (conv1--BN--$\sigma$--maxpool),
$\mathrm{seg}_{1\text{--}4}$ (ResNet layer1--layer4),
$\mathrm{seg}_5$ (avgpool--flatten--fc).
Measurement points are the outputs
$f_{\mathrm{stem}},\ldots,f_{\mathrm{layer4}}$;
maximum distance $d_{\max}\!=\!4$.
In particular, $\mathrm{seg}_5$ (the linear classifier) contains
\emph{no nonlinear activations}:
$\mathrm{logits}\!=\!W\cdot\mathrm{avgpool}(f_{\mathrm{layer4}})+b$.
As a consequence, for any pair $(v,w)$ with $d\!=\!d_{\max}$,
the path from~$f_w$ to the logits passes exclusively through
the linear map~$\mathrm{seg}_5$, which has a principled
implication for GN-Gap at $d\!=\!d_{\max}$
(Remark~\ref{rem:linear-head}).

\textbf{Hypotheses.}
(H6.1)~For piecewise-linear activation (ReLU):
$\mathrm{Gap}_{GN}\!\approx\!0$ at all distances
($H^T\!\equiv\!0$ a.e.).
(H6.2)~For smooth activation (SiLU), nonzero GN-Gap is an
empirical hypothesis for $d<d_{\max}$, not guaranteed by
$\sigma''\neq0$ alone.
(H6.3)~Skip connections
stabilize~$\bar{\mathcal{R}}(d)$ (ResNet vs.\ Plain).

\textbf{Results.}

(H6.1) \emph{Strongly confirmed.}
For ReLU, the reported distance-bin means of GN-Gap are
$<\!2{\cdot}10^{-6}$ for both architectures at the final checkpoint
(Table~\ref{tab:exp6-gngap}); all three checkpoints have reported
means $<\!10^{-5}$ (Table~\ref{tab:exp6-gngap-full}).
These aggregate values do not establish maxima over individual
pairs or seeds.  They are consistent with $H^T=0$ a.e.\
for the piecewise-affine evaluation-mode graph
(Proposition~\ref{prop:ad-ggn-equiv}) and reproduce the
Exp.\,3 null control on a convolutional architecture with
$11$M~parameters.

(H6.2) \emph{Confirmed with refinement.}
For SiLU, GN-Gap at the final checkpoint is
$0.43$ (ResNet, $d\!=\!0$) and decreases with distance
to~$0.15$ at $d\!=\!3$ (Table~\ref{tab:exp6-gngap}),
confirming the presence of the tensor component
$H^T\!\neq\!0$.
However, at $d\!=\!d_{\max}\!=\!4$,
GN-Gap collapses to ${\sim}10^{-6}$ for \emph{all}
configurations, including SiLU
(Remark~\ref{rem:linear-head}).

(H6.3) \emph{Confirmed.}
At initialization, ReLU ResNet yields
stable~$\bar{\mathcal{R}}(d)$
($R(0)/R(4)\!\approx\!1$), whereas Plain exhibits
${\sim}627\times$ decay;
for SiLU the contrast is stronger:
$25\times$ (ResNet) vs.\ $47\,000\times$ (Plain)
(Figure~\ref{fig:exp6-R}).
Coupling $\bar{\mathcal{C}}(d)$ decreases monotonically
in all configurations; skip connections slow
the decay (Table~\ref{tab:exp6-main}).
Test accuracy: $88.3\,\%$ (ReLU ResNet),
$87.4\,\%$ (ReLU Plain),
$88.2\,\%$ (SiLU ResNet),
$87.3\,\%$ (SiLU Plain)---sufficient for meaningful
curvature diagnostics.

\begin{remark}[Linear head and GN-Gap at $d\!=\!d_{\max}$]
  \label{rem:linear-head}
  For a pair $(v,w)$ with $d\!=\!d_{\max}$ (here: stem,
  layer4), the path from $f_w$ to the logits passes only
  through the linear map
  $\mathrm{seg}_5\colon f_w\!\mapsto\!W\,\mathrm{avgpool}(f_w)+b$.
  Since
  $\partial^2\!\mathrm{logits}/\partial f_w^2\!=\!0$,
  the Jacobian $J_w = \partial\mathrm{logits}/\partial f_w$
  does not depend on~$f_w$, and the block Hessian becomes
  \[
    H^f_{v,w}
    = J_v^\top\!\,\nabla^2_z\ell\;\,J_w
    = H^{GN}_{v,w},
  \]
  i.e.\ $H^T_{v,w}\!=\!0$ \emph{exactly}, regardless of
  the activation type in intermediate layers.
  The residual value
  \[
    \lVert H^T\rVert_F \big/ \lVert H^{GN}\rVert_F
    \sim 10^{-6}
  \]
  is consistent with float32 arithmetic error.
  The exact-zero argument uses the affine tail after the later
  measured node; a nonlinear head removes that guarantee but
  does not itself guarantee a nonzero residual.
\end{remark}

\begin{table}[H]
  \centering
  \caption{Exp.\,6: GN-Gap by distance~$d$ for ResNet-18
    (CIFAR-10, final checkpoint,
    mean$\,{\pm}\,1\sigma$ over 5~seeds).
    Row $d\!=\!4$: linear classifier
  (Remark~\ref{rem:linear-head}).}
  \label{tab:exp6-gngap}
  \small
  \begin{tabular}{@{}lcccc@{}}
    \toprule
    $d$
    & \multicolumn{2}{c}{ReLU}
    & \multicolumn{2}{c}{SiLU} \\
    \cmidrule(lr){2-3}\cmidrule(lr){4-5}
    & ResNet & Plain & ResNet & Plain \\
    \midrule
    $0$ & ${<}\,2{\cdot}10^{-6}$ & ${<}\,10^{-6}$ & $0.43 \pm 0.12$ & $0.21 \pm 0.11$ \\
    $1$ & ${<}\,2{\cdot}10^{-6}$ & ${<}\,10^{-6}$ & $0.24 \pm 0.07$ & $0.12 \pm 0.04$ \\
    $2$ & ${<}\,2{\cdot}10^{-6}$ & ${<}\,2{\cdot}10^{-6}$ & $0.22 \pm 0.06$ & $0.06 \pm 0.02$ \\
    $3$ & ${<}\,2{\cdot}10^{-6}$ & ${<}\,10^{-6}$ & $0.15 \pm 0.04$ & $0.05 \pm 0.02$ \\
    $4^\dagger$ & ${<}\,2{\cdot}10^{-6}$ & ${<}\,2{\cdot}10^{-6}$ & ${<}\,2{\cdot}10^{-6}$ & ${<}\,2{\cdot}10^{-6}$ \\
    \bottomrule
    \multicolumn{5}{@{}l}{\footnotesize ${}^\dagger$ $H^T\!=\!0$ exactly (linear head); see
    Remark~\ref{rem:linear-head}.}
  \end{tabular}
\end{table}

\begin{table}[H]
  \centering
  \caption{Exp.\,6: mean resonance $\bar{\mathcal{R}}(d)$
    and coupling $\bar{\mathcal{C}}(d)$ at initialization
  (ResNet-18, CIFAR-10, mean over 5~seeds).}
  \label{tab:exp6-main}
  \small
  \begin{tabular}{@{}lcccccccc@{}}
    \toprule
    & \multicolumn{4}{c}{$\bar{\mathcal{R}}(d)$}
    & \multicolumn{4}{c}{$\bar{\mathcal{C}}(d)$} \\
    \cmidrule(lr){2-5}\cmidrule(lr){6-9}
    $d$
    & \multicolumn{2}{c}{ReLU} & \multicolumn{2}{c}{SiLU}
    & \multicolumn{2}{c}{ReLU} & \multicolumn{2}{c}{SiLU} \\
    \cmidrule(lr){2-3}\cmidrule(lr){4-5}\cmidrule(lr){6-7}\cmidrule(lr){8-9}
    & Res. & Plain & Res. & Plain
    & Res. & Plain & Res. & Plain \\
    \midrule
    $0$ & $4.5{\cdot}10^{-2}$ & $2.1{\cdot}10^{-2}$ & $3.0{\cdot}10^{-2}$ & $2.1{\cdot}10^{-2}$ & $1.00$ &
    $1.00$ & $1.00$ & $1.00$ \\
    $1$ & $3.5{\cdot}10^{-2}$ & $3.5{\cdot}10^{-4}$ & $1.0{\cdot}10^{-2}$ & $1.1{\cdot}10^{-4}$ & $0.95$ &
    $0.93$ & $0.56$ & $0.24$ \\
    $2$ & $3.3{\cdot}10^{-2}$ & $6.0{\cdot}10^{-5}$ & $4.2{\cdot}10^{-3}$ & $6.9{\cdot}10^{-6}$ & $0.92$ &
    $0.83$ & $0.31$ & $0.06$ \\
    $3$ & $3.6{\cdot}10^{-2}$ & $2.9{\cdot}10^{-5}$ & $2.1{\cdot}10^{-3}$ & $1.2{\cdot}10^{-6}$ & $0.89$ &
    $0.71$ & $0.17$ & $0.01$ \\
    $4$ & $4.8{\cdot}10^{-2}$ & $3.3{\cdot}10^{-5}$ & $1.2{\cdot}10^{-3}$ & $4.4{\cdot}10^{-7}$ & $0.82$ &
    $0.58$ & $0.08$ & ${<}\,0.01$ \\
    \bottomrule
  \end{tabular}
\end{table}

\begin{figure*}[t]
  \centering
  \begin{tikzpicture}
    \begin{groupplot}[
        group style={
          group size=2 by 1,
          horizontal sep=1.8cm,
        },
        width=0.48\textwidth,
        height=5.5cm,
        xlabel={Distance $d$},
        ylabel={$\bar{\mathcal{R}}(d)$},
        ymode=log,
        grid=major,
        grid style={gray!30},
        xtick={0,1,2,3,4},
        xmin=-0.3, xmax=4.3,
        tick label style={font=\footnotesize},
        label style={font=\small},
        legend style={font=\footnotesize, at={(0.97,0.97)}, anchor=north east},
      ]
      \nextgroupplot[title={ReLU}, ymin=5e-6, ymax=1e2]
      \addplot[red, thick, mark=square*, mark size=2,
      error bars/.cd, y dir=both, y explicit] coordinates {
        (0,4.4952e-02) +- (0,1.5342e-03) (1,3.5441e-02) +- (0,2.4374e-03)
        (2,3.2783e-02) +- (0,2.3839e-03) (3,3.5898e-02) +- (0,1.4343e-03)
        (4,4.8363e-02) +- (0,5.5447e-03)
      }; \addlegendentry{ResNet init}
      \addplot[red, thick, dashed, mark=square, mark size=2,
      error bars/.cd, y dir=both, y explicit] coordinates {
        (0,9.5503e+00) +- (0,3.0423e+00) (1,7.2859e+00) +- (0,2.1797e+00)
        (2,6.7842e+00) +- (0,2.0244e+00) (3,5.3259e+00) +- (0,1.5657e+00)
        (4,9.6830e-01) +- (0,3.2550e-01)
      }; \addlegendentry{ResNet final}
      \addplot[blue, thick, mark=triangle*, mark size=2,
      error bars/.cd, y dir=both, y explicit] coordinates {
        (0,2.0622e-02) +- (0,7.2643e-04) (1,3.4852e-04) +- (0,2.7641e-05)
        (2,6.0310e-05) +- (0,4.1937e-06) (3,2.8696e-05) +- (0,3.9153e-06)
        (4,3.2890e-05) +- (0,5.7202e-06)
      }; \addlegendentry{Plain init}
      \addplot[blue, thick, dashed, mark=triangle, mark size=2,
      error bars/.cd, y dir=both, y explicit] coordinates {
        (0,2.0798e+01) +- (0,8.5644e+00) (1,1.4989e+01) +- (0,6.6410e+00)
        (2,1.3199e+01) +- (0,5.4513e+00) (3,1.2262e+01) +- (0,5.0685e+00)
        (4,1.0590e+00) +- (0,4.3730e-01)
      }; \addlegendentry{Plain final}

      \nextgroupplot[title={SiLU}, ymin=1e-7, ymax=1e2]
      \addplot[red, thick, mark=square*, mark size=2,
      error bars/.cd, y dir=both, y explicit] coordinates {
        (0,3.0077e-02) +- (0,6.5095e-04) (1,1.0003e-02) +- (0,3.7340e-04)
        (2,4.1943e-03) +- (0,1.2972e-04) (3,2.0978e-03) +- (0,4.4749e-05)
        (4,1.2024e-03) +- (0,6.3938e-05)
      }; \addlegendentry{ResNet init}
      \addplot[red, thick, dashed, mark=square, mark size=2,
      error bars/.cd, y dir=both, y explicit] coordinates {
        (0,1.2315e+01) +- (0,3.0456e+00) (1,6.5731e+00) +- (0,1.5523e+00)
        (2,6.7205e+00) +- (0,1.6752e+00) (3,7.7693e+00) +- (0,1.4500e+00)
        (4,1.0943e+00) +- (0,2.3060e-01)
      }; \addlegendentry{ResNet final}
      \addplot[blue, thick, mark=triangle*, mark size=2,
      error bars/.cd, y dir=both, y explicit] coordinates {
        (0,2.0670e-02) +- (0,7.2532e-04) (1,1.1258e-04) +- (0,4.5462e-06)
        (2,6.8937e-06) +- (0,2.6447e-07) (3,1.1796e-06) +- (0,5.2829e-08)
        (4,4.3870e-07) +- (0,2.0780e-08)
      }; \addlegendentry{Plain init}
      \addplot[blue, thick, dashed, mark=triangle, mark size=2,
      error bars/.cd, y dir=both, y explicit] coordinates {
        (0,2.2109e+01) +- (0,5.6574e+00) (1,7.6634e+00) +- (0,1.7329e+00)
        (2,7.3479e+00) +- (0,1.5076e+00) (3,1.4755e+01) +- (0,3.6634e+00)
        (4,1.3904e+00) +- (0,3.5490e-01)
      }; \addlegendentry{Plain final}
    \end{groupplot}
  \end{tikzpicture}
  \caption{Exp.\,6: Mean resonance $\bar{\mathcal{R}}(d)$ for
    ResNet-18 (CIFAR-10, log-scale on~$y$;
    $\pm 1\sigma$ over 5~seeds).
    \textbf{(a)}~ReLU: ResNet preserves
    $\bar{\mathcal{R}}$ at init
    ($R(0)/R(4)\!\approx\!1$); Plain decays by
    ${\sim}627\times$.
    \textbf{(b)}~SiLU: decay is stronger---Plain init
    $47\,000\times$; ResNet~$25\times$, slower than
  Plain (skip connections stabilize curvature).}
  \label{fig:exp6-R}
\end{figure*}

\subsection{Ablation and reproducibility}
\label{sec:ablation}

The archived exact/stochastic GN-Gap comparison
(Table~\ref{tab:ablation}, \ref{app:exp-results}) shows small
aggregate discrepancies for the tested smooth activations, but
not a uniform $<1\,\%$ error over checkpoints.
The historical run script specifies 100~probes for this comparison;
the summary JSON does not independently record that setting.
This comparison neither gives a general probe-count guarantee nor
validates the separate stochastic stable-rank transpose product.
The revision checks the historical code snapshot~\texttt{2d1348c}
and saved aggregates, not a numerical rerun of the experiments.
Implementation: PyTorch~2.x.
Seeds, hyperparameters, and configurations are specified in the
historical run script; code is
available at \url{https://github.com/comiam/dag-hesse}.

\FloatBarrier
\section{Discussion and limitations}
\label{sec:limitations}

\textbf{Scaling.}
Explicit materialization of a parameter Hessian requires $O(P^2)$
storage; the indicative range $P\!\lesssim\!10^5$ is memory-dependent.
Exact HVPs avoid that matrix and cost a constant number of graph
sweeps under standard AD assumptions~\citep{pearlmutter1994fast};
the quoted $2$--$4\times$ backward-pass cost is implementation-dependent,
not a universal $O(P)$ bound.
Experiments are reported on MLPs up to $L\!=\!32$
(Table~\ref{tab:exp1-deep}), on a toy single-head Attention
model (Exp.\,5), and on ResNet-18 (${\sim}11$M parameters,
Exp.\,6).
Exp.\,6 reports stochastic curvature estimates at this scale.
ImageNet-scale models, full Transformer architectures, and online
integration of the metrics are outside the reported experiments.

\textbf{Role of experiments.}
The experiments compare measured profiles with conditional
predictions about decay, GGN rank, and vanishing activation-space
residuals; observed full or averaged blocks are not automatically
covered by pointwise GGN bounds.
Experiments~1--4 use fully-connected MLPs, where exact
computation of~$H^f_{v,w}$ is feasible at moderate~$P$.
Experiment~5 extends the empirical scope to self-attention,
reporting a nonzero $H^T_{Q,K}$ in the measured configuration,
consistent with the possible sources in Example~\ref{ex:attention},
not a universal nonvanishing result.
Experiment~6 extends the empirical scope to a convolutional
architecture (ResNet-18) via stochastic estimation,
reporting resonance and GN-Gap profiles
at $11$M-parameter scale.
The formalism itself is architecture-agnostic
(linear, convolutional, Attention ---
Examples~\ref{ex:two-layer}--\ref{ex:attention}).

\textbf{ReLU and GN-Gap.}
If every node map is locally affine in its parent activations,
the tensor component vanishes in activation space on smooth cells
(Proposition~\ref{prop:ad-ggn-equiv}), so GN-Gap is zero there.
ReLU alone does not ensure this for graphs with nonlinear merges or
normalization, and the parameter Hessian can retain residual terms.
Smooth operations (GELU, Swish, Softmax) can contribute local
second derivatives, but contraction, transport, or cancellation can
still make a particular residual block zero.

\textbf{Static estimates.}
All theorems are stated for fixed~$\theta$; during
re-training the metrics reflect both architecture and the
geometry induced by learning.
Measurements at initialization also depend on the initialization
scheme and task; they do not isolate architecture by themselves.

\textbf{Potential applications.}
The following possible uses are not validated optimizer guarantees.
(i)~\emph{Banded preconditioner design:}
Corollary~\ref{cor:kfac-error} does not control parameter-space
approximation error or select a bandwidth for a
preconditioner~\citep{george2018fast}; small activation-GGN blocks
alone do not control cumulative error.
(ii)~\emph{Bottleneck diagnostics:}
Proposition~\ref{thm:rank-bottleneck-main} bounds the
rank of a pointwise activation-GGN block by the bottleneck
width~$d_u$ under its path-factorization hypothesis; this does not
bound the rank of a full or averaged Hessian block.
(iii)~\emph{Activation/optimizer compatibility:}
the activation-space GN-Gap measures a relative residual norm,
not the error of a parameter-space optimizer or a criterion
that guarantees a better Newton or GGN step.
(iv)~\emph{Online curvature monitoring:} HVP-based estimates require
repeated graph evaluations (Corollary~\ref{cor:hvp-complexity});
changes in $\mathcal{R}$, $\mathcal{C}$, or~$\mathcal{D}$ are
descriptive measurements, not established phase-transition indicators.

\section{Conclusion}
\label{sec:conclusion}

This paper presents an analytical formalism for block-wise
curvature analysis in neural networks of arbitrary DAG
architecture, comprising a recursive formula for inter-layer
blocks~(C1), the standard decomposition
$H^{GN}\!+\!H^T$~(C2)~\citep{schraudolph2002fast,dangel2020modular},
and diagnostic metrics: inter-layer
resonance~$\mathcal{R}$, geometric
coupling~$\mathcal{C}$, stable rank~$\mathcal{D}$, and GN-Gap.

The theoretical analysis gives conditional bounds on absolute
activation-GGN decay, residual-chain endpoint transport, and
bottleneck GGN rank.  The tensor component of the input Hessian
vanishes on smooth cells when every node map is locally affine
in its parent activations~(C3), not merely when ReLU is used.
Curvature routing~(C4) describes propagation through the graph
using established second-order calculus
\citep{martens2012estimating,gower2012framework}; exact HVPs have
graph-evaluation cost under standard AD assumptions
\citep{pearlmutter1994fast}.

The reported experiments (Section~\ref{sec:experiments}) describe
resonance decay without skip connections, bottleneck-dependent
profiles, zero GN-Gap in the locally affine controls,
nonzero $H^T_{Q,K}$ in the measured
Softmax self-attention configuration (Exp.\,5), and corresponding
curvature profiles on a convolutional architecture
(ResNet-18, ${\sim}11$M parameters, Exp.\,6).
A negative Lyapunov exponent gives a decaying pointwise upper
bound when the remaining curvature factor stays bounded
(Theorem~\ref{thm:lyapunov-decay}); a positive exponent does not
force measured resonance to grow.  The reported batch-averaged
profiles (Table~\ref{tab:rho-accuracy}, Remark~\ref{rem:batch-decay})
are not a universal training-dynamics law.

This corrected archival version no longer maintains the original
general novelty claims: general DAG-Hessian calculus has prior
precedents.  The retained material is the descriptive organization of
graph-indexed curvature and the reported diagnostic case studies,
subject to the stated limitations.  This work is no longer being
developed.


\section*{Declaration of competing interests}
The authors declare that they have no known competing financial interests
or personal relationships that could have appeared to influence the work
reported in this paper.

\section*{CRediT authorship contribution statement}
\textbf{Maxim Bolshim}: Conceptualization, Methodology, Software,
Formal analysis, Writing -- original draft, Visualization.
\textbf{Alexander Kugaevskikh}: Validation, Writing -- review \& editing.

\section*{Acknowledgements}
The authors thank the anonymous reviewers for their constructive feedback.

\section*{Funding}
This work was supported by the state assignment (project FSER-2025-0004).

\section*{Declaration of generative AI and AI-assisted technologies
  in the manuscript preparation process}
During the preparation of this work the authors used Claude Opus
(Anthropic) for partial editing and typesetting of the manuscript, and
Claude Sonnet (Anthropic) for partial development of the accompanying
software code.  After using these tools, the authors reviewed and edited
the content as needed and take full responsibility for the content of the
manuscript.

\section*{Data availability}
The source code and experimental results supporting the findings of this
study are openly available at
\url{https://github.com/comiam/dag-hesse}
(archived: \url{https://doi.org/10.5281/zenodo.19545553}).

\bibliographystyle{elsarticle-harv}
\bibliography{refs}

\appendix
\raggedbottom
\renewcommand{\theequation}{\Alph{section}.\arabic{equation}}
\counterwithin{equation}{section}
\renewcommand{\theHequation}{appendix.\Alph{section}.\arabic{equation}}
\renewcommand{\thefigure}{\Alph{section}.\arabic{figure}}
\counterwithin{figure}{section}
\renewcommand{\thetable}{\Alph{section}.\arabic{table}}
\counterwithin{table}{section}

\section{Notation details and function spaces}
\label{app:notation-details}

\begin{remark}[Index convention]
  $i$ --- component index of the node output;
  $j,k$ --- component indices of inputs;
  $k,\ell$ (in parameter context) --- component indices of~$\theta_v$;
  $v,w,u$ --- node indices.
\end{remark}

Standard function spaces are employed:
$C^2$ (twice continuously differentiable),
$C^{1,1}$ (continuously differentiable with locally Lipschitz
first derivative), and $PC^2$ (piecewise~$C^2$ with
piece boundaries of measure zero). A locally Lipschitz function
has a Clarke first-order generalized gradient
$\partial_C f(x)$~\citep{clarke1990optimization}. The
classical Clarke-gradient Hessian $\partial_C^2 F(x)
:=\partial_C(\nabla F)(x)$ instead requires the stronger
condition $F\in C^{1,1}$ locally; $PC^2$ and local Lipschitz
continuity of $F$ do not suffice. Unless this stronger condition
is explicitly imposed, all Hessian identities below are classical
identities on open $C^2$ regions, with an independent AD
convention at switching boundaries (\ref{app:proof-clarke}).

\begin{remark}[Tensor notation and contraction rules]
  Index~$i$ refers to the output component of a node
  ($f_u$ or~$f_v$); $j,m$ to inputs from parent nodes;
  $\alpha,\beta$ to parameters~$\theta_v$.
  Notation $[T]_{i,\bullet,\bullet}$ denotes a tensor slice at
  fixed~$i$.
  Contraction
  $[T_{u;v}]_{i,j,k}\,\delta_{u,i}$ yields a $d_v\!\times\!d_v$
  matrix with entries
  $\sum_{i=1}^{d_u}[T_{u;v}]_{i,j,k}\,\delta_{u,i}$.
  Under matrix multiplication
  $D_{u\gets v}^\top H^f_{u,u}\,D_{u\gets w}$, the dimensions
  chain as
  $(d_v\!\times\!d_u)(d_u\!\times\!d_u)(d_u\!\times\!d_w)
  =d_v\!\times\!d_w$.
  Rules for differentiating matrix expressions follow the
  standard
  formalism~\citep{magnus2019matrix}.\\[4pt]
  Element-wise definitions of second-derivative tensors
  ($v,w\!\in\!\Pa(u)$):
  \begin{align*}
    [T_{u;v}]_{i,j,k}
    &= \frac{\partial^2 (f_u)_i}
    {\partial(f_v)_j\,\partial(f_v)_k}
    \in\mathbb{R}^{d_u\times d_v\times d_v},\\[2pt]
    [T_{u;v,w}]_{i,j,k}
    &= \frac{\partial^2 (f_u)_i}
    {\partial(f_v)_j\,\partial(f_w)_k}
    \in\mathbb{R}^{d_u\times d_v\times d_w},\\[2pt]
    [T_{v;w,\theta}]_{i,j,k}
    &= \frac{\partial^2 (f_v)_i}
    {\partial(f_w)_j\,\partial(\theta_v)_k}
    \in\mathbb{R}^{d_v\times d_w\times p_v},\\[2pt]
    [T_v^\theta]_{i,k,\ell}
    &= \frac{\partial^2 (f_v)_i}
    {\partial(\theta_v)_k\,\partial(\theta_v)_\ell}
    \in\mathbb{R}^{d_v\times p_v\times p_v}.
  \end{align*}
\end{remark}

\begin{remark}[Symmetry of the tensor component]
  The mixed tensor
  $T_{u;v,w}\!\in\!\mathbb{R}^{d_u\times d_v\times d_w}$ for
  $v\!\neq\!w$ produces rectangular $d_v\!\times\!d_w$ slices
  that cannot be symmetrized.
  Symmetry of the full Hessian is ensured by the relation
  $H^f_{v,w}\!=\!(H^f_{w,v})^\top$ between blocks.
  For diagonal blocks ($v\!=\!w$) the slices are square; in the
  smooth case
  $[T_{u;v}]_{i,:,:}\!=\![T_{u;v}]_{i,:,:}^\top$ by
  Schwarz's theorem.
\end{remark}

\paragraph{Assembling blocks into the global Hessian.}
For disjoint parameter blocks, the global Hessian
$\nabla^2_\theta\mathcal{L}\!\in\!\mathbb{R}^{P\times P}$ is
assembled as $[H_{\theta_{v_i},\theta_{v_j}}]_{i,j=1}^{n}$.
With shared parameters, assembling independent node-wise copies
instead gives $H_\eta$; for the linear tying $\eta=S\theta$,
the physical Hessian is $H_\theta=S^\top H_\eta S$
(Remark~\ref{prop:shared-params}). Symmetry follows from
Schwarz's theorem wherever the composed loss is $C^2$.

\section{Stability and invariance of the curvature recursion}
\label{app:stability}

\begin{remark}[Conditioning of the curvature recursion]
  \label{thm:error-accumulation}
  Relative errors in local Jacobians and tensors do not imply
  a depth-only bound on the relative error of a Hessian block.
  Path contributions can cancel: two scalar contributions
  $1$ and $-(1-\delta)$, with $0<\delta<1$, sum to $\delta$,
  while a relative perturbation $\epsilon_D$ of the first
  changes this sum by
  $\epsilon_D$, giving relative error $\epsilon_D/\delta$.
  A relative stability guarantee therefore requires additional
  conditioning or non-cancellation assumptions; none is
  asserted here from graph depth and local errors alone.
\end{remark}

The Hutchinson
estimator~\citep{avron2011randomized} is used for trace and
norm computation, converging as
$\|\hat{T}-T\|_F
\leq C\|T\|_F\sqrt{\log(1/\delta)/m}$
with probability $\geq 1\!-\!\delta$ using $m$ probe vectors.

\begin{theorem}[Invariance of geometric coupling]
  \label{thm:coupling-invariance}
  Let $v\neq w$ and $\alpha>0$. Suppose the original and
  transformed perturbed losses $F$ and $\widetilde F$ are
  $C^2$ near zero and satisfy
  $\widetilde F(S\varepsilon)=F(\varepsilon)$ throughout
  that neighborhood, where $S$ is block diagonal with
  $S_v=\alpha I$ and all other blocks equal to the identity.
  This is a compensated linear change of activation/perturbation
  coordinates, not merely preservation of an output value.
  The chain rule gives $\widetilde H^f=S^{-\top}H^fS^{-1}$;
  hence
  $\widetilde H^f_{v,v}=\alpha^{-2}H^f_{v,v}$,
  $\widetilde H^f_{v,w}=\alpha^{-1}H^f_{v,w}$, and
  $\widetilde H^f_{w,w}=H^f_{w,w}$.
  When both diagonal norms are nonzero,
  \[
    \widetilde{\mathcal{C}}(v,w)
    =\frac{\alpha^{-1}\|H^f_{v,w}\|_F}
    {\sqrt{\alpha^{-2}\|H^f_{v,v}\|_F
    \cdot\|H^f_{w,w}\|_F}}
    =\mathcal{C}(v,w).
  \]
  A weight rescaling must implement this local functional
  identity, including compensation on every affected branch;
  equality of the unperturbed outputs alone is insufficient.
\end{theorem}

\section{Optimization-theoretic remarks}
\label{app:convergence}

\subsection{Convergence guarantees}

Classical local Newton guarantees concern smooth objectives: near a
stationary point with nonsingular Hessian, continuity of the Hessian
gives local superlinear convergence for exact Newton steps, and a
locally Lipschitz Hessian gives quadratic convergence under the usual
local hypotheses~\citep{nocedal2006numerical}. Such results do not
follow at ReLU switching points from an AD second-derivative
convention. Regularization $B_k+\lambda_k I$ and symmetrization
$\tfrac{1}{2}(B_k+B_k^\top)$ define modified matrices but by
themselves imply no convergence guarantee.

\subsection{Pseudoinverse and Newton direction}

\begin{remark}[Relation to the pseudoinverse]
  For a smooth objective and its Hessian $H$, the direction
  $d=-H^{\dagger}g$ ($g=\nabla\mathcal{L}$) is the
  minimum-norm least-squares solution of $Hd=-g$; it solves
  the equation exactly only when $g\in\operatorname{range}(H)$.
  At full rank, $H^{\dagger}=H^{-1}$ and the direction
  coincides with the standard Newton step
  \citep{ben2003generalized}. At a non-smooth point the same
  algebra applies to a chosen AD matrix, without identifying
  that matrix as a Clarke Hessian or obtaining Newton convergence.
\end{remark}

\section{Non-smooth case: piecewise-smooth AD convention}
\label{app:proof-clarke}

We separate classical derivatives within open smooth regions from
the values produced by a specified AD convention at their
boundaries. The Clarke first-order generalized gradient of a
locally Lipschitz loss is well defined; applying the same
construction to its gradient additionally requires a locally
Lipschitz gradient ($C^{1,1}$), which is not implied by
Assumption~\ref{ass:regularity}(B). The results here do not
assert a generalized second-order chain rule or a convergence
theorem for non-smooth Newton steps.

\begin{theorem}[AD-Hessian of a ReLU DAG network]
  \label{thm:clarke_hessian_relu}
  Let a finite DAG comprise affine maps, affine merges and
  ReLU nodes, with fixed parameters, and let the output
  loss $\ell\in C^2$. As a function of its input (or
  independent additive activation perturbations):
  \begin{enumerate}
    \item The network output and composed loss are locally
      Lipschitz.
    \item They are $C^2$ on the interiors of finitely many
      activation regions; the classical Hessian exists and
      the smooth chain-rule recursion agrees with AD there.
    \item On a switching boundary, the value computed with
      $\operatorname{ReLU}'(0):=0$ and
      $\operatorname{ReLU}''(0):=0$ is a convention for that
      computational graph, not necessarily a classical or
      Clarke-gradient Hessian of the composed loss.
  \end{enumerate}
\end{theorem}

\begin{definition}[AD-Hessian at a non-smooth point]
  \label{def:clarke-min-norm}
  For a specified graph and the primitive conventions
  $\operatorname{ReLU}'(0):=0$ and
  $\operatorname{ReLU}''(0):=0$, its \emph{AD-Hessian} is the
  result of differentiating the corresponding AD gradient.
  This computational definition makes no minimum-norm or
  generalized-Hessian selection claim; other smooth or
  bilinear nodes retain their own local second derivatives.
\end{definition}

\begin{theorem}[Limits of smooth-region Hessians]
  \label{thm:clarke-stability}
  If a loss $F$ is additionally $C^{1,1}$ near $x_0$ and
  $x_n\to x_0$ are points where $\nabla F$ is
  differentiable, then any finite limit
  $\nabla^2F(x_n)\to M$ belongs to the Clarke generalized
  Jacobian $\partial_C(\nabla F)(x_0)$. If $F$ is $C^2$
  on a neighborhood of $x_0$, its classical Hessian is
  continuous there. Neither conclusion identifies a
  boundary AD-Hessian with a smooth-region limit.
\end{theorem}

\begin{proof}[Proof of Theorem~\ref{thm:clarke_hessian_relu}]
  Finite compositions of affine maps and ReLU are locally
  Lipschitz, and fix an affine branch for each ReLU on
  every activation region. With fixed parameters the
  resulting map is affine on each such region. Hence the
  output loss is $C^2$ in the interior and its classical
  Hessian agrees with the chain-rule/AD computation there.
  A finite ReLU DAG has finitely many polyhedral regions;
  their genuine switching boundaries have measure zero.
  This does not imply a differentiable gradient on those
  boundaries. For $F(x)=\tfrac12(\operatorname{ReLU}(x)-1)^2$,
  $F'(x)=0$ for $x<0$ and $F'(x)=x-1$ for $x>0$.
  The gradient jumps at zero, so $F\notin C^{1,1}$ there;
  its boundary AD-Hessian cannot be justified by applying
  the Clarke generalized Jacobian to $\nabla F$.
\end{proof}

\begin{proof}[Proof of Theorem~\ref{thm:clarke-stability}]
  Under the additional $C^{1,1}$ assumption,
  $\nabla F$ is locally Lipschitz and its Clarke
  generalized Jacobian contains limits of its ordinary
  Jacobians at differentiability points by definition.
  Continuity on a $C^2$ neighborhood follows from the
  continuity of classical second derivatives.
\end{proof}

\section{Implementation via autodiff frameworks}
\label{app:autodiff-impl}

\begin{algorithm}[H]
  \caption{Computing a joint-injection Hessian block}
  \label{alg:compute-input-hess}
  \begin{algorithmic}[1]
    \Require DAG with independent post-node injections $\varepsilon$,
      output loss $F(\varepsilon)$, nodes $v,w$
    \State Evaluate $F$ at $\varepsilon=0$ with fixed parameters
    \State \Return $\nabla_{\varepsilon_w}
      (\nabla_{\varepsilon_v}F)(0)$ with orientation
      $d_v\times d_w$
  \end{algorithmic}
\end{algorithm}

\begin{algorithm}[H]
  \caption{Computing an independent-copy parameter Hessian block}
  \label{alg:compute-param-hess}
  \begin{algorithmic}[1]
    \Require Parameters $\eta_v,\eta_w$ treated as independent copies
    \State Compute $H_{\eta_v,\eta_w}$ by~\eqref{eq:Htheta},
      including both mixed orientations and all ancestor responses
    \State \Return $H_{\eta_v,\eta_w}$
  \end{algorithmic}
\end{algorithm}

\begin{algorithm}[H]
  \caption{Full Hessian computation by differentiating the gradient}
  \label{alg:full-hessian}
  \begin{algorithmic}[1]
    \Require Smooth graph loss $\mathcal{L}(\theta)$, unique parameter
      coordinates $\theta\in\mathbb{R}^P$
    \State Compute $g(\theta)=\nabla_\theta\mathcal{L}(\theta)$
      on the original graph (with any parameter sharing)
    \For{$j=1,\ldots,P$}
      \State Compute column $j$ as the directional derivative
      $\mathrm{D}g(\theta)[e_j]$
    \EndFor
    \State \Return the $P\times P$ matrix of columns
  \end{algorithmic}
\end{algorithm}

\section{Proof of the canonical decomposition}
\label{app:canonical-proof}

We give the full proof of
Theorem~\ref{thm:canonical-decomposition}.

\begin{proof}
  Let $C=\nabla^2\mathcal{L}(f_{out})$ and
  $J_v=\mathcal{J}_{out\gets v}$.  The output row is
  $H^f_{out,w}=CJ_w$ and $H^{GN}_{out,w}=CJ_w$.
  For $v\neq out$, the chain rule gives
  $J_v=\sum_{u\in\Ch(v)}J_uD_{u\gets v}$, so
  \[
    H^{GN}_{v,w}=\sum_{u\in\Ch(v)}
       D_{u\gets v}^\top H^{GN}_{u,w}.
  \]
  Differentiating
  $\delta_v=\sum_{u\in\Ch(v)}D_{u\gets v}^\top\delta_u$
  with respect to the independent perturbation $\varepsilon_w$
  yields~\eqref{eq:Hf}: differentiating $\delta_u$ gives
  $H^f_{u,w}$, and differentiating the local edge Jacobian
  gives $\sum_{p,i}\delta_{u,i}T_{u;v,p}^{(i)}\mathcal{J}_{p\gets w}$.
  Subtracting the GN transport gives the residual recursion with
  a zero output row.  No two-sided expansion or independence of
  incomparable nodes is needed.
\end{proof}

\subsection{Full parametric Hessian formula}
\label{app:Htheta-full}

For independent local parameter copies $\eta_v$ let
$A_v=\partial g_v/\partial\eta_v$,
$B_v=\nabla^2_{\eta_v\eta_v}\langle\delta_v,g_v\rangle$ and
$M_{v,p}=\nabla^2_{\eta_v f_p}\langle\delta_v,g_v\rangle$
for $p\in\Pa(v)$, evaluating these derivatives with $\delta_v$
held fixed.  Equation~\eqref{eq:Htheta} reads explicitly
\begin{align}\label{eq:Htheta-full}
 H_{\eta_v,\eta_w}={}&A_v^\top H^f_{v,w}A_w
 +\mathbf{1}_{v=w}B_v\nonumber\\
 &+\sum_{p\in\Pa(v)}M_{v,p}\mathcal{J}_{p\gets w}A_w
 +\sum_{p\in\Pa(w)}A_v^\top \mathcal{J}_{p\gets v}^\top M_{w,p}^\top.
\end{align}
For example, in $f_1=ax$, $f_2=bf_1$, $\mathcal{L}=f_2$,
the $a,b$ cross derivative is $x$: it comes from the second
mixed term of~\eqref{eq:Htheta-full}, not from $H^f$ or a
pure parameter Hessian.  For a linear tying map $\eta=S\theta$,
\begin{equation*}
 H_\theta=S^\top H_\eta S,
 \qquad (H_\theta)_{ij}=\sum_{v,w}
 S_{v i}^\top H_{\eta_v,\eta_w}S_{w j}.
\end{equation*}
The tying map itself has zero second derivative.

\begin{remark}[Weight sharing between nodes]
  \label{prop:shared-params}
  When $\theta$ is shared by parameter copies at nodes
  $V_\theta=\{v_1,\dots,v_m\}$ with identity embedding, its
  diagonal block is
  $H_{\theta,\theta}=\sum_{i,j=1}^m
  H_{\eta_{v_i},\eta_{v_j}}$.
  All blocks with other parameters also receive the corresponding
  sums, as specified by $S^\top H_\eta S$; linear tying introduces
  no additional second-derivative term.
\end{remark}

\section{Diagnostic metrics and architectural analysis}
\label{app:diagnostics}

\subsection{Inter-layer resonance}

The block structure of the inter-layer Hessian
$\{H^f_{v,w}\}$ allows introducing a quantitative measure of
``geometric connectivity'' between nodes.
The metrics $\mathcal{R}$, $\mathcal{C}$, and
$\mathcal{D}$ are defined in the main text
(Definitions~\ref{def:resonance-coupling-main}
and~\ref{def:stable-rank-main}); below we develop additional
properties and proofs.

From Definition~\ref{def:resonance-coupling-main},
$\mathcal{R}(v,w)=\mathcal{R}(w,v)$ in the smooth
symmetric-Hessian setting, and $\mathcal{R}(v,w)\geq0$.
The absence of a shared loss-dependent descendant suffices
for a zero block; a shared descendant does not ensure a
nonzero value, because derivatives or sources can vanish and
path contributions can cancel.

\begin{remark}[Meaning of inter-layer resonance]
  Resonance summarizes the magnitude of a chosen activation
  cross-block and can be visualized as a weighted pairwise
  profile.  Large values alone do not establish a training
  mechanism or a causal benefit of skip connections.
\end{remark}

\subsection{Geometric coupling and layer coherence}
\label{subsec:coupling-coherence}

Geometric coupling $\mathcal{C}(v,w)$
(Definition~\ref{def:resonance-coupling-main}) normalizes
resonance by the diagonal blocks when their norms are nonzero;
it is a descriptive relative block-norm measure.

Properties:
(i)~$\mathcal{C}\geq 0$; $\mathcal{C}=0$ iff
$\mathcal{R}=0$.
(ii)~\textbf{PSD bound:} if $H^T\!=\!0$ and
$[H^f_{v,w}]_{v,w}\succeq 0$, then $\mathcal{C}\leq 1$
(Cauchy--Schwarz for block PSD matrices).
(iii)~\textbf{Scale invariance} under compensated scalar
changes of activation coordinates, provided the corresponding
perturbed loss functions agree under that linear coordinate
change throughout a neighborhood
(Theorem~\ref{thm:coupling-invariance}).  Preserving only an
unperturbed output value is insufficient.
In the general indefinite regime, $\mathcal{C}>1$ is
possible; this rules out PSD of the measured full block matrix
but does not by itself quantify negative eigenvalues or
identify a particular residual source.

\begin{proof}[Proof of the PSD bound]
  With $H^T\!=\!0$, the block matrix
  $\mathbf{M}:=[H^{GN}_{v,w}]_{v,w}\succeq 0$ admits a
  decomposition $\mathbf{M}=LL^\top$.
  Writing
  $L=\bigl(
    \begin{smallmatrix}P\\Q
  \end{smallmatrix}\bigr)$
  ($P\!\in\!\mathbb{R}^{d_v\times r}$,
  $Q\!\in\!\mathbb{R}^{d_w\times r}$):
  $M_{v,v}=PP^\top$, $M_{v,w}=PQ^\top$, $M_{w,w}=QQ^\top$.
  Set $S:=P^\top\!P$, $T:=Q^\top\!Q$ ($S,T\succeq 0$).
  Then $\|M_{v,w}\|_F^2=\operatorname{tr}(ST)$,
  $\|M_{v,v}\|_F^2=\operatorname{tr}(S^2)$,
  $\|M_{w,w}\|_F^2=\operatorname{tr}(T^2)$.
  By the Cauchy--Schwarz inequality for the Hilbert--Schmidt
  inner product:
  $(\operatorname{tr}(ST))^2
  \leq\operatorname{tr}(S^2)\operatorname{tr}(T^2)$,
  whence $\mathcal{C}(v,w)\leq 1$.
\end{proof}

\begin{remark}[Stochastic estimation of coupling]
  \label{rem:coupling-stochastic-artifact}
  A ratio formed from estimated block norms need not respect
  the PSD bound $\mathcal{C}\leq1$ even if the underlying
  matrix is PSD.  In the general indefinite regime values
  exceeding one may also be genuine; unconditional clipping
  would remove this information.  The estimate and its
  uncertainty should be reported without such projection.
\end{remark}

\begin{remark}[Interpretation of coupling]
  $\mathcal{C}$ compares the cross-block norm to the two
  diagonal-block norms at the chosen point and in the chosen
  activation coordinates.  Values near zero or one alone do
  not establish independent training, a single functional
  module, or a calibrated degree of interaction.
\end{remark}

\begin{remark}[Self-compensation under sequential narrowing]
  A bottleneck can constrain the rank of the pointwise
  activation-GGN cross-block and its GGN diagonal blocks,
  while normalization can obscure absolute scale.  Neither
  the full-Hessian rank nor a fixed value of $\mathcal{C}$
  follows without controlling residual contributions and
  the relevant diagonal norms.
\end{remark}

\begin{remark}[Computing $\mathcal{R}(v,w)$ without the full matrix]
  For an already materialized block
  $H^f_{v,w}\in\mathbb{R}^{d_v\times d_w}$, computing its
  Frobenius norm costs $O(d_vd_w)$; this excludes the cost
  of constructing the block.  Without materializing it,
  the Hutchinson identity~\citep{avron2011randomized}
  $\|A\|_F^2=\mathbb{E}_z\|Az\|^2$ permits estimation
  from block actions $Az$.  Their cost is governed by graph
  evaluation, with activation/probe and AD storage, rather
  than universally by $P$: one shared scalar parameter can
  occur in a graph with arbitrarily many activations.
\end{remark}

\subsubsection{Rank bottleneck}

\begin{proposition}[Rank bound on path contributions]
  \label{thm:rank-propagation}
  Let $p_v:v\to out$ and $p_w:w\to out$ be terminal paths.
  Their contribution to the pointwise GGN block is
  \[
    \Gamma(p_v,p_w)
    =D_{p_v}^\top\nabla^2\!\mathcal{L}\,D_{p_w}.
  \]
  Its rank does not exceed the rank of either path Jacobian
  or of $\nabla^2\!\mathcal{L}$; in particular it is at most
  the rank of every edge Jacobian on either path.
  For sums, only subadditivity holds:\newline
  $\mathrm{rank}(\textstyle\sum_i\Gamma_i)
  \leq\sum_i\mathrm{rank}(\Gamma_i)$.
\end{proposition}

\begin{proof}
  By rank submultiplicativity, the rank of a nontrivial
  path Jacobian is at most the rank of every edge Jacobian
  on that path.  Since
  $\Gamma=D_{p_v}^\top\nabla^2\!\mathcal{L}\,D_{p_w}$,
  its rank is also at most the rank of either path Jacobian
  and of the terminal loss Hessian.  A length-zero path has
  identity Jacobian and introduces no edge-rank constraint.
\end{proof}

\begin{corollary}[Architectural bottleneck]
  \label{cor:bottleneck}
  If \textbf{all} paths from $v$ and $w$ to $\mathrm{out}$
  pass through a narrow layer~$u$ with
  $d_u\!\ll\!\min(d_v,d_w)$, then
  $\mathrm{rank}(H^{GN}_{v,w})\leq d_u$.
  This bounds the rank of the pointwise activation-GGN
  block.  Residual sources can make $H^f_{v,w}$ full-rank,
  and averaging low-rank pointwise blocks need not preserve
  the bound.
\end{corollary}

\begin{proof}
  All paths factor through~$u$:
  $\mathcal{J}_{\mathrm{out}\gets v}=\mathcal{J}_{\mathrm{out}\gets u}M_v$,
  $\mathcal{J}_{\mathrm{out}\gets w}=\mathcal{J}_{\mathrm{out}\gets u}M_w$,
  so $H^{GN}_{v,w}=M_v^\top(\cdots)M_w$ with the inner
  matrix in $\mathbb{R}^{d_u\times d_u}$.
\end{proof}

\begin{definition}[Effective interaction dimension]
  \label{def:effective-interaction-dim}
  \begin{equation}\label{eq:effective-dim}
    d_{\mathrm{eff}}(v,w)
    :=\frac{\|H^f_{v,w}\|_*}{\|H^f_{v,w}\|_2},
  \end{equation}
  where $\|\cdot\|_*$ is the nuclear norm,
  $\|\cdot\|_2$ the spectral norm.
\end{definition}

\begin{proposition}[Properties of $d_{\mathrm{eff}}$]
  \label{prop:effective-dim-properties}
  (i)~$1\leq d_{\mathrm{eff}}\leq\mathrm{rank}(H^f_{v,w})$.
  (ii)~$d_{\mathrm{eff}}=1$ iff $\mathrm{rank}=1$.
  (iii)~Equality with rank iff all nonzero singular values are
  equal.
\end{proposition}

\begin{remark}[Rank vs.\ norm]
  Proposition~\ref{thm:rank-bottleneck-main} bounds the
  \emph{rank} of~$H^{GN}_{v,w}$, not the Frobenius norm.
  Even at $\mathrm{rank}\leq d_u$,
  $\mathcal{R}(v,w)=\|H\|_F$ can remain large (singular
  values concentrate on~$d_u$ directions).
  A bottleneck test for exact pointwise GGN rank uses the
  factorization in Proposition~\ref{thm:rank-bottleneck-main}.
  The stable rank $\mathcal{D}$ of a nonzero measured full
  block is a separate singular-value summary and cannot
  establish an exact GGN or full-Hessian rank.
\end{remark}

\begin{remark}[Double rank bound]
  \label{rem:double-rank-bound}
  For $K$-class softmax-CE loss,
  $\mathrm{rank}(\nabla^2\!\mathcal{L})\leq K\!-\!1$.
  Together with
  Proposition~\ref{thm:rank-bottleneck-main}, the
  \emph{pointwise activation-GGN block} satisfies
  $\mathrm{rank}(H^{GN}_{v,w})\leq\min(d_u,K\!-\!1)$.
  Its nonzero stable rank obeys the same upper bound;
  the full-Hessian $\mathcal{D}(v,w)$ need not.
\end{remark}

\begin{remark}[Architectural consequences]
  Bypassing a separating bottleneck can remove its specific
  pointwise GGN rank ceiling.  It does not imply full rank,
  improved training coordination, or a nonzero full-Hessian
  block.
\end{remark}

\subsubsection{Stochastically computable stable rank}
\label{subsubsec:stable-rank}

Computing $d_{\mathrm{eff}}$ requires the nuclear norm
($\|H\|_*=\sum_i\sigma_i$), equivalent to a full SVD.
For scalable analysis we use the stable rank
$\mathcal{D}(v,w)$
(Definition~\ref{def:stable-rank-main})---an alternative
effective-dimension measure admitting stochastic estimation
via HVP without materializing the full matrix.

\begin{proposition}[Properties of stable rank]
  \label{prop:stable-rank-properties}
  For a nonzero block:
  (i)~$1\leq\mathcal{D}\leq\mathrm{rank}(H^f_{v,w})$.
  (ii)~$\mathcal{D}=1$ iff rank~1.
  (iii)~$\mathcal{D}=\mathrm{rank}$ iff all nonzero singular
  values are equal.
  (iv)~$\mathcal{D}\leq d_{\mathrm{eff}}\leq
  \sqrt{\mathrm{rank}\cdot\mathcal{D}}$.
\end{proposition}

\begin{proof}
  Let $\sigma_1\geq\cdots\geq\sigma_r>0$ be the nonzero
  singular values and $a_i:=\sigma_i/\sigma_1\in(0,1]$.
  Then $\mathcal{D}=\sum a_i^2$,
  $d_{\mathrm{eff}}=\sum a_i$.
  (i)--(iii) follow from $a_1=1$, $a_i\leq 1$.
  (iv)~$a_i^2\leq a_i$ gives $\mathcal{D}\leq d_{\mathrm{eff}}$;
  Cauchy--Schwarz gives
  $d_{\mathrm{eff}}\leq\sqrt{r\,\mathcal{D}}$.
\end{proof}

\begin{remark}[Stochastic estimation of~$\mathcal{D}$]
  \label{rem:stable-rank-estimation}
  $\|H\|_F^2$ is estimated by the Hutchinson
  method~\citep{avron2011randomized}:
  $\widehat{\|H\|}_F^2
  =\frac{1}{m}\sum_{k=1}^{m}\|Hz_k\|^2$
  with Rademacher vectors $z_k\!\in\!\{-1,+1\}^{d_w}$.
  $\|H\|_2^2=\sigma_1^2$ is estimated via power iteration on
  $H^\top H$ using the true Euclidean adjoint of the same
  block operator.  Swapping derivatives of dependent node
  tensors does not establish this adjoint identity.
  With a spectral gap and nonzero initial overlap with the
  leading right-singular subspace, the power iteration
  converges geometrically~\citep{nocedal2006numerical}.
  The algorithm uses $m+T+1$ block actions and $T$ adjoint
  actions when it completes; their costs depend on graph
  evaluation.  Memory includes activation-space probes and
  the required AD state, with no universal $O(P)$ bound.
  A small unresolved norm returns an unresolved status, not
  stable rank zero: for $H=tI_d$ and any $t\neq0$,
  $\mathcal D=d$.  Zero is reserved for a block established
  to be exactly zero.
\end{remark}

\begin{algorithm}[H]
  \caption{Stochastic estimation of $\mathcal{D}(v,w)$ via HVP}
  \label{alg:stochastic-stable-rank}
  {\footnotesize
    \begin{algorithmic}[1]
      \Require Block operator $A=H^f_{v,w}$ and its true Euclidean
      adjoint $A^\top$; $m,T\geq1$; numerical resolution
      thresholds $\tau_q,\tau_s>0$
      \Ensure Estimate $\hat{\mathcal D}(v,w)$ or \texttt{unresolved}
      \If{$A$ is established to be exactly zero}
        \State \Return $0$ \Comment{Zero-block convention only}
      \EndIf
      \State \textbf{// Power iteration:}
      \State $\mathbf q\gets\mathrm{randn}(d_w)$
      \State $a\gets\|\mathbf q\|$
      \If{$a$ is nonfinite or $a\leq\tau_q$}
        \State \Return \texttt{unresolved}
      \EndIf
      \State $\mathbf q\gets\mathbf q/a$
      \For{$t=1,\ldots,T$}
        \State $\mathbf p\gets A\mathbf q$
        \State $\mathbf q\gets A^\top\mathbf p$
        \State $a\gets\|\mathbf q\|$
        \If{$a$ is nonfinite or $a\leq\tau_q$}
          \State \Return \texttt{unresolved}
        \EndIf
        \State $\mathbf q\gets\mathbf q/a$
      \EndFor
      \State $\hat\sigma_1^2\gets\|A\mathbf q\|^2$
      \If{$\hat\sigma_1^2$ is nonfinite or $\hat\sigma_1^2\leq\tau_s$}
        \State \Return \texttt{unresolved}
      \EndIf
      \State \textbf{// Hutchinson:}
      \State $S\gets0$
      \For{$k=1,\ldots,m$}
        \State Sample $z_k\in\{-1,+1\}^{d_w}$ (Rademacher)
        \State $S\gets S+\|Az_k\|^2$
      \EndFor
      \If{$S$ is nonfinite or $S\leq0$}
        \State \Return \texttt{unresolved}
      \EndIf
      \State \Return $(S/m)/\hat\sigma_1^2$
  \end{algorithmic}}
\end{algorithm}

\subsection{Path-based curvature analysis}
\label{subsec:path-analysis}

\subsubsection{Path decomposition}

\begin{definition}[DAG path and path Jacobian]
  \label{def:dag-path}
  For nodes $v,w\!\in\!V$, denote by $\mathcal{P}(v\!\to\!c)$
  the set of all directed paths from~$v$ to $c$, including
  the length-zero path $v\to v$ with Jacobian~$I$.
  For a path $p=(u_0,u_1,\ldots,u_k)$, the
  \textbf{path Jacobian} is
  \begin{equation}\label{eq:path-jacobian}
    D_p:=\prod_{i=0}^{k-1}D_{u_{i+1}\gets u_i}
    \in\mathbb{R}^{d_{u_k}\times d_{u_0}}.
  \end{equation}
\end{definition}

\begin{theorem}[Path decomposition of the Hessian]
  \label{thm:path-decomposition}
  The path Jacobian includes the identity for a length-zero
  path.  With a sole terminal loss, define the local source
  $S_{c;a,b}:=\sum_i\delta_{c,i}
  [T_{c;a,b}]_{i,\bullet,\bullet}$ for
  $a,b\in\Pa(c)$.  Then the exact block expansion is
  \begin{align}\label{eq:hessian-path-integral}
    H^f_{v,w}
    ={}&\sum_{\substack{p_v\in\mathcal{P}(v\to out)\\
                        p_w\in\mathcal{P}(w\to out)}}
       D_{p_v}^\top\nabla^2\!\mathcal{L}\,D_{p_w}\nonumber\\
     &+\sum_{c\in V}\sum_{a,b\in\Pa(c)}
       \sum_{\substack{p_v\in\mathcal{P}(v\to a)\\
                       p_w\in\mathcal{P}(w\to b)}}
       D_{p_v}^\top S_{c;a,b}D_{p_w}.
  \end{align}
  Empty path sets contribute zero.  The first line counts the
  terminal GGN once; the second assigns each local residual
  source to its node and parent slots, without reusing a
  downstream full Hessian at every common descendant.
\end{theorem}

\begin{proof}
  The total first derivative $\mathcal{J}_{c\gets v}$ is the sum of
  path Jacobians from $v$ to $c$.  The second-order chain
  rule gives
  $H^f_{v,w}=\mathcal{J}_{out\gets v}^\top\nabla^2\!\mathcal{L}
  \mathcal{J}_{out\gets w}
  +\sum_c\sum_{a,b\in\Pa(c)}
  \mathcal{J}_{a\gets v}^\top S_{c;a,b}\mathcal{J}_{b\gets w}$.
  Expanding each total Jacobian into paths yields the stated
  terminal term and the transported local sources.
\end{proof}

\begin{definition}[Path resonance]
  \label{def:path-resonance}
  For an individual matrix summand $A$ in
  \eqref{eq:hessian-path-integral}, its path-contribution
  magnitude is $\|A\|_F$.
\end{definition}

\begin{corollary}[Norm bound over path contributions]
  \label{cor:resonance-additivity}
  $\mathcal{R}(v,w)$ is at most the sum of the Frobenius norms
  of the matrix summands in \eqref{eq:hessian-path-integral}.
  Equality holds if the nonzero summands are nonnegative
  scalar multiples of the same matrix; in general their
  interference or cancellation makes the inequality strict.
\end{corollary}

\subsubsection{Theoretical justification of skip connections}

\begin{definition}[Residual architecture]
  \label{def:residual-architecture}
  A \textbf{skip connection} between nodes $u\!\prec\!v$ adds
  an edge $(u,v)$ with an identity (or linear) map:
  $f_v\gets f_v+W_{\mathrm{skip}}f_u$.
\end{definition}

\begin{theorem}[Conditional global norm comparison]
  \label{thm:skip-connection-stabilization}
  Let $G_{\mathrm{base}}$ and $G_{\mathrm{res}}$ have the same
  indexed activation space, with respective Jacobians
  $J_{\mathrm{base}}$ and
  $J_{\mathrm{res}}=J_{\mathrm{base}}+\Delta J$.
  Suppose the terminal loss Hessian $A\succeq0$ is the same
  for both comparisons, and set
  $B=\Delta J^\top A J_{\mathrm{base}}
    +J_{\mathrm{base}}^\top A\Delta J
    +\Delta J^\top A\Delta J$.
  If $B\succeq0$ and each full block matrix obeys
  $\|\mathbf H^T\|_F\leq\eta\|\mathbf H^{GN}\|_F$
  with $\eta<1$, then
  \begin{equation}\label{eq:skip-stabilization}
    \|\mathbf H^f_{\mathrm{res}}\|_F
    \geq\frac{1-\eta}{1+\eta}
    \|\mathbf H^f_{\mathrm{base}}\|_F.
  \end{equation}
  This is a conditional global comparison, not a statement
  about each off-diagonal resonance $\mathcal R(v,w)$.
\end{theorem}

\begin{proof}
  The assumptions imply
  $\mathbf H^{GN}_{\mathrm{res}}
  =\mathbf H^{GN}_{\mathrm{base}}+B$ with both summands
  PSD; hence its full-matrix Frobenius norm is no smaller.
  The reverse triangle inequality gives
  $\|\mathbf H^f_{\mathrm{res}}\|_F
  \geq(1-\eta)\|\mathbf H^{GN}_{\mathrm{res}}\|_F$,
  while the triangle inequality gives
  $\|\mathbf H^f_{\mathrm{base}}\|_F
  \leq(1+\eta)\|\mathbf H^{GN}_{\mathrm{base}}\|_F$.
\end{proof}

\begin{remark}[Global vs.\ block-wise monotonicity]
  Even when $B\succeq0$, its off-diagonal blocks can cancel
  those of the base GGN matrix.  The global comparison above
  therefore does not establish monotonicity of individual
  cross-block norms or of averages over selected pairs.
\end{remark}

\begin{theorem}[Endpoint bounds: vanilla vs.\ residual chain]
  \label{thm:resnet-exponential}
  Under the terminal-loss and square-block hypotheses of
  Theorem~\ref{thm:resnet-main}, write $C=\|\nabla^2\!\mathcal{L}\|_F$.

  \textbf{(a) Vanilla (upper bound):}
  $\mathcal{R}_{\mathrm{vanilla}}(v_0,v_L)
  \leq C\rho^L$.

  \textbf{(b) Residual chain (upper bound):}
  $\mathcal{R}_{\mathrm{res}}(v_0,v_L)
  \leq C\prod_{j=1}^{L/k}(1+\beta_j)$.

  \textbf{(b$'$) Residual chain (lower bound):}
  \begin{equation}\label{eq:resnet-lower-bound}
    \mathcal{R}_{\mathrm{res}}(v_0,v_L)
    \geq C\prod_{j=1}^{L/k}(1-\beta_j).
  \end{equation}
  For $\beta_j\leq c\rho^k/L$, the corresponding upper
  and lower factors are bounded by quantities tending to
  $e^{c\rho^k/k}$ and $e^{-c\rho^k/k}$, respectively.
  Fixed positive branch bounds alone do not provide a
  depth-uniform lower bound.
\end{theorem}

\begin{proof}
  The vanilla endpoint Jacobian is a product of edge
  Jacobians with norms at most $\rho$.  For the residual
  chain $\mathcal{J}_{out\gets v_0}=\prod_j(I+B_j)$; the inequalities
  $\sigma_{\min}(I+B_j)\geq1-\beta_j$ and
  $\|I+B_j\|_2\leq1+\beta_j$ give the product bounds on
  $H^f_{v_0,out}=\mathcal{J}_{out\gets v_0}^\top\nabla^2\!\mathcal{L}$.
\end{proof}

\begin{remark}[Cross-path terms and the lower bound]
  \label{rem:cross-path-lower-bound}
  The path decomposition
  (Theorem~\ref{thm:path-decomposition}) writes the full
  Jacobian as
  $\mathcal{J}_{out\gets v_0}=\sum_{p\in\mathcal{P}}D_p$, where
  $\mathcal{P}$ is the set of paths from $v_0$ to $out$.
  Substituting into the GN block gives
  \[
    H^{GN}_{v_0,out}
    =\Bigl(\sum_p D_p\Bigr)^{\!\top}\nabla^2\!\mathcal{L}
    =\sum_p D_p^\top\nabla^2\!\mathcal{L},
  \]
  and individual path contributions
  $D_p^\top\nabla^2\!\mathcal{L}$ \textbf{may partially
  cancel}: the inequality
  $\|\sum_p D_p^\top A\|_F\geq\|D_{id}^\top A\|_F$
  does not hold in general.
  The proof of~(b$'$) circumvents this difficulty by applying
  $\sigma_{\min}$-multiplicativity to the factored product
  $\prod_j(I+D_{\mathrm{res},j})$ rather than to the
  element-wise sum of paths, thereby handling cross-path
  interactions automatically through the minimum singular value
  of the full Jacobian without the need to control their signs.
\end{remark}

\begin{corollary}[Criterion of geometric stability]
  \label{cor:geometric-stability}
  A sufficient condition for a nonvanishing endpoint
  resonance relative to a nonzero terminal curvature is a
  square residual chain with $\beta_j<1$ and a depth-uniform
  positive bound on $\prod_j(1-\beta_j)$.
  An individual strong path is not sufficient: other paths
  can cancel its contribution.  No necessary-and-sufficient
  criterion for arbitrary pairs follows from path counting.
\end{corollary}

\begin{remark}[Architectural consequences of path analysis]
  (1)~ResNet: sufficiently controlled residual branch
  products give conditional endpoint lower bounds.
  (2)~DenseNet: more paths enlarge the set of possible
  matrix contributions, not necessarily their sum's norm.
  (3)~Transformer: bilinear attention and Softmax can
  contribute local mixed residual sources.
  (4)~U-Net: bypass paths can remove a separating-bottleneck
  rank ceiling, without guaranteeing high $\mathcal{C}$.
  These are structural possibilities, not quantitative
  predictions for the named architectures.
\end{remark}

\begin{remark}[Limits of skip-placement claims]
  \label{prop:optimal-skip-placement}
  The graph-distance estimate is an upper bound on absolute
  pointwise GGN blocks; it supplies no positive lower bound
  on the minimum resonance or optimal placement of a fixed
  skip-edge budget.  Such an optimization would require
  additional assumptions on Jacobian products and path
  cancellation.
\end{remark}

\section{Additional theoretical results}
\label{app:theoretical-results}

\subsection{Structural sparsity of the Hessian}

For a sole terminal loss, the pointwise activation-GGN
blocks obey a conditional absolute graph-distance estimate.
Neither normalized coupling nor the full parameter Hessian
inherits this estimate automatically.

\begin{remark}[Computing blocks for a one-directional path]
  \label{prop:one-directional-path}
  For $w=out$, the terminal boundary block satisfies
  \[
    H^f_{v,out}
    =\sum_{u\in\Ch(v)}D_{u\gets v}^\top H^f_{u,out}.
  \]
  For a general descendant $w$, the full one-sided
  recursion~\eqref{eq:Hf} also includes local sources;
  ancestry alone does not make these sources vanish.
  Along a single-child chain terminating at $w$, however,
  the source term vanishes at each earlier node because $w$
  cannot reach its parent, and the same transport identity
  holds up to $w$.
\end{remark}

\begin{theorem}[Exponential curvature sensitivity (Hessian vanishing/exploding)]
  \label{thm:exponential-decay}
  Let $G=(V,E)$ be the DAG of a neural network with maximum
  out-degree $s=\max_v|\Ch(v)|$.
  Suppose the loss depends only on the terminal node $out$,
  every considered node reaches $out$, and edge Jacobian norms
  obey $\|D_{u\gets v}\|_2\leq\rho$ for all edges
  $(v,u)\!\in\!E$, with $0<s\rho<1$.
  Then for the GN-resonance
  $\mathcal{R}^{GN}(v,w):=\|H^{GN}_{v,w}\|_F$:
  \begin{equation}\label{eq:resonance-decay}
    \mathcal{R}^{GN}(v,w)
    \leq C^{GN}\cdot(s\rho)^{\mathrm{dist}(v,w)},
  \end{equation}
  where $\mathrm{dist}(v,w)$ is the shortest-path length in
  the undirected graph and
  $C^{GN}=\|H^{GN}_{\mathrm{out,out}}\|_F
  =\|\nabla^2\!\mathcal{L}\|_F$.
  If all node maps are locally affine at the evaluation
  point, every local residual source vanishes and the bound
  extends to the full input resonance
  $\mathcal{R}(v,w)=\|H^f_{v,w}\|_F$.
  Piecewise-linear activations alone do not suffice when
  other node maps include nonlinear merges or normalization.
\end{theorem}

\begin{proof}
  Denote by $\ell_v$ the length of the longest directed path
  from~$v$ to~$\mathrm{out}$ (topological depth to output).
  The proof proceeds by strong induction on $\ell_v+\ell_w$.

  \textbf{Base case:}
  $\ell_v+\ell_w=0$ means $v=w=\mathrm{out}$;
  $\mathcal{R}^{GN}(\mathrm{out,out})
  =C^{GN}=C^{GN}\cdot(s\rho)^0$.

  \textbf{Inductive step:}
  Suppose $\ell_v+\ell_w\geq 1$.
  WLOG $\ell_v\geq\ell_w$ (otherwise use
  $\|H^{GN}_{v,w}\|_F=\|H^{GN}_{w,v}\|_F$).
  Then $\ell_v\geq 1$, $v\neq\mathrm{out}$, and
  $\Ch(v)\neq\varnothing$.

  The GN recursion is closed
  (Theorem~\ref{thm:canonical-decomposition}):
  $H^{GN}_{v,w}
  =\sum_{u\in\Ch(v)}D_{u\gets v}^\top H^{GN}_{u,w}$
  for arbitrary $(v,w)$, without restrictions on activation
  type.
  Bounding norms:
  \[
    \|H^{GN}_{v,w}\|_F
    \leq|\Ch(v)|\cdot\rho
    \cdot\max_{u\in\Ch(v)}\|H^{GN}_{u,w}\|_F\\
    \leq s\rho\cdot\max_{u\in\Ch(v)}\|H^{GN}_{u,w}\|_F.
  \]
  For each $u\!\in\!\Ch(v)$ we have $\ell_u\leq\ell_v-1$
  (the path $v\!\to\!u\!\to\!\cdots\!\to\!\mathrm{out}$ is one
  edge longer than from~$u$), so
  $\ell_u+\ell_w<\ell_v+\ell_w$ and by the inductive
  hypothesis
  $\|H^{GN}_{u,w}\|_F
  \leq C^{GN}(s\rho)^{\mathrm{dist}(u,w)}$.

  By the triangle inequality in the undirected graph,
  $\mathrm{dist}(u,w)\geq\mathrm{dist}(v,w)-1$
  (since $u$ is a neighbor of~$v$, i.e.\
  $\mathrm{dist}(v,u)=1$).
  With $s\rho<1$:
  \[
    s\rho\cdot C^{GN}(s\rho)^{\mathrm{dist}(u,w)}
    \leq s\rho\cdot C^{GN}(s\rho)^{\mathrm{dist}(v,w)-1}\\
    =C^{GN}(s\rho)^{\mathrm{dist}(v,w)}.
  \]
\end{proof}

\begin{corollary}[Absolute graph-distance threshold]
  \label{cor:bandedness}
  For a sole terminal loss and $0<s\rho<1$, an
  activation-GGN block at distance $d$ obeys
  $\|H^{GN}_{v,w}\|_F\leq C^{GN}(s\rho)^d$.
  Thus blocks beyond a distance cutoff can be uniformly
  small in absolute norm.  This does not imply bandedness
  under a topological ordering, a sparse set of nearby
  pairs, a bound on normalized coupling, or a full parameter
  Hessian approximation guarantee.
\end{corollary}

\begin{definition}[Interaction radius]
  \label{def:interaction-radius}
  For an absolute tolerance $\varepsilon>0$ and a fixed
  pointwise activation-GGN evaluation with $0<s\rho<1$,
  an integer $k_\varepsilon$ satisfying
  $C^{GN}(s\rho)^{k_\varepsilon}<\varepsilon$
  is a graph-distance cutoff: blocks at larger distance
  have Frobenius norm below $\varepsilon$.
  The number of retained blocks depends on the graph; even
  an out-degree-one merging star has quadratically many
  pairs at distance two.
\end{definition}

\begin{remark}[Practical significance]
  Theorem~\ref{thm:exponential-decay} controls individual
  pointwise activation-GGN block norms under quantitative
  contraction.  It does not alone justify a block-diagonal
  parameter-Hessian approximation such as K-FAC.
\end{remark}

\begin{remark}[Mathematical justification of the ``edge of chaos'']
  \label{rem:rho-approx-one}
  Theorem~\ref{thm:exponential-decay} gives an upper
  bound on the absolute norm of pointwise activation-GGN
  blocks when $s\rho<1$; it does not identify a universal
  critical threshold for the full Hessian.
  For a sequential chain, contracting Jacobian products
  suffice for a decaying upper bound when the terminal curvature
  factor stays bounded with depth.  When $s\rho\geq1$,
  this particular estimate is inconclusive: neither growth
  nor failure of decay follows.  Skip connections admit the
  conditional product bounds of
  Theorem~\ref{thm:resnet-exponential}, rather than an
  unconditional depth-independent preservation guarantee.
\end{remark}

\subsection{Negative curvature and saddle points}

For a convex terminal loss, the full Hessian can have
negative curvature that its Gauss--Newton component omits.
Negative eigenvalues alone do not establish a stationary
saddle point.

\begin{definition}[Negative curvature mass]
  \label{def:negative-curvature-mass}
  For a symmetric matrix~$H$ with eigenvalues
  $\{\lambda_i\}_{i=1}^n$, define the
  \emph{negative curvature mass}:
  \begin{equation}\label{eq:negative-mass}
    m^-(H):=\sum_{\lambda_i<0}|\lambda_i|.
  \end{equation}
\end{definition}

\begin{theorem}[GGN and negative curvature]
  \label{thm:gn-blindness}
  Let the terminal loss be convex in its output, and consider
  the full joint activation-injection Hessian at a smooth point.
  Then:
  \begin{enumerate}
    \item The full Gauss--Newton block matrix is positive
      semi-definite: $m^-(H^{GN})=0$.
    \item The full Hessian $H^{full}=H^{GN}+H^T$ may have
      $m^-(H^{full})>0$.  This indicates negative curvature,
      not stationarity or a saddle point by itself.
    \item If $r^\top H^{full}r<0$, then $r^\top H^Tr<0$:
      a negative residual contribution is necessary along
      that direction.
  \end{enumerate}
\end{theorem}

\begin{proof}
  The terminal-loss Jacobian pullback defining the full GGN
  block matrix is PSD for a convex terminal loss
  (Theorem~\ref{thm:canonical-decomposition}).  The residual
  sources need not be PSD, so the full Hessian can have
  negative curvature.  For any such direction,
  $r^\top H^Tr=r^\top H^{full}r-r^\top H^{GN}r<0$.
  None of these statements constrains the gradient.
\end{proof}

\begin{remark}[Connection to the ReLU paradox]
  \label{rem:gn-relu-connection}
  If every node map is locally affine, including merges,
  normalizations and multiplicative operations, then all
  activation residual sources vanish and
  $H^f=H^{GN}$ at that point.  When the terminal loss is
  convex, the full activation-GGN \emph{block matrix} is PSD
  and has no negative eigenvalues; an individual rectangular
  cross-block is not itself PSD.  ReLU activations alone do
  not ensure locally affine behavior of the whole DAG.
  Parameter-space residual terms require separate analysis
  (see GN-Gap metric, Definition~\ref{def:gn-gap}).
\end{remark}

\begin{proposition}[Negative-curvature descent at a stationary point]
  \label{prop:escape-directions}
  Let $F\in C^2$ in a neighborhood of a stationary point
  $x_*$ in a specified Euclidean coordinate space, and let
  $\nabla^2F(x_*)v_i=\lambda_i v_i$ with
  $\|v_i\|=1$ and $\lambda_i<0$.
  Then both $x_*+\alpha v_i$ and $x_*-\alpha v_i$
  decrease $F$ for all sufficiently small $\alpha>0$.
  For the joint activation-injection Hessian these are
  injection directions, not automatically parameter-update
  directions.  A parameter-space conclusion requires the
  parameter Hessian and stationarity in that same space.
\end{proposition}

\begin{proof}
  The $C^2$ Taylor expansion and $\nabla F(x_*)=0$ give
  \[
    F(x_*\pm\alpha v_i)-F(x_*)
    =\tfrac{\alpha^2}{2}\lambda_i+o(\alpha^2).
  \]
  The leading term is negative for either sign.  Without
  stationarity, the first-order term can instead determine
  the sign of the change.
\end{proof}

\begin{corollary}[Practical application]
  \label{cor:saddle-escape-practical}
  In a smooth coordinate space, a direction of negative
  Hessian curvature permits local descent with a sufficiently
  small step: at a stationary point either sign works;
  otherwise choose the sign with nonpositive directional
  derivative.  If that derivative is zero, either sign still
  gives second-order descent.  A parameter optimizer must
  use parameter-space derivatives for this conclusion;
  an activation-injection eigenvector alone is not a
  realizable weight update.  Finite step selection requires
  a separate acceptance rule
  (cf.\ \citealt{dauphin2014identifying}).
\end{corollary}

\subsection{Functional-analytic properties of the Hessian}

Under Assumption~\ref{ass:regularity}(A), the composed loss
has a continuous classical Hessian in the variables being
differentiated. Under (B), the same identities hold on open
$C^2$ regions; a locally Lipschitz loss has a Clarke
\emph{first-order} generalized gradient, but its
Clarke-gradient Hessian need not exist at a switching point
(\ref{app:proof-clarke}). The AD value there is determined
by the chosen primitive conventions, not by continuity
across regions.

\subsection{Integration of specialized architectural components}

\begin{theorem}[Integration of specialized layers]
  The following architectural components can be represented as
  DAG nodes and included in the framework:
  \begin{enumerate}
    \item \textbf{Batch Normalization:} represented as a node
      with two parameter types (scale and shift) and additional
      internal variables (batch statistics).
    \item \textbf{Residual connections (ResNet):} modeled via
      parallel paths in the graph with subsequent merging.
    \item \textbf{Recurrent networks:} mapped to a DAG by
      unrolling in time and identifying parameters shared
      across time steps. Linear parameter tying is a pullback
      of the untied Hessian and adds no second derivative of
      the tying map.
  \end{enumerate}
  \emph{Note.} Attention mechanisms (softmax + weighted sum)
  are formally representable as a DAG subgraph; an explicit
  derivation of the tensors~$T$ for Softmax, cross-blocks
  between Q, K, V, and architectural consequences is given in
  \ref{app:attention-example} (single-head case).
  Bilinear products in attention have local second derivatives
  even when other nodes use ReLU.
\end{theorem}

\begin{proof}[Proof sketch]
  For each layer type one defines the node functions $g_v$ and
  their first and second derivatives.
  For example, for Batch Normalization:
  \[
    g_v(x,\gamma,\beta)
    =\gamma\frac{x-\mu_B}{\sqrt{\sigma_B^2+\epsilon}}+\beta,
  \]
  where $\mu_B,\sigma_B^2$ are the batch mean and variance,
  $\gamma,\beta$ are scale and shift parameters.
  The local Jacobians and second-derivative tensors are
  computed by standard differentiation rules for each node type,
  after which the general
  formulas~\eqref{eq:Hf} and~\eqref{eq:Htheta} apply.
\end{proof}

\subsection{Stochastic nodes}

\begin{remark}[Reduction to a deterministic DAG]
  \label{rem:stochastic-reduction}
  When a node admits a parameter-independent-noise
  reparameterization
  $f_v=h(f_{\Pa(v)},\theta_v,\varepsilon)$,
  the realization $\varepsilon$ is an additional input of a
  deterministic DAG. On each smooth region the Hessian
  formulas~\eqref{eq:Hf},\eqref{eq:Htheta} apply to the
  \emph{sample loss} at fixed $\varepsilon$. Passing those
  derivatives through $\mathbb{E}_\varepsilon$ requires a
  separately justified interchange. A sufficient condition
  is that sample losses are almost surely $C^2$ on a common
  parameter neighborhood, with integrable bounds there for
  the loss and its first and second derivatives. Domination
  of regionwise derivatives alone does not suffice when
  kinks move with the sampled noise; reparameterization
  itself gives no such interchange guarantee.
\end{remark}

\section{Full Hessian computation algorithm}
\label{app:full-algorithm}

\subsection{Recursive HVP via graph recursion}

An exact parameter-HVP for a direction $r\in\mathbb{R}^P$
is the directional derivative of the gradient on the original
parameter graph (including shared weights):
\begin{equation}\label{eq:recursive-hvp}
  \nabla^2_\theta\mathcal{L}(\theta)r
  =\left.\frac{d}{dt}\nabla_\theta
    \mathcal{L}(\theta+tr)\right|_{t=0}.
\end{equation}
Forward-over-reverse or reverse-over-forward differentiation
computes this product without forming the full Hessian
\citep{pearlmutter1994fast}.  It includes the local mixed
input--parameter derivatives and the full parameter-sharing
pullback.  For graph operations whose directional derivatives
have comparable cost, one probe costs
$O(T_{\mathrm{fwd}}+T_{\mathrm{bwd}})$.

\subsection{Full algorithm}

\begin{algorithm}[H]
  \caption{Full Hessian computation for a neural network}
  \label{alg:full-hessian-dag}
  \begin{algorithmic}[1]
    \Require Smooth DAG loss $\mathcal{L}(\theta)$ with $P$ unique
      parameter coordinates (including any shared weights)
    \State Compute the parameter gradient
      $g(\theta)=\nabla_\theta\mathcal{L}(\theta)$
    \For{$j=1,\ldots,P$}
      \State Set column $j$ to $\mathrm{D}g(\theta)[e_j]$
      using directional differentiation of the gradient
    \EndFor
    \State \Return the resulting $P\times P$ Hessian
  \end{algorithmic}
\end{algorithm}

\section{Full proof of the computational complexity theorem}
\label{app:complexity-proof}

\begin{proposition}[Computational complexity]
  \label{thm:computational-complexity}
  Let $P$ be the number of unique parameter coordinates and
  $T_{\mathrm{fwd}},T_{\mathrm{bwd}}$ the costs of graph
  forward and backward evaluation, including repeated uses of
  shared parameters.  Under comparable directional-derivative
  costs, a parameter-HVP takes
  $O(T_{\mathrm{fwd}}+T_{\mathrm{bwd}})$ time.  Computing
  every column of the dense parameter Hessian takes
  $O(P(T_{\mathrm{fwd}}+T_{\mathrm{bwd}})+P^2)$ time
  and $O(P^2)$ storage for the output.  Explicit activation
  blocks additionally depend on the requested block sizes and
  graph operations; no universal count follows from $P$ alone.
\end{proposition}

\begin{proof}
  Differentiate the parameter gradient once in each of the $P$
  coordinate directions as in Algorithm~\ref{alg:full-hessian-dag}.
  Each directional derivative costs a constant number of graph
  sweeps under the stated assumption; storing the $P$ columns
  requires $P^2$ entries.  Sharing reuses parameters but does
  not remove their repeated graph operations.
\end{proof}

\begin{corollary}[Summary of computational costs]
  \label{cor:complexity-summary}
  Let $P$ denote unique parameter coordinates and
  $T_g=T_{\mathrm{fwd}}+T_{\mathrm{bwd}}$ the cost of one
  graph evaluation and gradient, including shared-weight uses.
  (1)~\textbf{Block metrics:} a selected activation block
  $H^f_{v,w}\in\mathbb{R}^{d_v\times d_w}$ need not require
  the full parameter Hessian.  With $m$ norm probes and $T$
  spectral iterations, stochastic estimation requires
  $O((m+2T)T_g)$ time when each required block action and
  adjoint action costs $O(T_g)$; the probes live in the
  appropriate activation spaces, not necessarily $\mathbb{R}^P$.
  (2)~\textbf{Full parameter Hessian:}
  $O(PT_g+P^2)$ time and $O(P^2)$ output storage under the
  assumptions of Proposition~\ref{thm:computational-complexity}.
  (3)~\textbf{Parameter-HVP:} $O(T_g)$ per direction.
  These costs do not assert $T_g=O(P)$ for shared parameters.
\end{corollary}

\section{Low-rank approximation and cost reduction}
\label{app:low-rank}

The rank bound for a separating bottleneck applies to a
pointwise activation-GGN block, not to every full
inter-layer Hessian block.  Graph-distance decay controls
absolute GGN block norms under its stated contraction
hypotheses; it does not establish low rank by itself.

\begin{definition}[Low-rank factorization]
  \label{def:low-rank-factorization}
  For a block
  $H^f_{v,w}\!\in\!\mathbb{R}^{d_v\times d_w}$, define the
  \textbf{rank-$r$ approximation}:
  \begin{equation}\label{eq:low-rank}
    H^f_{v,w}\approx\tilde{H}^f_{v,w}
    :=U_{v,w}S_{v,w}V_{v,w}^\top,
  \end{equation}
  where $U_{v,w}\!\in\!\mathbb{R}^{d_v\times r}$,
  $S_{v,w}\!\in\!\mathbb{R}^{r\times r}$ diagonal,
  $V_{v,w}\!\in\!\mathbb{R}^{d_w\times r}$, and
  $r\ll\min(d_v,d_w)$.
  The optimal Frobenius-norm approximation is given by
  truncated SVD.
\end{definition}

\begin{proposition}[Factorized GGN block at a bottleneck]
  \label{prop:recursive-lr-update}
  If all paths from $v$ and $w$ to the terminal loss pass
  through the same node $u$, then their pointwise GGN block
  factors as
  \begin{equation}\label{eq:lr-propagation}
    H^{GN}_{v,w}=M_v^\top H^{GN}_{u,u}M_w,
    \qquad M_v=\mathcal{J}_{u\gets v},\quad M_w=\mathcal{J}_{u\gets w}.
  \end{equation}
  Its exact rank is at most $d_u$.  This factorization
  gives no fixed-rank guarantee for the full block
  $H^f_{v,w}$, which may include residual sources.
\end{proposition}

\begin{theorem}[Approximation error bound]
  \label{thm:lr-error-bound}
  Let
  $\sigma_1\geq\sigma_2\geq\cdots\geq\sigma_{d_{\min}}$ be the
  singular values of $H^f_{v,w}$
  ($d_{\min}=\min(d_v,d_w)$).
  Then for the rank-$r$ approximation:
  \begin{equation}\label{eq:lr-error}
    \|H^f_{v,w}-\tilde{H}^f_{v,w}\|_F
    =\sqrt{\sum_{i=r+1}^{d_{\min}}\sigma_i^2}
    \leq\sqrt{d_{\min}-r}\cdot\sigma_{r+1}.
  \end{equation}
  In particular, if
  $\sigma_{r+1}\leq\varepsilon\sigma_1$:
  \[
    \frac{\|H^f_{v,w}-\tilde{H}^f_{v,w}\|_F}
    {\|H^f_{v,w}\|_F}
    \leq\varepsilon\sqrt{d_{\min}-r}.
  \]
\end{theorem}

\begin{proof}
  By the Eckart--Young--Mirsky theorem, truncated SVD gives the
  best rank-$r$ approximation in Frobenius norm.  The error
  equals the Frobenius norm of the discarded part:
  $\|H^f_{v,w}-\tilde{H}^f_{v,w}\|_F^2
  =\sum_{i=r+1}^{d_{\min}}\sigma_i^2$.
  The bound via $\sigma_{r+1}$ follows from
  $\sigma_i\leq\sigma_{r+1}$ for $i>r$.
\end{proof}

\begin{corollary}[Rank selection]
  \label{cor:rank-selection}
  For relative accuracy~$\varepsilon$ it suffices to choose
  rank~$r$ such that
  \[
    r=\min\Bigl\{k:
      \frac{\sigma_{k+1}}{\sigma_1}
    \leq\frac{\varepsilon}{\sqrt{d_{\min}-k}}\Bigr\}.
  \]
  This singular-value criterion concerns a chosen full
  block; distance decay of its absolute GGN norm does not
  imply that its relative spectral tail is small.
\end{corollary}

\begin{remark}[Complexity consequences]
  When a selected block admits an accurate rank-$r$
  approximation, its storage and multiplication can use
  $O((d_v+d_w)r)$ rather than $O(d_vd_w)$ entries or
  operations.  A separating bottleneck supplies an exact
  pointwise activation-GGN rank bound; neither this nor
  graph-distance decay guarantees low-rank full blocks,
  topological bandedness, or an $O(nk_\varepsilon dr)$
  representation of the full parameter Hessian.
\end{remark}

\begin{remark}[Cost reduction methods]
  Approaches for reducing the computational cost of the full
  Hessian:
  (1)~\textbf{Block approximation:} retaining selected
  diagonal blocks reduces output storage relative to a full
  matrix, but construction cost depends on the graph and
  the chosen derivative method.
  (2)~\textbf{Low-rank approximation:} once suitable factors
  have been obtained, their storage and action costs are as
  above; obtaining them incurs additional graph-evaluation
  or matrix-factorization work.
  (3)~\textbf{Gauss--Newton approximation:} using the
  terminal-loss Jacobian pullback
  $J^\top\nabla^2\!\mathcal{L}\,J$ omits local second-order
  sources and is PSD when the terminal loss Hessian is PSD.
  (4)~\textbf{Kronecker factorization:} representing matrix
  blocks as Kronecker products of smaller matrices.
\end{remark}

\begin{remark}[Stability and invariance]
  Numerical stability depends on derivative scales,
  conditioning, and cancellation between transported
  contributions.  Depth alone does not bound relative
  error or provide a universal float32 stability threshold;
  nearly cancelling paths can amplify small absolute errors.
  The rescaling behavior of geometric coupling depends on
  the chosen activation coordinates and on whether a weight
  rescaling actually induces a compensating change of those
  coordinates; the general metric is not invariant under
  arbitrary weight changes.  Detailed analysis:
  \ref{app:stability}.
\end{remark}

\begin{remark}[Convergence of optimization methods]
  For a smooth objective near a nondegenerate stationary
  minimum, classical undamped Newton iteration has local
  superlinear convergence under a continuous nonsingular
  Hessian and quadratic convergence when that Hessian is
  locally Lipschitz.  These conclusions do not follow for
  a fixed regularization $H+\lambda I$ or for a
  nonsmooth activation boundary from the assumptions used
  here; see \ref{app:convergence}.
\end{remark}

\section{Smooth-region limits at activation boundaries}
\label{app:clarke-continuity}

\begin{theorem}[Regionwise continuity, not boundary continuity]
  Let $F$ be a finite affine/ReLU DAG with fixed parameters
  and $\ell\in C^2$. The classical Hessian of
  $\mathcal{L}=\ell\circ F$ is continuous in the interior
  of each activation region. At a boundary it need not
  extend continuously, and the gradient need not be
  continuous. A Clarke-gradient Hessian can be invoked
  only with the additional local $C^{1,1}$ condition.
\end{theorem}

\begin{proof}
  On the interior of a region $F(x)=A_rx+b_r$, so
  $\nabla^2\mathcal{L}(x)
  =A_r^\top\nabla^2\ell(F(x))A_r$ is continuous.
  For $F(x)=\operatorname{ReLU}(x)$ and
  $\ell(y)=\tfrac12(y-1)^2$, the one-sided gradients at
  zero are $0$ and $-1$: local Lipschitz continuity of the
  loss does not give a Clarke-gradient Hessian there.
  Even for $\ell(y)=\tfrac12 y^2$, where the gradient is
  locally Lipschitz, the one-sided Hessians are $0$ and $1$.
  Thus no boundary continuity of the classical or
  conventionally computed AD-Hessian follows.
\end{proof}

\section{Example: Attention block (\texorpdfstring{Softmax~$\times$~Value}{Softmax x Value})}
\label{app:attention-example}

To illustrate the framework on an Attention subgraph, we analyze
a single head; this is not an analysis of a full Transformer.

\subsection{DAG representation}

Single-head Attention:
input $X\!\in\!\mathbb{R}^{n\times d}$, three linear nodes:
\begin{gather*}
  Q=XW_Q,\quad K=XW_K,\quad V=XW_V,\\
  W_Q,W_K\in\mathbb{R}^{d\times d_k},\;
  W_V\in\mathbb{R}^{d\times d_v}.
\end{gather*}
Softmax node:
$A=\sigma(QK^\top/\sqrt{d_k})\in\mathbb{R}^{n\times n}$,
where $\sigma(Z)_{ij}=\exp(Z_{ij})/\sum_k\exp(Z_{ik})$.
Output node: $O=AV\in\mathbb{R}^{n\times d_v}$.

\begin{remark}[DAG structure of Attention]
  In terms of our formalism, the Attention block is a DAG with
  5 nodes:
  $\{v_{\mathrm{in}},v_Q,v_K,v_V,v_{\mathrm{out}}\}$,
  where $v_{\mathrm{out}}$ combines Softmax and multiplication
  by~$V$.
  Parameters: $\theta_Q=W_Q$, $\theta_K=W_K$,
  $\theta_V=W_V$.
  Parents of $v_{\mathrm{out}}$:
  $\Pa(v_{\mathrm{out}})=\{v_Q,v_K,v_V\}$.
\end{remark}

\subsection{Softmax Jacobian and Hessian}

The Softmax Jacobian w.r.t.\ logits
$z\!\in\!\mathbb{R}^n$ (fixing a row of
$QK^\top/\sqrt{d_k}$):
\begin{equation}\label{eq:softmax-jacobian}
  \frac{\partial\sigma_i}{\partial z_j}
  =\sigma_i(\delta_{ij}-\sigma_j)=:S_{ij},
\end{equation}
where $S=\mathrm{diag}(\sigma)-\sigma\sigma^\top$ is a
dense $n\!\times\!n$ matrix.
The Hessian (tensor~$T$ in our notation) for the $i$-th
output:
\begin{equation}\label{eq:softmax-hessian}
  [T_{\sigma;z}]_{i,j,k}
  =\frac{\partial^2\sigma_i}{\partial z_j\partial z_k}
  =\sigma_i\bigl(\delta_{ij}\delta_{ik}
    -\delta_{ij}\sigma_k-\delta_{ik}\sigma_j
  -\delta_{jk}\sigma_j+2\sigma_j\sigma_k\bigr).
\end{equation}

\begin{remark}[Density of Softmax cross-blocks]
  For more than one unmasked position, $S$ couples logits within
  a row.  This does not guarantee nonzero entries of the Softmax
  Hessian or of the loss-Hessian blocks.  For example, binary
  Softmax at equal logits has a zero Hessian.  ReLU's local
  second derivative is zero a.e., not a source of diagonal curvature.
\end{remark}

\subsection{Cross-blocks between Q, K, and V}

Applying formula~\eqref{eq:Hf} to node $v_{\mathrm{out}}$
with $\Pa(v_{\mathrm{out}})=\{v_Q,v_K,v_V\}$:

\begin{proposition}[Attention cross-blocks]
  \label{prop:attention-cross-blocks}
  The cross-blocks of the \emph{input} Hessian contain transported
  downstream curvature and possible local mixed sources:
  \begin{enumerate}
    \item $H^f_{Q,V}$: the mixed derivative of $O=AV$ contains
      a first derivative of $A$ with respect to $Q$, acting on
      the Value perturbation.
    \item $H^f_{Q,K}$: the local mixed derivative includes both
      the Softmax Hessian and its Jacobian acting on the bilinear
      score map $QK^\top/\sqrt{d_k}$.
    \item $H^f_{K,V}$: similarly contains a first derivative
      of $A$ with respect to $K$, acting on the Value perturbation.
  \end{enumerate}
  These source descriptions do not imply that every block entry
  is nonzero, or identify a dominant contribution to $H^T$.
\end{proposition}

\begin{proof}[Sketch]
  For $Z=QK^\top/\sqrt{d_k}$ and directions $U$ in $Q$ and
  $W$ in $K$, the local mixed derivative is
  \[
    \mathrm{D}^2_{Q,K}O[U,W]
    =\left(\mathrm{D}^2\sigma(Z)
      [UK^\top/\sqrt{d_k},QW^\top/\sqrt{d_k}]
      +\mathrm{D}\sigma(Z)[UW^\top/\sqrt{d_k}]\right)V.
  \]
  The second term is the bilinear-score contribution and can
  survive even when the Softmax Hessian vanishes.  Equation~\eqref{eq:Hf}
  contracts local derivatives with the downstream adjoint and
  transports curvature; either operation, or cancellation, can
  make the resulting residual zero.  For scalar prediction $p$
  with loss $(p-y)^2$, $H^T_{Q,K}=2(p-y)\nabla^2_{Q,K}p$,
  so it vanishes at zero prediction error.
\end{proof}

\subsection{Consequences for architectural analysis}

\begin{remark}[Structural consequences of Attention]
  \label{rem:attention-consequences}
  \leavevmode
  \begin{enumerate}
    \item \textbf{Possible cross-position coupling:} Attention
      permits interactions across positions; this alone establishes
      neither full numerical density nor large effective dimension.
      A banded approximation needs quantitative justification.
    \item \textbf{Q--K curvature:} a large measured GN-Gap
      indicates a large combined residual relative to GGN for that
      pair, not Softmax dominance or a cause of training instability.
    \item \textbf{State dependence:} local source tensors,
      downstream adjoints, and their transported sums all matter.
      Architecture alone does not fix the residual magnitude.
  \end{enumerate}
  Multi-head analysis must also include interactions induced by
  shared inputs, output projection, and the downstream loss;
  independent head computations do not imply a block-diagonal
  loss Hessian.
\end{remark}

\section{Additional experimental results}
\label{app:exp-results}

Table~\ref{tab:exp2-depth8} reports the metrics
$\mathcal{D}_{\mathrm{far}}$ and
$\mathcal{C}_{\mathrm{far}}$ for BottleneckMLP with $L\!=\!8$
(protocol identical to Exp.~2 in the main text, $L\!=\!6$).
Qualitatively, the archived stochastic estimates
$\mathcal{D}_{\mathrm{far}}$ grow with~$d_u$ at initialization.
For $d_u\geq64$, the $L=8$ estimates are lower than those for
$L=6$, while the final values are less width-dependent.
These are descriptive profiles subject to~\ref{app:ablation},
not evidence for an exact-rank compression theorem.

\begin{table}[htbp]
  \centering
  \caption{Exp.\,2: bottleneck ablation ($L\!=\!8$, CIFAR-100).
    Same protocol as Table~\ref{tab:exp2}
  (mean $\pm 1\sigma$ over 5~seeds).}
  \label{tab:exp2-depth8}
  {\footnotesize
    \begin{tabular}{@{}rccccc@{}}
      \toprule
      $d_u$
      & $\mathcal{D}_{\mathrm{far}}^{\mathrm{init}}$
      & $\mathcal{D}_{\mathrm{far}}^{\mathrm{final}}$
      & $\mathcal{C}_{\mathrm{far}}^{\mathrm{init}}$
      & $\mathcal{C}_{\mathrm{far}}^{\mathrm{final}}$
      & Acc\,(\%) \\
      \midrule
      4 & $1.58 \pm 0.43$ & $2.44 \pm 0.17$ & $0.20 \pm 0.08$ & $0.33 \pm 0.01$ & 9.6$^{\ddagger}$ \\
      8 & $2.43 \pm 0.30$ & $3.92 \pm 0.39$ & $0.34 \pm 0.08$ & $0.45 \pm 0.04$ & 14.1 \\
      16 & $4.27 \pm 0.57$ & $5.01 \pm 0.62$ & $0.43 \pm 0.03$ & $0.56 \pm 0.01$ & 17.7 \\
      32 & $6.36 \pm 0.69$ & $5.13 \pm 0.45$ & $0.51 \pm 0.03$ & $0.62 \pm 0.01$ & 19.1 \\
      64 & $6.40 \pm 0.60$ & $5.41 \pm 0.55$ & $0.58 \pm 0.02$ & $0.68 \pm 0.03$ & 20.5 \\
      128 & $7.90 \pm 0.28$ & $4.89 \pm 0.53$ & $0.65 \pm 0.01$ & $0.76 \pm 0.02$ & 21.4 \\
      256 & $8.60 \pm 0.51$ & $5.73 \pm 0.47$ & $0.69 \pm 0.01$ & $0.79 \pm 0.02$ & 22.2 \\
      \midrule
      $512^\dagger$ & $9.21 \pm 0.65$ & $4.89 \pm 0.33$ & $0.70 \pm 0.02$ & $0.81 \pm 0.01$ & 22.4 \\
      \bottomrule
  \end{tabular}}
\end{table}

\FloatBarrier
\subsection{Exp.~1b: Spectral-normalization control}
\label{app:exp1b}

\textbf{Motivation.}
In the standard regime (He~init, width$\,=64$)
$\rho_{\max}\!>\!1$ and the sufficient condition of
Theorem~\ref{thm:exponential-decay-main} is not satisfied
(Section~\ref{sec:exp1}).
We repeat the Exp.~1 protocol with two modifications:
(i)~spectral normalization~\citep{miyato2018spectral} on the
Linear layers of Plain~MLP;
(ii)~LayerNorm disabled to ensure piecewise linearity.
For ReLU, exact weight normalization to $\|W_i\|_2=1$ would give
$\|D_{\mathrm{block}}\|_2\leq1$, but the implementation uses
an iterative spectral-norm estimate.  Moreover, the decay theorem
requires the strict condition $s\rho<1$, not $s\rho\leq1$.
Reported initialization norms are below one
($\rho_{\max}\!\approx\!0.91$), without certifying a uniform
bound over every example.

\textbf{Protocol.}
Plain~MLP (width~64, depths
$L\!\in\!\{8,12,16\}$, CIFAR-10, 50 epochs) with spectral
normalization.
Residual~MLP as control (no SN, as in Exp.~1).
All metrics $\bar{\mathcal{R}}(d)$,
$\bar{\mathcal{C}}(d)$, $\rho$, $\lambda_1$ computed by the
same protocol.

\textbf{Results}
(Table~\ref{tab:exp1b-supp},
Figure~\ref{fig:decay-R-sn-supp}).
(a)~The reported initialization norms are below one; after
training the $L=8$ norms approach the theorem's boundary.
(b)~The initial $\bar{\mathcal{R}}(d)$ profiles decline with
$d$, qualitatively consistent with contraction.  These
batch-averaged profiles do not directly test the pointwise bound.
After training, decay persists for $L\geq12$, while the $L=8$
profile is nearly flat.
(c)~The reported Lyapunov exponents are negative for Plain~SN,
not for the Residual control (Table~\ref{tab:exp1b-supp}).
(d)~Plain~SN accuracy falls with depth
(36.3\% at $L=8$; 10.0\% at $L\geq12$), whereas the control
reaches approximately 51\%.  This association does not establish
curvature suppression as the cause of poor training.

\begin{table}[htbp]
  \centering
  \caption{Exp.~1b: spectral normalization.
    Same metrics as Table~\ref{tab:rho-accuracy}.
    Plain~SN has negative reported $\lambda_1$;
    Residual~MLP is the control without SN
  (CIFAR-10, mean over 5~seeds).}
  \label{tab:exp1b-supp}
  {\footnotesize
    \begin{tabular}{@{}llccccc@{}}
      \toprule
      $L$ & Arch.
      & $\rho_{\max}^{\mathrm{init}}$
      & $\rho_{\max}^{\mathrm{final}}$
      & $\lambda_1^{\mathrm{init}}$
      & $\lambda_1^{\mathrm{final}}$
      & Acc\,(\%) \\
      \midrule
      8 & Plain SN & 0.91 & 1.00 & $-$1.01 & $-$0.24 & 36.3 \\
      8 & Res. & 1.25 & 1.33 & 0.02 & 0.04 & 51.3 \\
      12 & Plain SN & 0.91 & 0.89 & $-$1.01 & $-$1.03 & 10.0 \\
      12 & Res. & 1.20 & 1.25 & 0.01 & 0.03 & 51.3 \\
      16 & Plain SN & 0.91 & 0.90 & $-$1.01 & $-$1.03 & 10.0 \\
      16 & Res. & 1.17 & 1.20 & 0.01 & 0.02 & 51.0 \\
      \bottomrule
  \end{tabular}}
\end{table}

\begin{figure}[htbp]
  \centering
  \begin{tikzpicture}
    \begin{axis}[
        width=0.95\columnwidth,
        height=5.5cm,
        xlabel={Distance $d$},
        ylabel={$\bar{\mathcal{R}}(d)$},
        ymode=log,
        grid=major,
        grid style={gray!30},
        legend style={font=\scriptsize, at={(0.97,0.97)},
        anchor=north east},
        tick label style={font=\footnotesize},
        label style={font=\small},
        xmin=-0.5, xmax=16,
      ]
      \addplot[blue, thick, mark=*, mark size=1.5,
      error bars/.cd, y dir=both, y explicit] coordinates {
        (0,1.5653e-02) +- (0,4.8008e-04) (1,6.4308e-03) +- (0,5.3708e-04)
        (2,2.8765e-03) +- (0,2.7328e-04) (3,1.2474e-03) +- (0,1.5963e-04)
        (4,5.6542e-04) +- (0,7.1612e-05) (5,2.7099e-04) +- (0,3.3374e-05)
        (6,1.3360e-04) +- (0,1.1005e-05) (7,7.6334e-05) +- (0,9.6825e-06)
      }; \addlegendentry{$L\!=\!8$ init}
      \addplot[red, thick, mark=square*, mark size=1.5,
      error bars/.cd, y dir=both, y explicit] coordinates {
        (0,1.0732e-02) +- (0,3.2025e-04) (1,4.1436e-03) +- (0,3.9320e-04)
        (2,1.6983e-03) +- (0,1.9441e-04) (3,6.9690e-04) +- (0,1.0916e-04)
        (4,2.7737e-04) +- (0,4.3527e-05) (5,1.1024e-04) +- (0,1.3714e-05)
        (6,4.9188e-05) +- (0,5.9815e-06) (7,2.1528e-05) +- (0,3.2495e-06)
        (8,1.0084e-05) +- (0,2.4589e-06) (9,4.7410e-06) +- (0,1.1219e-06)
        (10,2.3928e-06) +- (0,4.3804e-07) (11,1.3607e-06) +- (0,3.0890e-07)
      }; \addlegendentry{$L\!=\!12$ init}
      \addplot[green!50!black, thick, mark=triangle*, mark size=1.5,
      error bars/.cd, y dir=both, y explicit] coordinates {
        (0,8.0152e-03) +- (0,4.0587e-04) (1,3.1455e-03) +- (0,4.0464e-04)
        (2,1.2628e-03) +- (0,2.4618e-04) (3,4.9627e-04) +- (0,1.0330e-04)
        (4,2.0364e-04) +- (0,5.3188e-05) (5,7.6368e-05) +- (0,1.1991e-05)
        (6,3.3081e-05) +- (0,4.1085e-06) (7,1.3318e-05) +- (0,2.2328e-06)
        (8,5.3274e-06) +- (0,8.4754e-07) (9,2.0933e-06) +- (0,3.9471e-07)
        (10,9.8834e-07) +- (0,2.2733e-07) (11,4.5372e-07) +- (0,1.2928e-07)
        (12,2.1633e-07) +- (0,7.3385e-08) (13,1.0049e-07) +- (0,3.6587e-08)
        (14,5.0707e-08) +- (0,1.8826e-08) (15,2.7882e-08) +- (0,1.0312e-08)
      }; \addlegendentry{$L\!=\!16$ init}
      \addplot[blue, thick, dashed, mark=o, mark size=1.5,
      error bars/.cd, y dir=both, y explicit] coordinates {
        (0,9.4824e-01) +- (0,1.0231e-01) (1,9.4935e-01) +- (0,1.0188e-01)
        (2,9.4892e-01) +- (0,1.0191e-01) (3,9.4795e-01) +- (0,1.0213e-01)
        (4,9.4631e-01) +- (0,1.0253e-01) (5,9.4363e-01) +- (0,1.0324e-01)
        (6,9.3915e-01) +- (0,1.0456e-01) (7,9.2948e-01) +- (0,1.0732e-01)
      }; \addlegendentry{$L\!=\!8$ final}
      \addplot[red, thick, dashed, mark=square, mark size=1.5,
      error bars/.cd, y dir=both, y explicit] coordinates {
        (0,8.3044e-03) +- (0,2.4803e-04) (1,2.8727e-03) +- (0,3.2184e-04)
        (2,1.1404e-03) +- (0,1.3983e-04) (3,4.5885e-04) +- (0,7.6923e-05)
        (4,1.8760e-04) +- (0,3.8182e-05) (5,7.4475e-05) +- (0,1.2516e-05)
        (6,3.2678e-05) +- (0,5.8262e-06) (7,1.4217e-05) +- (0,3.4462e-06)
        (8,6.7653e-06) +- (0,2.4144e-06) (9,3.1067e-06) +- (0,9.8487e-07)
        (10,1.5570e-06) +- (0,4.0521e-07) (11,8.8042e-07) +- (0,2.8639e-07)
      }; \addlegendentry{$L\!=\!12$ final}
    \end{axis}
  \end{tikzpicture}
  \caption{Exp.~1b: resonance $\bar{\mathcal{R}}(d)$ under
    spectral normalization (Plain~SN).
    Solid curves --- initialization
    ($\rho\!\approx\!0.91\!<\!1$): exponential decay over
    ${\sim}5.5$ orders of magnitude at $L\!=\!16$, qualitatively
    consistent with contraction rather than a direct pointwise
    test of~\eqref{eq:resonance-decay-main}.
    Dashed --- after training:
    at $L\!=\!8$ ($\rho\!\to\!1.0$) $R$ plateaus at
    ${\approx}0.94$ (boundary of the theorem condition);
    at $L\!=\!12$ decay persists ($\lambda_1\!<\!0$).
  Error bands: $\pm 1\sigma$ over 5~seeds.}
  \label{fig:decay-R-sn-supp}
\end{figure}

\FloatBarrier
\subsection{Exp.~1: depths \texorpdfstring{$L\!=\!16$ and $L\!=\!32$}{L=16 and L=32}}
\label{app:exp1-deep}

Table~\ref{tab:exp1-deep} extends the Exp.~1 results to
depths $L\!\in\!\{16,32\}$.
Exponential decay of $\bar{\mathcal{R}}(d)$ persists up to
$L\!=\!32$:
the ratio
$\bar{\mathcal{R}}(0)/\bar{\mathcal{R}}(L{-}1)$ reaches
$5.9\times$ ($L\!=\!32$, plain, init), and the normalized
coupling $\bar{\mathcal{C}}(L{-}1)$ drops to~0.145.
Residual~MLP has flatter resonance profiles than Plain in these
runs, but its final profiles are not constant with distance.

For $L\!=\!32$ (plain, final),
$\bar{\mathcal{R}}(d)$ exhibits a U-shaped pattern:
decay from~1.75 to a minimum of ${\approx}0.33$ at $d\!=\!14$,
then rise to~0.52 at $d\!=\!31$.
Non-uniform self-curvature and the changing composition of
pair-distance bins can affect this profile; the saved summaries
do not isolate its cause.  Normalized coupling decreases with
distance, illustrating a different aspect of the same blocks.

\begin{table}[htbp]
  \centering
  \caption{Exp.~1: metrics for depths $L\!\in\!\{16,32\}$
    (CIFAR-10, width$\,=128$, mean $\pm 1\sigma$ over 5~seeds;
      stochastic estimation: 100 Rademacher probes,
    subsample~64).
    $\bar{\mathcal{R}}(0)$/$\bar{\mathcal{R}}(L{-}1)$ ---
    resonance of nearest/most distant pairs;
    $\bar{\mathcal{C}}(L{-}1)$ --- coupling at maximum
  distance.}
  \label{tab:exp1-deep}
  {\footnotesize
    \begin{tabular}{@{}llccccc@{}}
      \toprule
      $L$ & Arch. & Phase
      & $\rho_{\max}$
      & $\bar{\mathcal{R}}(0)$
      & $\bar{\mathcal{R}}(L{-}1)$
      & $\bar{\mathcal{C}}(L{-}1)$ \\
      \midrule
      16 & Plain & init & $2.59 \pm 0.19$ & $0.072 \pm 0.018$ & $0.021 \pm 0.001$ & $0.263$ \\
      16 & Plain & final & $2.56 \pm 0.03$ & $4.852 \pm 1.219$ & $0.80 \pm 0.13$ & $0.624$ \\
      16 & Res. & init & $1.17 \pm 0.00$ & $0.121 \pm 0.005$ & $0.12 \pm 0.00$ & $0.983$ \\
      16 & Res. & final & $1.21 \pm 0.01$ & $2.156 \pm 0.111$ & $1.44 \pm 0.08$ & $0.887$ \\
      32 & Plain & init & $2.71 \pm 0.13$ & $0.067 \pm 0.018$ & $0.011 \pm 0.003$ & $0.145$ \\
      32 & Plain & final & $2.62 \pm 0.03$ & $1.750 \pm 0.530$ & $0.52 \pm 0.16$ & $0.403$ \\
      32 & Res. & init & $1.12 \pm 0.00$ & $0.120 \pm 0.004$ & $0.12 \pm 0.00$ & $0.986$ \\
      32 & Res. & final & $1.13 \pm 0.00$ & $2.176 \pm 0.082$ & $1.48 \pm 0.05$ & $0.915$ \\
      \bottomrule
  \end{tabular}}
\end{table}

\FloatBarrier
\subsection{Exp.~4: GN-Gap by branch depth \texorpdfstring{$k\!\in\!\{1,2,3\}$}{k in \{1,2,3\}}}
\label{app:exp4-ksweep}

Table~\ref{tab:exp4-sweep} extends the Exp.~4 results from the
main text ($k\!=\!2$) with data for $k\!=\!1$ and $k\!=\!3$.
The qualitative picture is independent of~$k$:
cat+SiLU (nonlinear merging, $\sigma''\!\neq\!0$) exhibits
$\mathrm{Gap}_{GN}\!>\!1$ at initialization for all~$k$;
remaining configurations stay at machine precision
($\sim\!10^{-8}$).
The variation of Gap across~$k$ for cat+SiLU is less than 3\%
at init and ${\sim}15\%$ at final, confirming the robustness
of the conclusion about the dominant role of merging type and
activation, not branch depth.

\begin{table}[htbp]
  \centering
  \caption{Exp.~4: GN-Gap at the merging node of Diamond~MLP
    for $k\!\in\!\{1,2,3\}$ (width$\,=32$, CIFAR-10,
    mean over 5~seeds).  Columns $\|T\|_F$ denote the measured
  residual norm $\|H^T_{A,B}\|_F$, as in Table~\ref{tab:exp4}.}
  \label{tab:exp4-sweep}
  {\footnotesize
    \begin{tabular}{@{}clcccc@{}}
      \toprule
      $k$ & Configuration
      & $\mathrm{Gap}^{\mathrm{init}}$
      & $\mathrm{Gap}^{\mathrm{final}}$
      & $\|T\|_F^{\mathrm{init}}$
      & $\|T\|_F^{\mathrm{final}}$ \\
      \midrule
      1 & sum+ReLU & $6.5{\cdot}10^{-8}$ & $8.7{\cdot}10^{-8}$ & $3.6{\cdot}10^{-9}$ & $2.5{\cdot}10^{-8}$ \\
      1 & sum+SiLU & $6.5{\cdot}10^{-8}$ & $9.2{\cdot}10^{-8}$ & $3.5{\cdot}10^{-9}$ & $2.6{\cdot}10^{-8}$ \\
      1 & cat+ReLU & $8.1{\cdot}10^{-8}$ & $1.2{\cdot}10^{-7}$ & $6.0{\cdot}10^{-10}$ & $5.6{\cdot}10^{-8}$ \\
      1 & cat+SiLU & $1.35$ & $0.078$ & $6.6{\cdot}10^{-3}$ & $3.3{\cdot}10^{-2}$ \\
      \midrule
      2 & sum+ReLU & $6.7{\cdot}10^{-8}$ & $9.3{\cdot}10^{-8}$ & $4.1{\cdot}10^{-9}$ & $2.0{\cdot}10^{-8}$ \\
      2 & sum+SiLU & $7.0{\cdot}10^{-8}$ & $9.6{\cdot}10^{-8}$ & $4.1{\cdot}10^{-9}$ & $2.0{\cdot}10^{-8}$ \\
      2 & cat+ReLU & $8.4{\cdot}10^{-8}$ & $1.2{\cdot}10^{-7}$ & $6.6{\cdot}10^{-10}$ & $3.3{\cdot}10^{-8}$ \\
      2 & cat+SiLU & $1.33$ & $0.082$ & $6.8{\cdot}10^{-3}$ & $2.5{\cdot}10^{-2}$ \\
      \midrule
      3 & sum+ReLU & $6.5{\cdot}10^{-8}$ & $8.9{\cdot}10^{-8}$ & $3.8{\cdot}10^{-9}$ & $1.7{\cdot}10^{-8}$ \\
      3 & sum+SiLU & $7.1{\cdot}10^{-8}$ & $9.4{\cdot}10^{-8}$ & $4.1{\cdot}10^{-9}$ & $1.9{\cdot}10^{-8}$ \\
      3 & cat+ReLU & $8.3{\cdot}10^{-8}$ & $1.4{\cdot}10^{-7}$ & $7.6{\cdot}10^{-10}$ & $3.0{\cdot}10^{-8}$ \\
      3 & cat+SiLU & $1.32$ & $0.090$ & $7.0{\cdot}10^{-3}$ & $2.1{\cdot}10^{-2}$ \\
      \bottomrule
  \end{tabular}}
\end{table}

\FloatBarrier
\subsection{Exp.~3: decomposition of \texorpdfstring{$\|H^T\|_F$ and $\|H^{GN}\|_F$}{||H\textasciicircum T||
and ||H\textasciicircum GN||} by distance}
\label{app:exp3-decomp}

Table~\ref{tab:exp3-decomp} reports the Frobenius norms of the
tensor and Gauss--Newton components of the inter-layer Hessian
by distance~$d$ for three smooth activations (init, $L\!=\!6$,
width$\,=64$).
Both components decay exponentially with~$d$;
at maximum distance $d\!=\!L{-}1\!=\!5$,
$\|H^T\|_F\to 0$ because the later measured node feeds an
affine classifier.  For that pair there is no mixed second
derivative of the logits, as in Remark~\ref{rem:linear-head}.

\begin{table}[htbp]
  \centering
  \caption{Exp.~3: $\|H^T\|_F$ and $\|H^{GN}\|_F$ by
    distance~$d$ (init, $L\!=\!6$, width$\,=64$, CIFAR-10,
    mean over 5~seeds).
    At $d\!=\!5$ the tensor component vanishes
  ($H^T_{d=L-1}\!\equiv\!0$).}
  \label{tab:exp3-decomp}
  {\footnotesize
    \begin{tabular}{@{}clccc@{}}
      \toprule
      $d$ & Activation
      & $\|H^T\|_F$
      & $\|H^{GN}\|_F$
      & $\mathrm{Gap}_{GN}$ \\
      \midrule
      0 & Softplus & $1.50{\cdot}10^{-2}$ & $1.22{\cdot}10^{-1}$ & $0.123$ \\
      0 & SiLU & $2.65{\cdot}10^{-2}$ & $1.17{\cdot}10^{-1}$ & $0.227$ \\
      0 & GELU & $4.22{\cdot}10^{-2}$ & $1.17{\cdot}10^{-1}$ & $0.361$ \\
      \midrule
      2 & Softplus & $1.24{\cdot}10^{-3}$ & $1.05{\cdot}10^{-2}$ & $0.119$ \\
      2 & SiLU & $2.15{\cdot}10^{-3}$ & $9.59{\cdot}10^{-3}$ & $0.224$ \\
      2 & GELU & $3.45{\cdot}10^{-3}$ & $9.70{\cdot}10^{-3}$ & $0.357$ \\
      \midrule
      4 & Softplus & $7.68{\cdot}10^{-5}$ & $8.63{\cdot}10^{-4}$ & $0.089$ \\
      4 & SiLU & $1.32{\cdot}10^{-4}$ & $7.88{\cdot}10^{-4}$ & $0.168$ \\
      4 & GELU & $2.15{\cdot}10^{-4}$ & $8.04{\cdot}10^{-4}$ & $0.269$ \\
      \midrule
      5 & Softplus & ${<}\,10^{-10}$ & $2.37{\cdot}10^{-4}$ & ${<}\,10^{-7}$ \\
      5 & SiLU & ${<}\,10^{-10}$ & $2.08{\cdot}10^{-4}$ & ${<}\,10^{-7}$ \\
      5 & GELU & ${<}\,10^{-10}$ & $2.15{\cdot}10^{-4}$ & ${<}\,10^{-7}$ \\
      \bottomrule
  \end{tabular}}
\end{table}

\FloatBarrier
\newpage
\subsection{\texorpdfstring{Exp.~4: cross-branch resonance
$\bar{\mathcal{R}}_{AB}$}{Exp. 4: cross-branch resonance R\_AB}}
\label{app:exp4-cross-branch}

\begin{figure}[htbp]
  \centering
  \begin{tikzpicture}
    \begin{axis}[
        width=0.95\columnwidth,
        height=5.5cm,
        xlabel={Graph distance $d_{\mathrm{graph}}$},
        ylabel={$\bar{\mathcal{R}}_{AB}(d_{\mathrm{graph}})$},
        ymode=log,
        grid=major,
        grid style={gray!30},
        xtick={2,3,4},
        tick label style={font=\footnotesize},
        label style={font=\small},
        legend style={font=\tiny, at={(0.5,1.02)},
        anchor=south, legend columns=4},
      ]
      \addplot[blue, thick, mark=triangle*, mark size=2,
      error bars/.cd, y dir=both, y explicit] coordinates {
        (2,6.0992e-02) +- (0,4.3401e-03) (3,4.7855e-02) +- (0,4.0656e-03)
        (4,3.7628e-02) +- (0,6.4583e-03)
      }; \addlegendentry{sum+ReLU init}
      \addplot[red, thick, mark=square*, mark size=2,
      error bars/.cd, y dir=both, y explicit] coordinates {
        (2,7.7722e-03) +- (0,1.0071e-03) (3,6.5686e-03) +- (0,5.8953e-04)
        (4,5.5552e-03) +- (0,6.5934e-04)
      }; \addlegendentry{cat+ReLU init}
      \addplot[blue!60!green, thick, mark=diamond*, mark size=2,
      error bars/.cd, y dir=both, y explicit] coordinates {
        (2,5.9324e-02) +- (0,3.6116e-03) (3,5.1067e-02) +- (0,3.3746e-03)
        (4,4.4352e-02) +- (0,5.3851e-03)
      }; \addlegendentry{sum+SiLU init}
      \addplot[orange, thick, mark=pentagon*, mark size=2,
      error bars/.cd, y dir=both, y explicit] coordinates {
        (2,8.6942e-03) +- (0,6.6118e-04) (3,9.0981e-03) +- (0,9.9226e-04)
        (4,9.4942e-03) +- (0,1.2580e-03)
      }; \addlegendentry{cat+SiLU init}
      \addplot[blue, thick, dashed, mark=triangle, mark size=2,
      error bars/.cd, y dir=both, y explicit] coordinates {
        (2,2.1234e-01) +- (0,1.1395e-02) (3,4.5951e-01) +- (0,5.2584e-02)
        (4,1.0244e+00) +- (0,1.8796e-01)
      }; \addlegendentry{sum+ReLU final}
      \addplot[red, thick, dashed, mark=square, mark size=2,
      error bars/.cd, y dir=both, y explicit] coordinates {
        (2,2.7134e-01) +- (0,9.5812e-03) (3,5.6752e-01) +- (0,3.7635e-02)
        (4,1.2114e+00) +- (0,1.2608e-01)
      }; \addlegendentry{cat+ReLU final}
      \addplot[blue!60!green, thick, dashed, mark=diamond,
        mark size=2,
      error bars/.cd, y dir=both, y explicit] coordinates {
        (2,2.0974e-01) +- (0,1.3604e-02) (3,5.2587e-01) +- (0,5.0455e-02)
        (4,1.3570e+00) +- (0,1.7463e-01)
      }; \addlegendentry{sum+SiLU final}
      \addplot[orange, thick, dashed, mark=pentagon, mark size=2,
      error bars/.cd, y dir=both, y explicit] coordinates {
        (2,3.1208e-01) +- (0,2.8817e-02) (3,7.3811e-01) +- (0,8.4158e-02)
        (4,1.7766e+00) +- (0,3.2019e-01)
      }; \addlegendentry{cat+SiLU final}
    \end{axis}
  \end{tikzpicture}
  \caption{Exp.~4: cross-branch resonance
    $\bar{\mathcal{R}}_{AB}(d_{\mathrm{graph}})$ in
    Diamond~MLP ($k\!=\!2$).
    Solid curves --- initialization (decline except for
    cat+SiLU, which rises slightly; H4.2a is not universal);
    dashed --- after training (growth, H4.2b).
    All 4 configurations show $R$ growth with
    $d_{\mathrm{graph}}$ at final.
  Error bands: $\pm 1\sigma$ over 5~seeds.}
  \label{fig:cross-branch-R}
\end{figure}

\FloatBarrier
\subsection{Ablation: exact vs.\ stochastic estimation}
\label{app:ablation}

To assess sensitivity to the Hessian computation method, we compare
the saved exact and stochastic Exp.~3 aggregates on the same
configurations.  The historical run script specifies 100~Rademacher
probes, rather than the previously stated 30; the JSON summaries
do not independently record the probe count.
Table~\ref{tab:ablation} retains the reported init/final discrepancies.

\begin{table}[htbp]
  \centering
  \caption{Ablation: GN-Gap discrepancy between exact and
    stochastic estimation (probe provenance discussed above)
  (Exp.~3, $L\!=\!6$, width$\,=64$, mean over 5~seeds).}
  \label{tab:ablation}
  {\footnotesize
    \begin{tabular}{@{}lcccc@{}}
      \toprule
      Activation
      & $\mathrm{Gap}_{\mathrm{exact}}^{\mathrm{init}}$
      & $\mathrm{Gap}_{\mathrm{stoch}}^{\mathrm{init}}$
      & $\Delta^{\mathrm{init}}$\,(\%)
      & $\Delta^{\mathrm{final}}$\,(\%) \\
      \midrule
      Softplus & $0.122$ & $0.123$ & $0.82$ & $0.25$ \\
      SiLU & $0.225$ & $0.226$ & $0.46$ & $0.33$ \\
      GELU & $0.358$ & $0.359$ & $0.43$ & $0.31$ \\
      \bottomrule
  \end{tabular}}
\end{table}

The tabulated discrepancies are below one percent at init and
final, but the saved mid-checkpoint aggregates exceed one percent
for all three smooth activations.  Thus neither a uniform error
threshold nor a systematic overestimate across checkpoints is
established.  These are comparisons of saved aggregates, not a
replication or an estimator error guarantee.

Some historical stochastic $\mathcal{D}$ estimates fall below one,
which is impossible for the exact stable rank of a nonzero block.
Besides sampling and small-norm numerical errors, the historical
power iteration uses a dependent-tensor derivative as a purported
transpose without a guaranteed adjoint identity for nonlinear
ancestor--descendant pairs.  Increasing the number of probes alone
does not resolve that issue.  The affected stochastic values,
including the unified Exp.~2 sweep and Exp.~6, are preserved for
provenance but cannot certify rank bounds or spectral concentration
without an adjoint check and, where needed, recomputation.
This is a limitation of the guarantee, not evidence that every
historical value is wrong; exact-matrix stable ranks and the
forward-product norm/GN-Gap estimators do not share this particular
transpose issue.

\FloatBarrier
\subsection{Exp.~5: Toy-Attention protocol and additional details}
\label{app:exp5}

\textbf{Architecture.}
ToyAttentionModel: input $X\!\in\!\mathbb{R}^{S\times d}$
($S\!=\!8$, $d\!=\!16$),
single-head self-attention
$O=\mathrm{Softmax}(QK^\top\!/\sqrt{d})\,V$,
mean-pool, Linear($d$,\,1).
Parameters: $3d^2\!+\!d\!+\!1=785$.
ToyReluMLP: input $X\!\in\!\mathbb{R}^{S\times d}$,
$3\!\times\![\mathrm{Linear}(d,d){+}\mathrm{ReLU}]$
applied identically at each position (shared weights),
mean-pool, Linear($d$,\,1).
Parameters: $3(d^2\!+\!d)+d+1=833$.
Both models aggregate via mean-pooling, but only Attention mixes
positions before pooling and contains bilinear score/value maps.
The measured pairs also differ: sibling nodes $(Q,K)$ versus the
ancestor--descendant pair $(\mathrm{block}_0,\mathrm{block}_1)$.
The comparison is therefore not an isolated change of nonlinearity.

\textbf{Synthetic data.}
Teacher network with fixed random weights
($\mathrm{seed}\!=\!0$, independent of training seeds):
\[
  y = \frac{1}{S}\sum_{s=1}^{S}
  \tanh\!\bigl(x_s^\top W_{\mathrm{t}}\bigr)\;
  w_{\mathrm{r}} + \varepsilon,
  \qquad
  W_{\mathrm{t}}\!\in\!\mathbb{R}^{d\times d},\;
  w_{\mathrm{r}}\!\in\!\mathbb{R}^{d},\;
  \varepsilon\!\sim\!\mathcal{N}(0,0.01).
\]
$W_{\mathrm{t},ij}\!\sim\!\mathcal{N}(0,1/d)$,
$w_{\mathrm{r},j}\!\sim\!\mathcal{N}(0,1/d)$.
Training set: 2048 samples; validation: 512; targets
standardized per training split.

\textbf{GN-Gap computation.}
Both models expose a \texttt{forward\_with\_intermediates}
method returning the graph of intermediate activations.
For Attention, the pair $(v,w)\!=\!(Q,K)$;
for ReLU-MLP, $(\mathrm{block}_0,\mathrm{block}_1)$.
The cross-Hessian $\partial^2\mathcal{L}/\partial f_v\,\partial f_w$
is computed exactly via second-order autograd; the Gauss--Newton
component is
$H^{GN}_{v,w}=(2/B)\sum_{b=1}^{B}J_v^{(b)\top}J_w^{(b)}$
(MSE loss with mean reduction, scalar output).

\FloatBarrier
\subsection{Exp.~6: ResNet-18 protocol and additional results}
\label{app:exp6}

\textbf{Architecture.}
SegmentedResNet18: standard torchvision ResNet-18
(${\sim}11$M parameters) with 5~layer names
(stem, layer1--layer4) and 6~segments.
$\mathrm{seg}_5\!=\!\mathrm{avgpool}+\mathrm{flatten}+\mathrm{fc}$
is purely linear (no activation).
SegmentedPlainResNet18: identical parameterization with
identity shortcuts replaced by sequential
conv--BN--$\sigma$ paths; stride-based downsampling preserved.

\textbf{Training details.}
SGD ($\eta\!=\!0.1$, momentum~$0.9$,
weight decay~$5{\cdot}10^{-4}$), cosine schedule,
100~epochs, batch~128.
CIFAR-10 augmentation: random crop $32{\times}32$
with padding~4, random horizontal flip,
normalization to channel mean/std.
Five seeds $\{42,\ldots,46\}$.

\textbf{Stochastic estimation.}
Hutchinson estimator: $m\!=\!30$ Rademacher probes,
subsample of 32~examples.
Power iteration uses the historical estimator default of 20~steps;
its transpose limitation is described in~\ref{app:ablation}.
The estimator returns $\mathcal{D}=\texttt{NaN}$ when its squared
spectral-norm estimate falls below the numerical threshold; this
is not proof of a structural zero.
GN-Gap estimator: $m\!=\!30$ probes, common probe vector,
analytical $J^\top L'' J$ computation.

\textbf{Mid-checkpoint results.}
Table~\ref{tab:exp6-mid} provides the reported metrics
at the mid checkpoint (epoch~50), complementing the
init and final data in the main text.  Its stochastic
$\mathcal{D}$ entries retain the limitation in~\ref{app:ablation}.

\begin{table}[H]
  \centering
  \caption{Exp.~6 (mid, epoch~50): metrics $\mathcal{R}$,
    $\mathcal{C}$, $\mathcal{D}$ by distance~$d$
    (mean over 5~seeds).
    Test accuracy: ReLU ResNet~$80.8\,\%$,
    ReLU Plain~$77.0\,\%$,
    SiLU ResNet~$79.9\,\%$,
  SiLU Plain~$78.4\,\%$.}
  \label{tab:exp6-mid}
  \footnotesize
  \begin{tabular}{@{}llccc@{}}
    \toprule
    Condition & $d$ & $\mathcal{R}(d)$ & $\mathcal{C}(d)$ & $\mathcal{D}(d)$ \\
    \midrule
    ReLU ResNet & 0 & $1.05$ & $1.00$ & $2.87$ \\
    & 1 & $0.87$ & $0.93$ & $2.73$ \\
    & 2 & $0.78$ & $0.85$ & $2.65$ \\
    & 3 & $0.72$ & $0.79$ & $2.85$ \\
    & 4 & $0.39$ & $0.73$ & $2.91$ \\
    \midrule
    ReLU Plain & 0 & $3.81$ & $1.00$ & $2.52$ \\
    & 1 & $2.16$ & $0.90$ & $2.40$ \\
    & 2 & $1.80$ & $0.81$ & $2.29$ \\
    & 3 & $1.73$ & $0.71$ & $2.47$ \\
    & 4 & $0.41$ & $0.66$ & $3.06$ \\
    \midrule
    SiLU ResNet & 0 & $2.60$ & $1.00$ & $3.47$ \\
    & 1 & $1.68$ & $0.90$ & $3.33$ \\
    & 2 & $1.27$ & $0.83$ & $3.30$ \\
    & 3 & $1.64$ & $0.74$ & $3.30$ \\
    & 4 & $0.49$ & $0.61$ & $2.95$ \\
    \midrule
    SiLU Plain & 0 & $3.89$ & $1.00$ & $2.79$ \\
    & 1 & $2.30$ & $0.87$ & $2.64$ \\
    & 2 & $1.61$ & $0.79$ & $2.29$ \\
    & 3 & $1.87$ & $0.72$ & $2.46$ \\
    & 4 & $0.45$ & $0.60$ & $2.51$ \\
    \bottomrule
  \end{tabular}
\end{table}

\textbf{GN-Gap across checkpoints.}
Table~\ref{tab:exp6-gngap-full} reports GN-Gap$(d)$ for all
three checkpoints.
For ReLU conditions, the reported distance-bin means are
below $10^{-5}$ at all checkpoints, consistent with
Proposition~\ref{prop:ad-ggn-equiv} for evaluation-mode BatchNorm.
For SiLU, the changes depend on architecture and distance:
in Plain networks Gap increases at $d\leq1$ but decreases at
$d=2,3$ from init to final; in ResNet it decreases at $d=1,2,3$.
The ratios alone do not distinguish changing source magnitudes
from transport, cancellation, or changes in the GGN denominator.
At $d\!=\!4$ (linear head), Gap~${\sim}\!10^{-6}$ for
\emph{all} conditions---the analogy with ReLU is explained
by linearity of avgpool$\,{+}\,$fc
(Remark~\ref{rem:linear-head}).

\begin{table}[H]
  \centering
  \caption{Exp.~6: GN-Gap$(d)$ across checkpoints
    (mean over 5~seeds).
    $^\dagger$~At $d\!=\!4$ the segment is purely linear
  ($H^T\!\equiv\!0$).}
  \label{tab:exp6-gngap-full}
  \footnotesize
  \begin{tabular}{@{}llccccc@{}}
    \toprule
    Checkpoint & Condition & $d\!=\!0$ & $d\!=\!1$ & $d\!=\!2$ & $d\!=\!3$ & $d\!=\!4^\dagger$ \\
    \midrule
    init & ReLU ResNet & ${<}10^{-5}$ & ${<}10^{-5}$ & ${<}10^{-5}$ & ${<}10^{-5}$ & ${<}10^{-5}$ \\
    & ReLU Plain & ${<}10^{-5}$ & ${<}10^{-5}$ & ${<}10^{-5}$ & ${<}10^{-5}$ & ${<}10^{-5}$ \\
    & SiLU ResNet & $0.426$ & $0.526$ & $0.619$ & $0.642$ & ${<}10^{-5}$ \\
    & SiLU Plain & $0.003$ & $0.028$ & $0.080$ & $0.132$ & ${<}10^{-5}$ \\
    \midrule
    mid & ReLU ResNet & ${<}10^{-5}$ & ${<}10^{-5}$ & ${<}10^{-5}$ & ${<}10^{-5}$ & ${<}10^{-5}$ \\
    & ReLU Plain & ${<}10^{-5}$ & ${<}10^{-5}$ & ${<}10^{-5}$ & ${<}10^{-5}$ & ${<}10^{-5}$ \\
    & SiLU ResNet & $0.274$ & $0.144$ & $0.137$ & $0.136$ & ${<}10^{-5}$ \\
    & SiLU Plain & $0.122$ & $0.062$ & $0.041$ & $0.056$ & ${<}10^{-5}$ \\
    \midrule
    final & ReLU ResNet & ${<}10^{-5}$ & ${<}10^{-5}$ & ${<}10^{-5}$ & ${<}10^{-5}$ & ${<}10^{-5}$ \\
    & ReLU Plain & ${<}10^{-5}$ & ${<}10^{-5}$ & ${<}10^{-5}$ & ${<}10^{-5}$ & ${<}10^{-5}$ \\
    & SiLU ResNet & $0.426$ & $0.241$ & $0.220$ & $0.150$ & ${<}10^{-5}$ \\
    & SiLU Plain & $0.213$ & $0.118$ & $0.058$ & $0.054$ & ${<}10^{-5}$ \\
    \bottomrule
  \end{tabular}
\end{table}

\end{document}